\documentclass{article}

\usepackage[utf8]{inputenc}
\usepackage[T1]{fontenc}
\usepackage{url}
\usepackage{booktabs}
\usepackage{amsfonts}
\usepackage{nicefrac}
\usepackage{microtype}
\usepackage{xcolor}
\usepackage{graphicx}
\usepackage{amsmath,amssymb,amsthm}
\usepackage[shortlabels]{enumitem}
\usepackage{textcomp}
\usepackage{amssymb}
\usepackage{mathrsfs}
\usepackage{algorithm}
\usepackage[noend]{algpseudocode}
\usepackage{xcolor}
\usepackage{graphicx}
\usepackage{tabularx}
\usepackage{tikz}
\usepackage{natbib}
\usepackage{geometry}
\AtBeginDocument{
  \newgeometry{
    textheight=9in,
    textwidth=5.5in,
    top=1in,
    headheight=12pt,
    headsep=25pt,
    footskip=30pt
  }
}

\usepackage{hyperref}

\title{Operationalizing Individual Fairness via Gradient Descent and Bradley–Terry Models}

\author{
  Conlan Olson \\
  Columbia University \\
  \texttt{co2583@columbia.edu}
  \and
  Linjun Zhang \\
  Rutgers University \\
  \texttt{linjun.zhang@rutgers.edu}
  \and
  Zhun Deng \\
  UNC Chapel Hill \\
  \texttt{zhundeng@cs.unc.edu}
  \and
  Pragya Sur \\
  Harvard University \\
  \texttt{pragya@fas.harvard.edu}
}

\date{}

\newtheorem{thm}{Theorem}
\newtheorem{lem}{Lemma}
\newtheorem{assumption}{Assumption}

\theoremstyle{definition}
\newtheorem{definition}{Definition}

\theoremstyle{remark}

\newcommand{\R}{\mathbb R}
\newcommand{\bX}{\boldsymbol{X}}
\newcommand{\bx}{\boldsymbol{x}}
\newcommand{\bK}{\boldsymbol{K}}
\newcommand{\bM}{\boldsymbol{M}}
\newcommand{\bI}{\boldsymbol{I}}
\newcommand{\bL}{\boldsymbol{L}}
\newcommand{\bA}{\boldsymbol{A}}
\newcommand{\bZ}{\boldsymbol{Z}}
\newcommand{\bz}{\boldsymbol{z}}
\newcommand{\bU}{\boldsymbol{U}}
\newcommand{\bu}{\boldsymbol{u}}
\newcommand{\bV}{\boldsymbol{V}}
\newcommand{\bv}{\boldsymbol{v}}
\newcommand{\bSigma}{\boldsymbol{\Sigma}}
\newcommand{\bD}{\boldsymbol{D}}
\newcommand{\bB}{\boldsymbol{B}}

\newcommand{\bH}{\boldsymbol{H}}

\newcommand{\bW}{\boldsymbol{W}}
\newcommand{\bO}{\boldsymbol{O}}
\newcommand{\bLambda}{\boldsymbol{\Lambda}}
\newcommand{\bGamma}{\boldsymbol{\Gamma}}
\newcommand{\bPsi}{\boldsymbol{\Psi}}

\newcommand{\dd}{\textnormal{d}}
\newcommand{\oS}{\overline{S}}
\DeclareMathOperator{\Var}{Var}
\DeclareMathOperator{\Cov}{Cov}
\DeclareMathOperator{\vc}{vec}
\DeclareMathOperator{\Tr}{Tr}
\DeclareMathOperator*{\EE}{\mathbb{E}}
\DeclareMathOperator*{\argmin}{arg\,min}

\begin{document}

\maketitle

\begin{abstract}
Individual fairness, the notion that ``similar individuals should be treated similarly,'' provides a strong and flexible fairness guarantee for algorithmic decision makers. However, a barrier to implementing individual fairness in practice is the difficulty of learning the similarity metric over individuals. In this work, we present an algorithm for learning a Mahalanobis similarity metric from triplet queries of the form ``is individual $i$ more similar to individual $j$ or $k$?'' We work in the standard Bradley-Terry model for pairwise comparisons. Our algorithm consists of a spectral initialization step followed by gradient descent. We provide extensive theoretical guarantees on our algorithm, showing that it converges quickly to the ground truth metric despite the non-convexity of the loss in our model. We present experimental results that demonstrate the convergence of our algorithm and show that individual fairness with respect to the estimated metric is sufficient to achieve similar fairness with respect to the true metric. We also discuss potential applications of our work to AI model tuning. 
\end{abstract}

\section{Introduction}
The growing use of algorithms in decision-making processes in fields such as medicine \citep{Chen2023}, hiring \citep{Gaebler2024AuditingTU}, criminal justice \citep{propublica}, and credit \citep{hardt_equalityofopportunity} underscores the importance of ensuring that algorithmic decision-makers are fair. A variety of algorithmic fairness notions have been proposed. On one end of the spectrum, group-based fairness notions protect predetermined groups of people, such as those based on race or gender. Examples of group-fairness notions include equalized odds \citep{hardt_equalityofopportunity} and calibration within groups \citep{kleinberg_tradeoffs}.
Motivated by the need to protect subgroups and intersectional groups, definitions such as multicalibration and multiaccuracy \citep{multicalibration} (collectively, ``multi-X'' notions) generalize group-based fairness to cover subgroups. 
In stark contrast from group-based fairness is the notion of 
 individual fairness \citep{fta}, which asserts that ``similar individuals should be treated similarly.'' This notion provides a very strong and individualized fairness guarantee, and is the subject of this work. 
 
 Our interest in individual fairness is motivated by areas such as personalized medicine \citep{Bertsimas2022}, which seeks to provide individuals highly individualized predictions of disease risks and other health parameters. Fairness is an important concern in these systems \citep{Chen2023, genevieve2020structural, Paulus2020}. %
 For example, \citet{Seyyed-Kalantari2021} show that AI algorithms used to diagnose pathologies from chest X-rays severely underdiagnose patients from under-served groups. Relying on group-based or multi-X fairness notions in these areas would be highly unsatisfactory because group-level adjustments would undermine the personalization at the heart of personalized medicine. In our example, it would not make sense to simply increase the diagnosis probability for, say, pneumonia just because the patient comes from an under-served group. Instead, we seek to ensure patients at their individual levels receive reliable predictions, irrespective of group membership. With this goal in mind, we turn to individual fairness.

\citet{fta} defined an algorithm to be individually fair if individuals that are inherently similar to each other in the context of the task at hand also receive similar predictions or classifications from the algorithm.
At the core of this notion lies the challenge of identifying which individuals are similar to which others in the context of a given task. \citet{fta} formalized ``similarity'' as a distance metric $d$ on individuals that quantifies how similar they are with respect to the task at hand; we call this the fairness metric. In our example of diagnosing pneumonia from chest X-rays, the fairness metric should measure how similar two X-rays actually are in terms of their pneumonia-related features. 

\citet{fta} describes how to achieve individual fairness, assuming full access to the fairness metric. However, the problem of obtaining the fairness metric remains. In theory, oracle access to human domain experts could provide a good proxy for the fairness metric. However, such access for all members of a population is hardly ever realistic in practice. Additionally, learning such a fairness metric is notoriously hard from observed data, because translating notions of similarity between individuals into a formal distance metric is, without further context, an ill-posed mathematical question. This bottleneck has prevented a widespread adoption of individual fairness, even in areas such as personalized medicine where individualized fairness notions are critical. To mitigate this, prior work \citep{ilvento, pmlr-v119-mukherjee20a, gillen} has proposed approaches to learning the fairness metric, under various models and from various types of data.

Our work provides a data-driven algorithm that provably learns this fairness metric in a realistic setting from input data that is reasonable to collect---especially a setting that has not been addressed in prior work. Importantly, we bring in concepts from the literature on Bradley-Terry models and ranking in statistics to learn the fairness metric needed for individual fairness, a problem that has been considered rather challenging in the fairness literature beyond the few settings considered in the aforementioned papers. %
Specifically, we work in the following setting. We observe $n$ individuals, each represented by a $p$-dimensional vector in Euclidean space, denoted by $\bx_i\in\R^p$. As training data, we observe responses  to triplet queries, provided by human experts, of the form ``is $\bx_i$ more similar to $\bx_j$ or $\bx_k$?'' In our example of diagnosing chest X-rays, an example of a query might be ``is X-ray $i$ more similar to X-ray $j$ or X-ray $k$ in terms of the evidence for pneumonia?''  As suggested in \citet{ilvento}, queries of this type are often psychologically easier for human experts to answer than queries such as ``how similar are $\bx_i$ and $\bx_j$?'' Triplet queries are also common in the broader machine learning  literature \citep{kleindessner2017kernel, nowak_jamieson, tamuz_adaptive, jain2016finite}. 
In this setting, we seek to learn the metric underlying the query responses. We assume that the responses to the triplet queries follow a Bradley-Terry model (see \eqref{eqn:bt_model}). 
We develop a novel algorithm based on spectral initialization and gradient descent that provably learns this metric (Theorem \ref{thm:alg_performance}). We describe our contributions in further details below.

In addition to the fairness considerations above, our work also has applications to post-training AI models in areas such as robotics or computer vision. An important step in AI model development is aligning models with human preferences. Many alignment techniques are based on human feedback in the form of pairwise preferences (``$\bx_i$ is better than $\bx_j$''). However, in certain areas of robotics and computer vision \citep{mattson_representation_2024, tian2024mattersyouvisualrepresentation, sucholutsky2023alignment}, human feedback may come in the form of \emph{relative} distance queries, such as the triplet queries we study in this work. For feedback in this form, our algorithm can help learn the metric underlying the feedback. The alignment literature calls this step \emph{reward modeling}, which can be a goal in itself or a first step in reinforcement learning from human feedback (RLHF) pipelines.

\paragraph*{Summary of Main Contributions}
We assume that the responses to the triplet queries ``is $\bx_i$ more similar to $\bx_j$ or $\bx_k$?'' follow the Bradley-Terry model \citep{bradley_terry}: for a triplet $t=(i,j,k)$, we observe $y_t\in\{-1,1\}$ distributed according to
\begin{align}\label{eqn:bt_model}
    \Pr[y_t=1]&=\frac{\exp(d^2(\bx_i,\bx_j))}{\exp(d^2(\bx_i,\bx_j))+\exp(d^2(\bx_i,\bx_k))},
\end{align}
where a response $y_t=-1$ indicates that the expert we queried finds $\bx_i$ closer to $\bx_j$ and $y_t=1$ indicates that they find $\bx_i$ closer to $\bx_k$. 
We also assume that the metric $d$ takes the form of a Mahalanobis distance $d_{\bK_\star}(\bx_i,\bx_j):=\sqrt{(\bx_i-\bx_j)^\top\bK_\star(\bx_i-\bx_j)}$, where $\bK_\star\in\R^{p\times p}$ is positive semi-definite. We assume that $\bK_\star$ has a small rank $r<p$. Our contributions are as follows.

\begin{enumerate}
    \item 
    Algorithm \ref{alg:init}, provides a spectral algorithm for estimating a Mahalanobis distance matrix from triplet queries where responses follow the Bradley-Terry model. Theorem \ref{thm:init} shows that the Mahalanobis distance matrix estimated by this algorithm approaches the ground truth matrix in Euclidean norm as the sample size grows. 
    \item Theorem \ref{thm:gd_induction} shows that gradient descent on a suitable loss function theoretically approaches the true Mahalanobis distance matrix in Euclidean norm as the number of iterations grows, provided that gradient descent is initialized from a point sufficiently close to the true matrix. Combined with our initialization algorithm, this allows us to show Theorem \ref{thm:alg_performance}, which states that our algorithm produces an estimate of the Mahalanobis distance matrix that converges to the true matrix.
    \item We present experimental results showing that our algorithm produces estimates approaching the ground truth Mahalanobis distance matrix. Our experiments also demonstrate that we can achieve individual fairness by first estimating the metric and then training an algorithm to be fair with respect to the estimate. 
\end{enumerate}

The remainder of the paper is structured as follows. 
Sections \ref{sec:related_work} and \ref{sec:set_up} review prior work and describe our problem set-up, respectively. 
Section \ref{sec:alg} presents our algorithm while Section \ref{sec:theoretical_results} presents corresponding theoretical results. Section \ref{sec:experiments} shows our experimental results while Section \ref{sec:discussion} concludes with a discussion. Proof details are deferred to the appendix.

\section{Related Work}\label{sec:related_work}
\paragraph*{Individual Fairness}
Individual fairness was proposed by \citet{fta}. The algorithms in \citet{fta} require access to all pairwise distances under the fairness metric. \citet{ftcba} gives an algorithm that only requires a small number of queries to the fairness metric. 

\citet{ilvento} gives an algorithm for learning a metric for individual fairness from triplet queries and a small number of real-valued queries of the form ``how far is individual $\bx_i$ from individual $\bx_j$?'' Answers to questions of the latter form may be difficult to obtain in practice---additionaly, \citet{ilvento} requires perfect answers to the triplet queries. In contrast, our approach only requires triplet queries and works in the setting where arbiters may make mistakes in their answers. \citet{pmlr-v119-mukherjee20a} provides algorithms for learning Mahalanobis distances for individual fairness, however, they do not consider triplet queries. \citet{gillen} assumes the metric is a Mahalanobis distance and gives an online algorithm that uses an arbiter that can determine which decisions are unfair. Other approaches to individual fairness avoid directly learning the metric itself. For example, \citet{jung_eliciting} learns from arbiter judgements of whether one individual should be treated better than another.

\citet{zemel} proposes mapping each individual to a representation that removes information about protected attributes while preserving relevant information to the task. Learning a Mahalanobis distance with a matrix $\bK=\bA\bA^\top$ can be viewed as learning an embedding $\bA:\R^p\to\R^r$ such that the similarity metric is the Euclidean distance in the embedding space. Via this connection, learning a Mahalanobis distance for individual fairness is similar to learning fair representations.

\paragraph*{Metric Learning}
The literature on metric learning is extensive.  \citet{bellen_survey} gives a useful survey of the field. \citet{jain2018learning} studies the problem of learning a Mahalanobis distance under a generalized version of the Bradley-Terry model. In contrast to our work, they analyze properties of the empirical risk minimizer rather than the performance an estimation algorithm. Our use of triplet queries also follows \citet{kleindessner2017kernel, nowak_jamieson, tamuz_adaptive, jain2016finite} (among others) in the metric learning literature.

Our initialization uses an algorithm from \citet{rank_centrality_chen} that recovers the underlying scores from pairwise comparisons under the Bradley-Terry model. By treating pairs of individuals as elements and treating pairwise distances as scores, we use their algorithm can be used to learn pairwise distances.

\paragraph*{Nonconvex Optimization}
Our algorithmic approach and proofs are based on \citet{ma2019implicit}, which shows that gradient descent, even on nonconvex loss functions, sometimes converges rapidly when started from particular initial conditions. \citet{ma2019implicit} shows this property of gradient descent for three learning tasks. Our work shows a similar property for the task of learning a Mahalanobis distance from triplet queries.

\paragraph*{AI Model Tuning}
The standard literature on reinforcement learning from human feedback shares our use of the Bradley-Terry model \citep{ouyang_training_nodate}. Several works suggest specialized reinforcement learning setups that depend on relative distance queries \citep{mattson_representation_2024,tian2024mattersyouvisualrepresentation, sucholutsky2023alignment}. 

\section{Problem Set-Up}\label{sec:set_up}
First, we recall our setup from the introduction. We observe $\bX=\{\bx_1,\ldots,\bx_n\}\in\R^{n\times p}$, a set of $n$ individuals each represented by a $p$-dimensional feature vector. We observe triplet queries ``is $\bx_i$ closer to $\bx_j$ or $\bx_k$?'', represented by $t=(i,j,k)$, and responses $y_t\in\{-1,1\}$. We observe queries for a set $S$ of triplets that is sampled  from the set of all possible triplets $\overline S:=\{(i,j,k)\mid i\neq j\neq k\neq i\}$ by including each triplet independently with constant probability $s$.

We model the problem as follows. We assume the responses $y_t$ are distributed according to the Bradley-Terry model \eqref{eqn:bt_model}. We assume the true distance metric is a Mahalanobis distance $d_{\bK_\star}(\bx_i,\bx_j):=\sqrt{(\bx_i-\bx_j)^\top\bK_\star(\bx_i-\bx_j)}$, where $\bK_\star$ is a positive semi-definite $p\times p$ matrix with rank $r$, norm $\|\bK_\star\|$, and condition number $\kappa$.
\footnote{Our results for gradient descent extend to the setting where $p$ grows, as long as $p^4/n=o(1)$. For simplicity, we treat $p$ as fixed. Our asymptotic bounds in Theorems \ref{thm:hessian} and \ref{thm:gd_induction} are unchanged under the weaker assumption that $p^4/n=o(1)$.} 
We can write the model concisely using the following notation, borrowed from \citet{jain2018learning}. We define $\bM_t:=\bx_i\bx_k^\top+\bx_k\bx_i^\top-\bx_i\bx_j^\top-\bx_j\bx_i^\top+\bx_j\bx_j^\top-\bx_k\bx_k^\top.$ Then, $d_{\bK}^2(\bx_i,\bx_j)-d_{\bK}^2(\bx_i,\bx_k)=\Tr(\bM_t\bK).$
This allows us to rewrite \eqref{eqn:bt_model} as
\begin{equation}
    \Pr[y_t=y]=\frac{1}{1+\exp(-y\Tr(\bM_t\bK_\star))}\label{eqn:bt_model_concise}
\end{equation}
for $y\in\{-1,1\}$. From the data $\{(t,y_t)\}_{t\in S}$, we produce an algorithm to accurately estimate the matrix $\bK_\star$. Since $\bK_\star$ has rank $r$, we will equivalently estimate a full-rank matrix $\bA_\star\in\R^{p\times r}$ such that $\bA_\star\bA_\star^\top=\bK_\star$. 

Notationally, we use $\|\cdot\|$ to mean the $L^2$ norm and $\lesssim$ to indicate inequality up to constants, i.e., $a\lesssim b$ if there is a constant $C$ such that $a<Cb$.

\section{Algorithm}\label{sec:alg}
Our algorithm is divided into two steps, an initialization step and a gradient descent step. At an intuitive level, we would like to maximize the log likelihood of a matrix $\bA$, given the observed data. In particular, the loss function
\begin{equation}\label{eqn:risk}
    L(\bA):=\frac{1}{|S|}\sum_{t\in S}\log(1+\exp(-y_t \Tr(\bM_t\bA\bA^\top)))
\end{equation}
is a natural target for minimization, since the negative log likelihood of a matrix $\bA\bA^\top$ for a single triplet $t$ is $\log(1+\exp(-y_t \Tr(\bM_t\bA\bA^\top)))$, as can be seen from the model \eqref{eqn:bt_model_concise} with $\bA\bA^\top$ as the Mahalanobis distance matrix. 

A typical algorithm for minimizing a loss function is gradient descent. However, the loss function \eqref{eqn:risk} is \emph{not} strongly convex and smooth, so there is no guarantee that na\"ive gradient descent will converge quickly. However, the loss function \emph{is} strongly convex (in a modified sense) and smooth within a certain radius around any $\bA_\star$ that satisfies $\bA_\star\bA_\star^\top=\bK_\star$. Furthermore, when we run a step of gradient descent from within this region, the next step will remain within the region by properties of non-convex optimization as uncovered in \citep{chi2019nonconvex}.

Therefore, gradient descent will work well if we can ensure that we start within a region where the loss function is well-behaved. This motivates our two-step approach. First, we run an initialization algorithm to find a point within the region where we have guarantees on the convexity and smoothness of our loss function. Then, starting from that point, we run gradient descent. Our algorithm is presented at a high level in pseudocode as Algorithm \ref{alg:summary} (lower level descriptions are given in Appendix \ref{sec:algs_appendix}). The initialization step is described in Section \ref{subsec:init} and the gradient descent step is described in Section \ref{subsec:gd}.

\begin{algorithm} [!hbt]
\caption{Summary of our algorithm}
\label{alg:summary}
\begin{algorithmic}[1]
\Statex{\hspace{-4mm}\textbf{Inputs:} individuals $\bX=\bx_1,\ldots,\bx_n$, triplets $S$, responses $\{y_t\mid t\in S\}$, step size $\eta$, and time $T$.}
\vspace{2mm}
\State{Use \textsc{RankCentrality} to obtain an estimate of the pairwise distance matrix $(d_{\bK_\star}(\bx_i,\bx_j))_{i,j}$.}
\State{Solve a generalized eigenvector problem to recover an estimate $\bA_0$ of $\bA_\star$.}
\State{Define $L(\bA)$ as in \eqref{eqn:risk}.}
\For{$t\in[T]$}
    \State{$\bA_t\gets \bA_{t-1}-\eta\nabla L(\bA_{t-1})$}
\EndFor
\State{Return $\bA_T\bA_T^\top$.}
\end{algorithmic}
\end{algorithm}
\subsection{Initialization Step}\label{subsec:init}
The major goal of the initialization step is to identify a region within which gradient descent remains well-behaved. To this end, we will utilize  a spectral approach. First, we use existing methods to estimate the full matrix of pairwise distances, that is,  $(d_{\bK_\star}(\bx_i,\bx_j))_{i,j}$. %
Then we solve a generalized eigenvector problem to recover the underlying matrix $\bK_\star$.

To estimate the pairwise distances $d_{\bK_\star}(\bx_i,\bx_j)$, %
 we use the Rank Centrality algorithm \citep{rank_centrality_chen}. We treat pairs of individuals as nodes in a graph, and we create a matrix representing the input data $\{y_i\}$ as a random walk on that graph. This random walk has the property that pairs with large pairwise distance are more likely to be reached than pairs with small pairwise distance. Then, we find the stationary distribution of this walk. Intuitively, the stationary distribution should allocate higher mass to pairs with larger pairwise distance. Indeed, \citet{rank_centrality_chen} shows that the stationary distribution approximately recovers the distance associated with each pair of individuals (up to an additive constant). After applying a centering matrix $\boldsymbol J$ to remove the additive constant, we obtain a matrix $\widehat{\boldsymbol H}$ that is approximately equal to ${\bX}^\top\bA_\star\bA_\star^\top\bX$. 

To finish the initialization step, we solve a generalized eigenvector problem to approximately recover $\bA_\star$ (up to orthogonal transformation) from ${\bX}^\top\bA_\star\bA_\star^\top\bX$. The full initialization step is given in Appendix \ref{sec:algs_appendix} as Algorithm \ref{alg:init}.

\subsection{Gradient Descent Step}\label{subsec:gd}
Starting from the initial point identified by the previous subsection, we run gradient descent on the loss function \eqref{eqn:risk}. We use a small constant as the step size and run until convergence. We tune the step size and total number of iterations in practice, and specify the settings in our experiments. Using gradient descent to minimize the loss corresponds intuitively to finding an $\widehat{\bA}$ that is likely, given our training data. We show below that the $\hat{A}$ thus obtained has desirable statistical properties and allows to recover the Mahalanobis distance $K_\star$ of interest to us. Downstream, this also allows to recover the fairness metric--we provide these mathematical details in the next section. The gradient descent step is described in Appendix \ref{sec:algs_appendix} as Algorithm \ref{alg:gd}.

\section{Theoretical Results}\label{sec:theoretical_results}

Our central theoretical result shows that our algorithm can be used to estimate a Mahalanobis distance from triplet queries. We state the following theorem, which guarantees the performance of Algorithm \ref{alg:summary}:
\begin{thm}\label{thm:alg_performance}
  Under Assumption \ref{assumption:distribution}, let $\bA_T$ be the output of Algorithm \ref{alg:summary} with time $T$ and sufficiently small step size $\eta$. Define $\widehat{\bK}:=\bA_T\bA_T^\top$ and let $\bK_\star$ be the true Mahalanobis distance matrix. Then, $\|\widehat{\bK}-\bK_\star\|\lesssim e^{-T}+1/n$
\end{thm}
The proof of Theorem \ref{thm:alg_performance} appears in Appendix \ref{pf:alg_performance}. It requires intermediate Theorems \ref{thm:init}, \ref{thm:hessian}, and \ref{thm:gd_induction}. Theorems \ref{thm:init}-\ref{thm:gd_induction} collectively show that our algorithm produces $\hat{\bK}$ that is close in norm to $\bK_\star$. To this end, we will follow a two-step procedure---we will first analyze the initialization step and then study properties of the gradient descent step. 
Our next result shows that the output of our initialization step from Section \ref{subsec:init} is close to the ground truth (up to rotations) for large sample sizes.

\begin{thm}\label{thm:init}
    Let $\bA_0$ be the output of Algorithm \ref{alg:init}. Under Assumption \ref{assumption:distribution}, there exists an orthogonal matrix $\boldsymbol O$ such that
    \(\|\bA_0\boldsymbol O-\bA_\star\|\lesssim \sqrt{1/n\log n}.\)
\end{thm}
For the above theorems and the ones below, we require standard distributional assumptions---such as moment bounds and Orlicz norm bounds---and an anti-concentration assumption on the observed features.
We collect our assumptions  in Appendix \ref{sec:assumptions}. 
To prove Theorem \ref{thm:init}, we start by using results from \citet{rank_centrality_chen} to show that RankCentrality approximately recovers the pairwise distance matrix $\{d_{\bK_\star}(\bx_i,\bx_j)\}_{i,k}$. The second and third steps to the initialization algorithm, after running RankCentrality, are centering and solving a generalized eigenvalue problem. The challenge in proving Theorem \ref{thm:init} lies in ensuring that those steps keep the error small (see proof details  in \ref{pf:init}).

Having shown that our initialization reaches a point close to the ground truth, we turn to the gradient descent phase of our algorithm. The loss function $L$, defined in \eqref{eqn:risk}, is nonconvex. However, the following theorem shows that, within a region of small radius around $\bA_\star$, $L$ is strongly convex (in a modified sense) and smooth. The strong convexity condition is modified since $L$ is only strongly convex in certain directions due to the orthogonal invariance of the loss function.

\begin{thm}\label{thm:hessian}
    Let $\bA_\star$ be any matrix such that $\bA_\star\bA_\star^\top=\bK_\star$. Assume Assumption \ref{assumption:distribution}. Also, suppose that there exists a constant $C$ such that $\|\bA-\bA_\star\|_\textnormal{F}\le C\sqrt{\frac{1}{n\log n}}$. %
    
    Then, for $n$ sufficiently large, with probability at least $1-O(\exp(-n))$, we have
    \begin{align}
        \|\nabla^2 L(\bA)\|^2&\lesssim 1\label{eqn:hessian_smoothness}
    \end{align}
    \begin{align}\label{eqn:hessian_strong_convexity}
        \vc(\widetilde{\bZ})^\top\nabla^2L(\bA)\vc(\widetilde{\bZ})&\gtrsim\|\widetilde{\bZ}\|_\textnormal{F}^2
    \end{align}
    for $\widetilde{\bZ}=\bZ\bH_{\bZ}-\bA_\star$, where $\bZ\in\R^{p\times r}$ is any matrix and $\bH_{\bZ}=\argmin_{\boldsymbol{O}\in\mathcal O_{r\times r}}\|\bZ\boldsymbol{O}-\bA_\star\|_\textnormal{F}.$
\end{thm}
The proof requires bounding all parts of the Hessian of the loss function $\nabla^2 L(\bA)$, which is
{\begin{equation*}
    \frac{1}{|S|}\sum_{t\in S}\Big(\underbrace{-\frac{2y_t\bI_r\otimes\bM_t}{{(\exp(y_t\Tr(\bM_t \bA\bA^\top))+1)^2}}}_{\alpha_1}+\underbrace{\frac{4\exp(y_t\Tr(\bM_t \bA\bA^\top))\vc(\bM_t\bA)\vc(\bM_t\bA)^\top}{(\exp(y_t\Tr(\bM_t \bA\bA^\top))+1)^2}}_{\alpha_2}\Big).
\end{equation*}}
At a very high level, the proof proceeds by showing that $\alpha_1$ is small, and then showing that $\alpha_2$ is both upper and lower bounded. Both steps are highly non-trivial; details can be found in Appendix \ref{pf:hessian}. Theorem \ref{thm:hessian} gives us control over the smoothness and convexity of the Hessian. The following theorem turns this into a guarantee that gradient descent converges quickly.

\begin{thm}\label{thm:gd_induction}
    Assume Assumption \ref{assumption:distribution}. Let $\mathcal O_{r\times r}$ be the set of $r\times r$ orthogonal matrices. Let $\bO_w=\argmin_{\bO\in\mathcal O_{r\times r}}\|\bA_w\bO-\bA_\star\|_\textnormal{F}.$ Suppose there exists a constant $C_1$ such that $\|\bA_w\bO-\bA_\star\|_\textnormal{F}\le C_1\sqrt{1/(n\log n)}.$ Then, there exist constants $C_2,C_3$ such that, if the step size $\eta<C_2$, \(\|\bA_{w+1}\bO_{w+1}-\bA_\star\|\|\bA_w\bO-\bA_\star\|_\textnormal{F}\le\rho\|\bA_w\bO-\bA_\star\|_\textnormal{F}+\frac{C_3}{n}\) with probability $1-O(n^{-10})$, for some $0<\rho<1$.
\end{thm}
To prove this theorem, we first show inductively that the proximity assumption of Theorem \ref{thm:hessian} holds at all time steps, so the Hessian is always smooth and (modified) strongly convex. Then, we follow a standard argument to show that gradient descent converges rapidly. The proof appears in Appendix \ref{pf:gd_induction}. Since we have an initialization step that ensures that we meet the inductive hypotheses at time 0, %
Theorem \ref{thm:gd_induction} shows that Algorithm \ref{alg:gd} converges quickly to the ground truth Mahalanobis matrix.
We now have all the pieces to prove Theorem \ref{thm:alg_performance}. The proof, presented in Appendix \ref{pf:alg_performance}, uses Theorem \ref{thm:init} and Theorem \ref{thm:gd_induction} to show that Algorithm \ref{alg:summary} produces an estimate that is close in norm to $\bK_\star$.

\section{Experiments}\label{sec:experiments}
\subsection{Data Sources}\label{sec:data_sources}
First, we use synthetically generated data. We use $n=120$ data points, each with $p=20$ features. We generate the feature vectors using distributions that include both isotropic and non-isotropic distributions and distributions with independent and correlated coordinates. We present two settings: first, $\bx_i\sim\mathcal N(0,\Sigma_1)$, where $\Sigma_1$ is a diagonal matrix in which a random subset of half of the entries are $5$ and the other entries are $1/5$; and second, $\bx_i\sim\mathcal N(0,\Sigma_2)$, where $\Sigma_2$ is an autoregressive matrix with entry $(i,j)$ given by $0.8^{|i-j|}$. We present more synthetic data settings in Appendix \ref{sec:more_synth_data}.

We also train a downstream fair classifier, so we require a test set and a label. To obtain a test set, we generate another $n$ data points. We generate labels synthetically: for each synthetic dataset, we draw a synthetic signal vector $\bv_\star\sim\mathcal N(0,\mathbf I_p)$. For each data point $\bx_i$, we form the probability $p_i:=1/(1+\exp(\langle \bx_i,\bv_\star\rangle))$. Then, we draw a synthetic label $y_i\sim\text{Bern}(p_i)$.

We also use real datasets. We present experiments on ACS Employment, Credit Card Default, and CDC Diabetes, three real-world settings where fair machine learning is important. In ACS Employment, the features are demographic data about a person and the label is whether the person is employed. In Credit Card Default, the features are a person's demographics and credit card payment history and the label is whether they defaulted on their credit card. In CDC Diabetes, the features are medically-relevant data about a person, and the label is whether the person has diabetes. Full description of these datasets, as well as description of additional real datasets, are in Appendix \ref{sec:more_real_data}.

\subsection{Experimental Procedure}\label{sec:procedure}

For each dataset, we generate a random ground-truth matrix with rank $r=3$ to induce a Mahalanobis distance on the data. Then, we generate training data according to the Bradley-Terry model, including each possible triplet in $S$ with probability $1/2$, and run our algorithm on this data for $T=200$ gradient descent iterations and with step size $\eta=0.1$. Figure \ref{fig:experiments} shows that gradient descent starting from our initialization converges quickly to the ground truth in all settings.

\begin{figure}
\begin{tikzpicture}
\path node[left, scale=0.95]
{
    \includegraphics[scale=0.4]{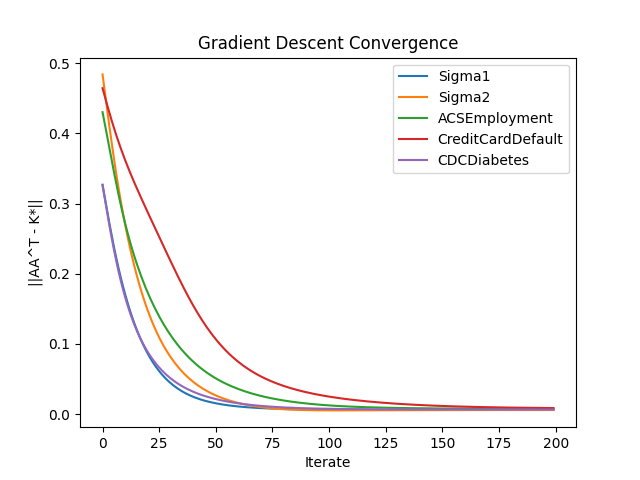}
};
\path node[right, scale=0.95]{
 \begin{tabular}{lccc}\toprule
Dataset & $\zeta_{d_{\widehat{\bK}}}$ & $\zeta_{d_{\bK_\star}}$ & Difference \\\midrule
Synthetic data, $\Sigma_1$ & 1.0066 & 1.0032 & 0.34\% \\
Synthetic data, $\Sigma_2$ & 1.0022 & 0.9962 & 0.60\% \\
ACSEmployment & 0.9999 & 0.998 & 0.19\% \\
CreditCardDefault & 0.9982 & 1.0028 & 0.46\% \\
CDCDiabetes & 1.0009 & 0.9995 & 0.14\% \\\midrule
\end{tabular}
};
\end{tikzpicture}

    \label{fig:experiments}
    \caption{\textbf{Left:} distance between our estimate and the true metric as gradient descent continues. \textbf{Right:} fairness performance of a downstream classifier, measured with respect to the true and the estimated fairness metric.}
\end{figure}

Now, we study the performance of downsteam fair classifiers. We show that classifiers trained to be fair with respect to the estimated fairness metric are also fair with respect to the true metric. We initialize a neural network $\mathcal A$ with three hidden layers, each with latent dimension 20, and sigmoid activations. We use the SenSeI algorithm \citep{yurochkin2021sensei} in the inFairness Python library \cite{infairness} to train the neural network to be both accurate and individually fair. However, we do not assume that we have access to the true input metric $d_{\bK_\star}$. Therefore, we train using the \emph{estimated} metric $d_{\widehat{\bK}}$ from our algorithm. We train for 200 epochs and use SenSeI hyperparameters $\rho=5$, $\epsilon=0.1$, 100 auditor steps, and an auditor learning rate of $0.01$.

Having trained the algorithm $\mathcal A$ to be fair with respect to the learned metric, we check whether it is similarly fair with respect to the true metric. Individual fairness, introduced by \citet{fta}, is defined as follows:
\begin{definition}\label{def:indiv_fairness}
    Let $\bX$ be a universe of individuals and let $\boldsymbol Y$ be a universe of outcomes. An algorithm $\mathcal A:\bX\to\boldsymbol Y$ is $l$-individually fair on $\bX$ with respect to metrics $D:\boldsymbol Y\times\boldsymbol Y\to\R$ and $d:\bX\times\bX\to\R$ if, for all $\bx_i,\bx_j\in\bX$,
    \(
        D(\mathcal A(\bx_i),\mathcal A(\bx_j))\le l\cdot d(\bx_i,\bx_j).
    \)
\end{definition}
That is, if the distance between $\bx_i,\bx_j$ is small per the fairness metric $d$, i.e., $\bx_i$ and $\bx_j$ are similar, then an individually fair algorithm assigns similar outputs to these individuals.

In experimental settings, the individual fairness parameter $l$ can be quite variable since it is a maximum over all pairs of individuals. Instead, we use a more practical fairness measurement proposed by \cite{maity2021statistical}. We fix a loss function $\ell$, in our case the $\ell_1$ loss. Then, defining a function $\Phi_d$ that maps a pair $(\bx,y)$ to a new point $\bx'$ such that $d(\bx,\bx')$ is small but the algorithm $\mathcal A$ %
performs much worse on $\bx'$ (measured with respect to $\ell$), we define $\zeta:=\EE[\ell(\mathcal A(\Phi_d(\bx,y)),y)/\ell(\mathcal A(x),y)]$. $\zeta$ is an interpretable and robust measurement of fairness \citep{maity2021statistical}: if $\mathcal A$ is individually fair, $\zeta$ should be not much larger than 1. In experiments, we use the inFairness Python library \citep{infairness} to measure $\zeta$. See Appendix \ref{sec:fairness_measurements} for details on fairness measurements. 

We define $\widehat\zeta$, obtained using $d=d_{\widehat{\bK}}$ and $\zeta^\star$, obtained using $d=d_{\bK_\star}$. If the fairness of the model with respect to the estimated fairness metric is similar to the fairness of the model with respect to the true fairness metric, then $\widehat{\zeta}$ should be close to $\zeta^\star$. In Figure \ref{fig:experiments}, we see that this holds in all settings: the relative difference between the two quantities never exceeds one percent. Additionally, the values for $\zeta$ are close to 1, indicating that the models are fair with respect to both the estimated and the true metric. In Appendix $\ref{sec:more_results}$, we show fairness experiments for additional datasets, as well as measurements of fairness according to the original individual fairness Definition \ref{def:indiv_fairness}.
 
\section{Discussion}\label{sec:discussion}
Obtaining a suitable similarity metric on individuals has long been a barrier to implementing individual fairness in real settings. In our work, we take steps to overcome this barrier by developing an algorithm that learns a similarity metric for individual fairness from triplet queries under the Bradley-Terry model. We emphasize the feasibility of our approach: triplet queries are reasonable to implement in the real world and the Bradley-Terry model is a realistic model of human responses. In proving theoretical guarantees on the performance of our algorithm, we use ideas from nonconvex optimization and tools from high-dimensional statistics. Our paper is the first in the fairness literature to leverage these statistical tools. Such connections between the algorithmic fairness literature and other areas such as optimization and statistics are an important piece of making fairness work in new and challenging settings. An adjacent and emerging area of research that requires the aggregation of human preferences is AI alignment. As mentioned in the Introduction, our algorithm can be used to learn metrics under human feedback as part of reinforcement learning from human feedback (RLHF) pipelines. In this paper, we focus on algorithmic fairness, which is one aspect of the broader issue of ethical AI. Exploring the applications of our algorithm to relevant RLHF approaches is important future work. 

\section*{Acknowledgments}
P.S.~would like to thank Cynthia Dwork for introducing her to the problem and importance of metric learning in the world of algorithmic fairness. P.S.~gratefully acknowledges support from the National Science Foundation (DMS CAREER 2440824) and the Office of Naval Research (N00014-26-1-2144). L.Z.'s research is partly supported by National Science Foundation (DMS CAREER 2340241).

\bibliographystyle{plainnat}
\bibliography{main}

\pagebreak

\appendix
\section{Organization of the appendicies}
Appendix \ref{sec:algs_appendix} describes our algorithm in pseudocode. Appendix \ref{sec:proofs} gives proofs of the main theorems. Key lemmas and their proofs appear in Appendix \ref{sec:key_lemmas} and technical lemmas and their proofs appear in Appendix \ref{sec:technical_lemmas}. Appendix \ref{sec:more_experiments} gives more details on the experiments and presents experimental results in new settings.
\section{Algorithms}\label{sec:algs_appendix}
This appendix contains pseudocode descriptions of algorithm, separated into Algorithm \ref{alg:init} for the initialization step and Algorithm \ref{alg:gd} for the gradient descent step.
\begin{algorithm} [!hbt]
\caption{Initialization Step}
\label{alg:init}
\begin{algorithmic}[1]
\Statex{\hspace{-4mm}\textbf{Inputs:} individuals $\bX=\bx_1,\ldots,\bx_n$, triplets $S$, and triplet query responses $\{y_t\mid t\in S\}$.}
\vspace{3mm}
\Statex{\hspace{-4mm}\textsc{RankCentrality} is a subroutine that takes as input a set of comparisons $\{(j,k)\}$ and a set of responses $\{y_{(j,k)}\}$ and returns an  estimate of the normalized underlying scores $\{\pi_j\}$}
\vspace{3mm}
\State{Form $\bX'$ by discarding the 0.1 fraction of the individuals in $\bX$ with the largest norm.}
\State{$\bX'\gets \bX'-\frac{1}{n}\boldsymbol{1}\boldsymbol{1}^\top\bX'$ to center the columns of $\bX'$.}
\State{Initialize $\widehat{\boldsymbol P}\in\R^{n\times n}$}
\For{$i\in[n]$}
    \State{$\boldsymbol \pi^i\gets\textsc{RankCentrality}(\{(j,k)\mid (i,j,k)\in S\},\{y_{(i,j,k)}\})$}
    \State{$\boldsymbol p^i\gets(\log\pi^i_1,\ldots,\log\pi^i_{i-1},0,\log\pi^i_{i+1},\ldots,\log\pi^i_{n})$}
    \State{$\widehat{\boldsymbol P}_i\gets\boldsymbol p^i$}
\EndFor
\State{$\boldsymbol J\gets \boldsymbol I_{n}-(\boldsymbol 1\boldsymbol 1^\top/{n})$}
\State{$\widehat{\boldsymbol H}\gets -\frac{1}{2}\boldsymbol J\widetilde{\boldsymbol P}\boldsymbol J$}
\State{$\bPsi\gets{\bX'}^\top\bX'$}
\State{$\widehat{\bU},\widehat{\bLambda}\gets\text{generalized eigenvectors and eigenvalues of the pair }(\frac{1}{{n}^2}{\bX'}^\top\widehat{\boldsymbol H}\bX',\frac{1}{n}\bPsi)$}
\State{$\bA_0\gets\widehat{\bU}\widehat{\bLambda}^{1/2}$}
\end{algorithmic}
\end{algorithm}

\begin{algorithm}[!hbt]
\caption{Gradient Descent Step}
\label{alg:gd}
\begin{algorithmic}[1]
\Statex{\hspace{-4mm}\textbf{Inputs:} individuals $\bX=\bx_1,\ldots,\bx_n$, triplets $S$, triplet query responses $\{y_t\mid t\in S\}$, step size $\eta$, time $T$, and initialization $\bA_0$}
\vspace{3mm}
\State{Define $L(\bA)$ as in \eqref{eqn:risk}}
\For{$t\in[T]$}
    \State{$\bA_t\gets \bA_{t-1}-\eta\nabla L(\bA_{t-1})$.}
\EndFor
\vspace{3mm}
\State{Return $\bA_T\bA_T^\top$}
\end{algorithmic}
\end{algorithm}

\section{Proofs of Main Theorems}\label{sec:proofs}
Before we begin to present the proofs, we include the following diagram as a helpful guide for how our results fit together. Each result is implied by its descendants.

\includegraphics[scale=0.53,angle=90]{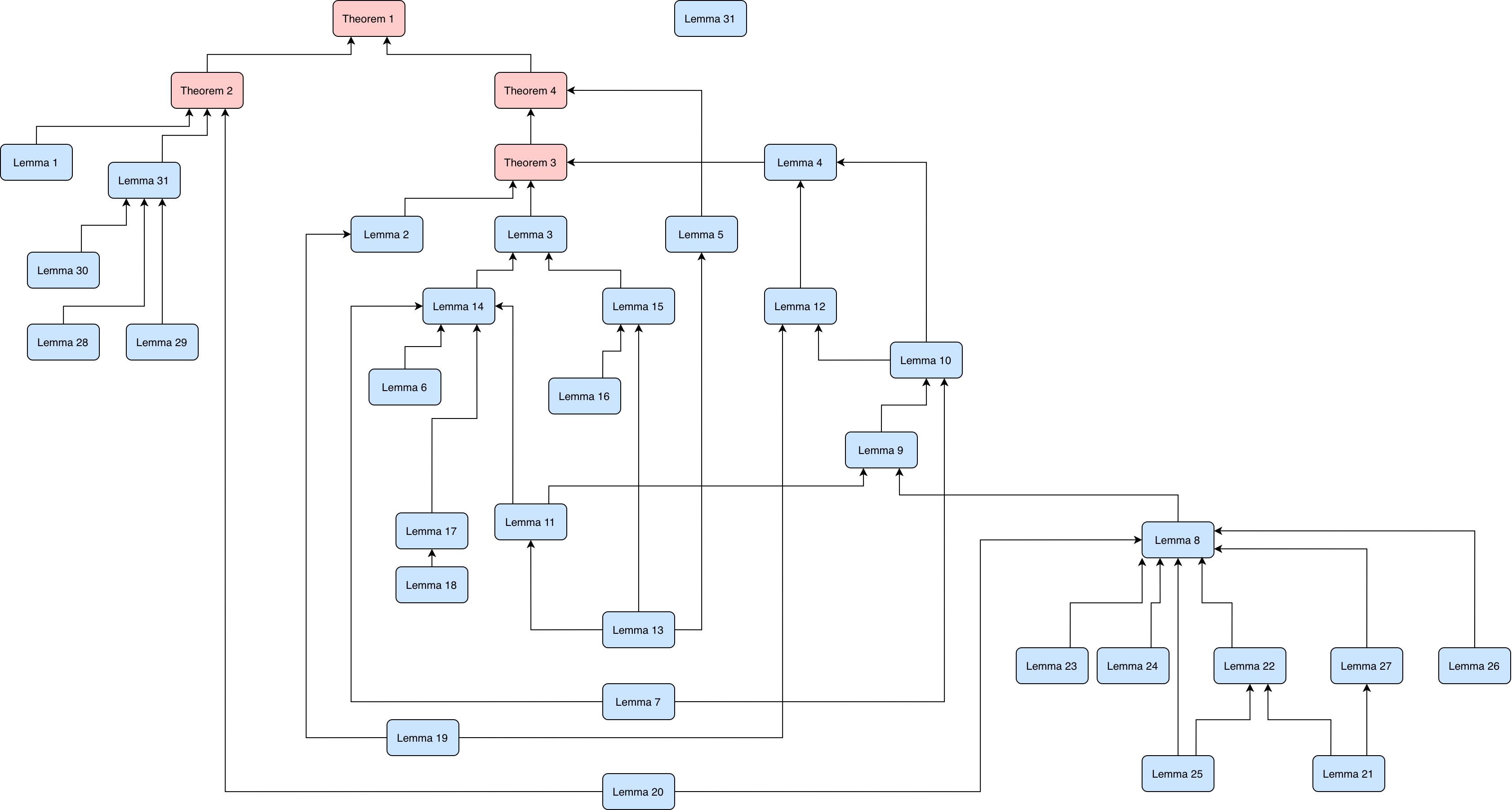}
\subsection{Distributional Assumptions}\label{sec:assumptions}
The following is the assumption that we require to obtain our theoretical results.
\begin{assumption}\label{assumption:distribution}
For all $i$, $\bx_i\sim_{\textnormal{i.i.d}}\mathcal D$, where $\mathcal D$ is a distribution over $\R^p$ with the following properties. For constants $C_1,C_2,C_3,C_4,C_5,C_6$,
\begin{enumerate}
    \item Moment bounds: We assume that the second and fourth moments of $\mathcal D$ are bounded, i.e., if $\bx_i\sim\mathcal D$, then $\|\bSigma\|\le C_1$ (where $\bSigma:=\EE[\bx_i\bx_i^\top]$) and $\|\EE[\vc(\bx_i\bx_i^\top)\vc(\bx_i\bx_i^\top)^\top]\|\le C_2$.
    \item Bounded subgaussian norm: Define the one-dimensional subgaussian norm by $\|x\|_{\psi_2}:=\inf\{K>0:\EE[\exp(x^2/K^2)\le 2\}$. Then, define the subgaussian norm of a random vector by $\|\bx\|_{\psi_2}:=\sup_{\|\bv\|=1}\|\langle \bx,\bv\rangle\|_{\psi_2}$. We assume $\|\bx_i\|_{\psi_2}\le C_3$.
    \item Concentration of norm: We assume $\left\|\|\bx_i\|-\EE[\|\bx_i\|]\right\|_{\psi_2}\le C_4$.
    \item Concentration of quadratic forms: We assume that, for any matrix $\boldsymbol B_1\in\R^{p\times p}$, if $\bx_i,\bx_j\sim_{\textnormal{i.i.d.}}\mathcal D$, then $\left\|\bx_i^\top\boldsymbol B_1\bx_j\right\|_{\psi_1}\le C_5 \|\boldsymbol B_1\|_{\textnormal{F}}$.
    \item We require a lower bound on the expected distance between quadratic forms in the two triplet distances being compared. Specifically, we assume that, for any rank $r$ symmetric matrix $\boldsymbol B_2\in\R^{p\times p}$, \[\EE_{\bx_i,\bx_j,\bx_k\sim_{\textnormal{i.i.d.}}\mathcal D}\left[\left|(\bx_i-\bx_j)^\top\boldsymbol{B}_2(\bx_i-\bx_j)-(\bx_i-\bx_k)^\top\boldsymbol{B}_2(\bx_i-\bx_k)\right|\right]\ge C_6\|\boldsymbol B_1\|_{\textnormal{F}}.\] Intuitively, this captures the requirement that the Mahalanobis distances between pairs of points cannot all be too similar.
\end{enumerate}
\end{assumption}
\subsection{Proof of Theorem \ref{thm:alg_performance}}\label{pf:alg_performance}
By Theorem \ref{thm:init}, the output $\bA_0$ of the initialization step satisfies $\|\bA_0\boldsymbol{O}-\bA_\star\|\lesssim\sqrt{\frac{1}{n\log n}}$ for some orthogonal matrix $\boldsymbol{O}$. Therefore, when we run gradient descent starting at $\bA_0$, the hypothesis of Theorem \ref{thm:gd_induction} is satisfied, so there exists $0<\rho<1$ such that $\|\bA_{w+1}\bO_{w+1}-\bA_\star\|\le \rho\|_\textnormal{F}\bA_{w}\bO_{w}-\bA_\star\|_\textnormal{F}+\frac{C}{n}$ for all $w$ with probability $1-O(T/n^{10})$, where $\bO_w=\argmin_{\bO\in\mathcal O_{r\times r}}\|\bA_w\bO-\bA_\star\|_\textnormal{F}$. This implies that, after running gradient descent for time $T$, we reach an iterate $\bA_T$ that satisfies $\|\bA_T\boldsymbol{O}_T-\bA_\star\|_\textnormal{F}\lesssim \rho^T\sqrt{\frac{1}{n\log n}}+\frac{1}{(1-\rho)n}\lesssim e^{-T}+\frac{1}{n}$. Defining $\widehat{\bK}:=\bA_T\bA_T^\top=\bA_T\boldsymbol{O}_T(\bA_T\boldsymbol{O}_T)^\top$, we have that $\|\bK-\bK_\star\|_\textnormal{F}\lesssim e^{-T}+\frac{1}{n}$. Since the rank is $r=O(1)$, the same bound is true in Euclidean norm.
\subsection{Proof of Theorem \ref{thm:init}}\label{pf:init}
Let $(\boldsymbol{w}_\star^i)_j:=\exp(d_{\bK_\star}(\bx_i,\bx_j))$. To apply the guarantees on Rank Centrality, we need a bound on the largest condition number $$\max_i\left(\frac{\max_j((\boldsymbol{w}_\star^i)_j)}{\min_j((\boldsymbol{w})_\star^i)_j)}\right).$$ 

We run the initialization step on only the subset of individuals with bounded norm to keep the condition number constant. In Algorithm \ref{alg:init}, we take $\bX'$ to be the subset of the individuals in $\bX$ whose norm is below the 90th percentile of all individuals in $\bX$. Observe that $n'=\Theta(n)$. Analogously, we define $S':=\{(i,j,k)\in S\mid i,j,k\in\bX'\}$, and we have that $|S'|=\Theta(n^3)$.

To bound the condition number, we first observe that, for all $\bx\in\bX'$, $\|\bx\|\le 20\EE_{\bx'\sim\mathcal D}[\|\bx'\|]=O(1)$ by Markov's inequality and Hoeffding's inequality. Then, we upper bound $\max_{i,j}((\boldsymbol{w}^i)_j)$ by observing that $d_{\bK_\star}(\bx_i,\bx_j)^2\le\|\bK_\star\|^2\|\bx_i-\bx_j\|$. Since $\|bx_i\|$ and $\|\bx_j\|$ are $O(1)$, we have $d_{\bK}(\bx_i,\bx_j)^2\lesssim\|\bK_\star\|^2$. Therefore $\max_{i,j}((\boldsymbol{w}_\star^i)_j)\lesssim \exp(\|\bK_\star\|^2)$. We have that $\min_{i,j}((\boldsymbol{w}_\star^i)_j)\ge 1$. Therefore, with probability $1-O(e^{-n})$, we have $\max_i\left(\frac{\max_j((\boldsymbol{w}_\star^i)_j)}{\min_j((\boldsymbol{w}_\star^i)_j)}\right)=O(1)$.

Now, we make some definitions. Let $\bD_\star\in\R^{n'\times n'}$ be the true pairwise distance matrix with $D_{i,j}=d_{\bK_\star}(\bx_i,\bx_j)$. For $i\in[n']$, define $\boldsymbol D_\star^i$ to be the $i$-th row of $\boldsymbol D$. Then, define
\[\boldsymbol\pi_\star^i:=\frac{\boldsymbol w_\star^i}{\|\boldsymbol w_\star^i\|_1}.\]

We combine \citet[Theorem 2.6]{rank_centrality_chen} and \citet[Theorem 5.2]{rank_centrality_chen} to conclude that, with probability at least $1-O(n^{-5})$, \textsc{RankCentrality} finds an estimate $\widehat{\boldsymbol\pi}^i$ such that
\begin{equation}
    \frac{\|\widehat{\boldsymbol\pi}^i-\boldsymbol\pi_\star^i\|_\infty}{\|\boldsymbol\pi_\star^i\|_\infty}\lesssim\sqrt{\frac{\log n'}{n's'}}\label{eqn:rank_centrality_1}
\end{equation}
and
\begin{equation}
    \frac{\|\widehat{\boldsymbol\pi}^i-\boldsymbol\pi_\star^i\|}{\|\boldsymbol\pi^i_\star\|}\lesssim\sqrt{\frac{1}{n's'}}.\label{eqn:rank_centrality_2}
\end{equation}

From this, we show the following Lemma:
\begin{lem}\label{lem:l1_norm_ratio}
    For all $i\in[n']$,
    \[\frac{\|\widehat{\boldsymbol\pi}^i-\boldsymbol\pi_\star^i\|_1}{\|\boldsymbol\pi^i_\star\|_1}\lesssim\sqrt{\frac{1}{n's'}}.\]
\end{lem}
\begin{proof}[Proof of Lemma \ref{lem:l1_norm_ratio}]
    We know
    \[\frac{\max_j((\boldsymbol{\pi}_\star^i)_j)}{\min_j((\boldsymbol{\pi}_\star^i)_j)}\lesssim 1\] for all $i$ (because $\boldsymbol\pi_\star^i$ is a scaled version of $\boldsymbol w_\star^i$). Then
    \begin{align*}
        \|\boldsymbol\pi_\star^i\|_1&=\sum_{j\in[n']}(\boldsymbol\pi_\star^i)_j\gtrsim n\max_j((\boldsymbol\pi_\star^i)_j).
    \end{align*}
    Note that \[\|\boldsymbol\pi_\star^i\|_2\le\sqrt{n}\max_j((\boldsymbol\pi_\star^i)_j).\]
    Therefore, $\|\boldsymbol\pi_\star^i\|_1\gtrsim\sqrt{n}\|\boldsymbol\pi_\star^i\|_2$.
    So,
    \begin{align*}
        \frac{\|\widehat{\boldsymbol\pi}^i-\boldsymbol\pi_\star^i\|_1}{\|\boldsymbol\pi^i_\star\|_1}&\lesssim \frac{\|\widehat{\boldsymbol\pi}^i-\boldsymbol\pi_\star^i\|_1}{\sqrt{n}\|\boldsymbol\pi_\star^i\|_2}\le \frac{\sqrt{n}\|\widehat{\boldsymbol\pi}^i-\boldsymbol\pi_\star^i\|_2}{\sqrt{n}\|\boldsymbol\pi_\star^i\|_2}= \frac{\|\widehat{\boldsymbol\pi}^i-\boldsymbol\pi_\star^i\|_2}{\|\boldsymbol\pi_\star^i\|_2}\lesssim \sqrt{\frac{1}{n's'}}.
    \end{align*}
\end{proof}
Now, for each $i$, we form the vector $\widehat{\boldsymbol p}^i$ by setting $(\widehat{\boldsymbol p}^i)_j:=\log((\widehat{\boldsymbol\pi}^i)_j)$. Similarly, define $\boldsymbol p_\star^i$ by setting $(\boldsymbol p_\star^i)_j:=\log((\boldsymbol\pi_\star^i)_j)$. By the definitions above, $\boldsymbol p_\star^i$ is a shifted version of the $i$-th row of the distance matrix $\boldsymbol D_\star$:
\[\boldsymbol p_\star^i=\boldsymbol D_\star^i+\log(\|\boldsymbol w_\star^i\|_1)\boldsymbol 1.\] Taking the log preserves the closeness of our estimate, which we show in the following.
\begin{align*}
    (\widehat{\boldsymbol p}^i)_j&=\log((\widehat{\boldsymbol\pi}^i)_j)\\
    &=\log(({\boldsymbol\pi}_\star^i)_j+((\widehat{\boldsymbol\pi}^i)_j-({\boldsymbol\pi}_\star^i)_j))\\
    &=\log\left(({\boldsymbol\pi}_\star^i)_j\left(1+\frac{(\widehat{\boldsymbol\pi}^i)_j-({\boldsymbol\pi}_\star^i)_j}{({\boldsymbol\pi}_\star^i)_j}\right)\right)\\
    &=\log({\boldsymbol\pi}_\star^i)_j+\log\left(1+\frac{(\widehat{\boldsymbol\pi}^i)_j-({\boldsymbol\pi}_\star^i)_j}{({\boldsymbol\pi}_\star^i)_j}\right)\\
    &=(\boldsymbol p_\star^i)_j+\log\left(1+\frac{(\widehat{\boldsymbol\pi}^i)_j-({\boldsymbol\pi}_\star^i)_j}{({\boldsymbol\pi}_\star^i)_j}\right).
\end{align*}
Therefore, 
\begin{equation}\label{eqn:log_bound}
    \left|(\widehat{\boldsymbol p}^i)_j-(\boldsymbol p_\star^i)_j\right|\le \left|\log\left(1+\frac{(\widehat{\boldsymbol\pi}^i)_j-({\boldsymbol\pi}_\star^i)_j}{({\boldsymbol\pi}_\star^i)_j}\right)\right|.
\end{equation}

Equation \eqref{eqn:rank_centrality_1} implies that 
\begin{equation}
    \frac{|(\widehat{\boldsymbol\pi}^i)_j-({\boldsymbol\pi}_\star^i)_j|}{\|{\boldsymbol\pi}_\star^i\|_\infty}\lesssim\sqrt{\frac{\log n'}{n's'}}.\label{eqn:pihat_to_pistar_1}
\end{equation} 
Furthermore we have
\begin{equation}
    \frac{\|{\boldsymbol\pi}_\star^i\|_\infty}{\min_j(({\boldsymbol\pi}_\star^i)_j)}=\frac{\max_j(({\boldsymbol\pi}_\star^i)_j)}{\min_j(({\boldsymbol\pi}_\star^i)_j)}\le \exp(4\|\bK_\star\|)=\Theta(1).\label{eqn:pihat_to_pistar_2}
\end{equation}
Combining \eqref{eqn:pihat_to_pistar_1} and \eqref{eqn:pihat_to_pistar_2}, we have that 
\[\frac{|(\widehat{\boldsymbol\pi}^i)_j-({\boldsymbol\pi}_\star^i)_j|}{({\boldsymbol\pi}_\star^i)_j}\lesssim\sqrt{\frac{\log n'}{n's'}}.\] Recalling \eqref{eqn:log_bound}, we have
\begin{align*}
    \left|(\widehat{\boldsymbol{p}}^i)_j-(\boldsymbol p_\star^i)_j\right|\le\left|\log\left(1+\frac{(\widehat{\boldsymbol\pi}^i)_j-({\boldsymbol\pi}_\star^i)_j}{({\boldsymbol\pi}_\star^i)_j}\right)\right|&\le \frac{|(\widehat{\boldsymbol\pi}^i)_j-({\boldsymbol\pi}_\star^i)_j|}{({\boldsymbol\pi}_\star^i)_j}\lesssim\sqrt{\frac{\log n'}{n's'}},
\end{align*}
where the first inequality holds because $|\log(1+x)|\le|x|$ when $|x|\le0.5$

Define $\widehat{\boldsymbol P}$ to be the matrix whose $(i,j)$-th entry is $(\widehat{\boldsymbol{p}}^i)_j$ and define $\boldsymbol P_\star$ to be the matrix whose $(i,j)$-th entry is $({\boldsymbol{p}}_\star^i)_j$. Then, 
\[\frac{1}{n'}\|\widehat{\boldsymbol P}-\boldsymbol P_\star\|\le\left|(\widehat{\boldsymbol{p}}^i)_j-(\boldsymbol p_\star^i)_j\right|\lesssim\sqrt{\frac{\log n'}{n's'}}.\]

Algorithm \ref{alg:init} then applies a centering matrix $\boldsymbol J:=\boldsymbol I_n-(\boldsymbol 1\boldsymbol 1^\top/n)$ to remove the translational ambiguity. Let $\boldsymbol\Delta$ be the matrix with $i$-th row given by $\log(\|\boldsymbol w^i_\star\|_1)\boldsymbol 1$. We have
\begin{align*}
    \frac{1}{n'}\boldsymbol J\widehat{\boldsymbol P}\boldsymbol J&=\frac{1}{n'}\boldsymbol J\boldsymbol P_\star\boldsymbol J+\frac{1}{n'}\boldsymbol J(\widehat{\boldsymbol P}-\boldsymbol P_\star)\boldsymbol J\\
    &=\frac{1}{n'}\boldsymbol J\boldsymbol D_\star\boldsymbol J-\frac{1}{n'}\boldsymbol J\boldsymbol \Delta\boldsymbol J+\frac{1}{n'}\boldsymbol J(\widehat{\boldsymbol P}-\boldsymbol P_\star)\boldsymbol J.
\end{align*}

Since $\boldsymbol\Delta$'s rows are uniform, we have $\boldsymbol J\boldsymbol\Delta\boldsymbol J=0$. 

Observing that $\|\boldsymbol J\|=1$, we have
\begin{align}\nonumber
    \left\|\frac{1}{n'}\boldsymbol J(\widehat{\boldsymbol P}-\boldsymbol P_\star)\boldsymbol J\right\|
    &\le \left\|\frac{1}{n'}(\widehat{\boldsymbol P}-\boldsymbol P_\star)\right\|\lesssim\sqrt{\frac{1}{n'\log n'}}
\end{align}

Therefore
\[\frac{1}{n'}\|\boldsymbol J\widehat{\boldsymbol P}\boldsymbol J-\boldsymbol J\boldsymbol D_\star\boldsymbol J\|\lesssim\sqrt{\frac{1}{{n'}\log n'}}.\]

The reason that $\boldsymbol J\boldsymbol D_\star\boldsymbol J$ is a useful quantity is given by \citet[Theorem 1]{zhang_wahba_yuan}. This theorem states that, if $\bA_\star^\top\bX^\top$ is centered (i.e. $\bA_\star^\top\bx_1+\cdots+\bA_\star^\top\bx_n=\boldsymbol 0$), then \[\bX\bA_\star\bA_\star^\top\bX^\top=-\frac{\boldsymbol J\bD_\star\boldsymbol J}{2},\] where $\boldsymbol J:=\boldsymbol I_n-(\boldsymbol 1\boldsymbol 1^\top/n)$. Since we center $\bX$ before running the initialization, $\bA_\star^\top\bX^\top$ is also centered, so the theorem applies.

Therefore,
\begin{align*}
    \frac{1}{n'}\left\|-\frac{1}{2}\boldsymbol J\widehat{\bD}\boldsymbol J-\bX\bA_\star\bA_\star^\top\bX^\top\right\|&\lesssim\sqrt{\frac{1}{{n'}\log n'}}.
\end{align*}

To simplify notation, write $\boldsymbol H_\star:=-\frac{1}{2}\boldsymbol J\bD_\star\boldsymbol J$ and $\widehat{\boldsymbol H}:=-\frac{1}{2}\boldsymbol J\widehat{\boldsymbol P}\boldsymbol J$. Using this notation, we have shown that 
\begin{equation}\label{eqn:hhat_to_h}
    \frac{1}{n'}\|\widehat{\boldsymbol H}-\boldsymbol H_\star\|\lesssim\sqrt{\frac{1}{{n'}\log n'}}.
\end{equation}

Now we show how Algorithm \ref{alg:init} recovers $\bA_\star\bA_\star^\top$ from $\boldsymbol H_\star$. Let $\bU_\star\bLambda_\star\bU_\star^\top$ be the SVD of $\bA_\star\bA_\star^\top$. We can make the following manipulations:
\begin{align*}
    \boldsymbol H_\star&=\bX\bA_\star\bA_\star^\top\bX^\top\\
    \bX^\top\boldsymbol H_\star\bX&=\bX^\top\bX\bA_\star\bA_\star^\top\bX^\top\bX\\
    \bX^\top\boldsymbol H_\star\bX&=\bPsi\bU_\star\bLambda_\star\bU_\star^\top\bPsi,
\end{align*}
where $\bPsi:=\bX^\top\bX$.

Let $\widetilde\bU\in\R^{p\times r}$ be the $r$ largest generalized eigenvectors of the pair $(\frac{1}{{n'}^2}\bX^\top\boldsymbol H_\star\bX,\frac{1}{{n'}}\bPsi)$, i.e. $\widetilde\bU$ satisfies
\[\frac{1}{{n'}^2}\bX^\top\boldsymbol H_\star\bX\widetilde\bU=\frac{1}{{n'}}\bPsi\widetilde\bU\widetilde{\bLambda},\] where $\widetilde{\bLambda}$ is a diagonal matrix sorted in nonincreasing order.

Finding $\widetilde\bU$ and $\widetilde{\bLambda}$ is sufficient to recover $\bA_\star\bA_\star^\top$ because, by Lemma \ref{lem:utilde_lambdatilde_utilde}, we have
$\bA_\star\bA_\star^\top=\widetilde\bU\widetilde{\bLambda}\widetilde\bU^\top$. Algorithm \ref{alg:init} finds approximate versions of $\widetilde\bU$ and $\widetilde{\bLambda}$. In particular, it finds $\widehat{\bU}$ and $\widehat{\bLambda}$, the generalized eigenvalues of $(\frac{1}{{n'}^2}\bX^\top\widehat{\boldsymbol{H}}\bX,\frac{1}{n'}\bPsi)$. By Lemma \ref{lem:approx_aat}, there exists an orthogonal matrix $\boldsymbol O$ such that
\[\|\widehat{\bU}\widehat{\bLambda}^{1/2}-\bA_\star\boldsymbol O\|\lesssim\sqrt{\frac{1}{n'\log n'}}.\] This implies
\[\|\widehat{\bU}\widehat{\bLambda}^{1/2}\boldsymbol O^\top-\bA_\star\|\lesssim\sqrt{\frac{1}{n'\log n'}}.\] Finally, we use the fact that $n=O(n')$ to obtain the desired bound of $\sqrt{\frac{1}{n\log n}}$

\subsection{Proof of Theorem \ref{thm:hessian}}\label{pf:hessian}
First we can find epressions for the gradient and Hessian of the loss function $L$. We can write $L$ as
\[L(\bA)=\frac{1}{|S|}\sum_{t\in S}\log(1+\exp(-y_t\Tr(\bM_t\bA\bA^\top))).\]
Then we can compute the gradient:
\begin{align}\nonumber
    \nabla L(\bA)&=-\frac{1}{|S|}\sum_{t\in S}\vc\left(\frac{\exp(-y_t\Tr(\bM_t\bA\bA^\top))\cdot y_t(\bM_t\bA+\bM_t^\top\bA)}{1+\exp(-y_t\Tr(\bM_t\bA\bA^\top))}\right)\\
    &=-\frac{1}{|S|}\sum_{t\in S}\vc\left(\frac{2y_t\bM_t\bA}{\exp(y_t\Tr(\bM_t\bA\bA^\top))+1}\right).\label{eqn:gradient}
\end{align}
For the Hessian:
\begin{align}\nonumber
    &\nabla^2 L(\bA)\\\nonumber
    &=-\frac{1}{|S|}\sum_{t\in S}\nabla\left(\frac{2y_t\vc(\bM_t\bA)}{\exp(y_t\Tr(\bM_t\bA\bA^\top))+1}\right)\\\nonumber
    &=-\frac{1}{|S|}\sum_{t\in S}\left(\frac{2y_t}{{(\exp(y_t\Tr(\bM_t \bA\bA^\top))+1)^2}}\nabla\vc(\bM_t\bA)\right.\\\nonumber
    &\hspace{15mm} \left.-\frac{2y_t}{{(\exp(y_t\Tr(\bM_t \bA\bA^\top))+1)^2}}\nabla(\exp(y_t\Tr(\bM_t \bA\bA^\top))+1)\otimes\vc(\bM_t\bA))\right)\\\nonumber
    &=\frac{1}{|S|}\sum_{t\in S}\left(\underbrace{-\frac{2y_t}{{(\exp(y_t\Tr(\bM_t \bA\bA^\top))+1)^2}}\>\bI_r\otimes\bM_t}_{\alpha_1}\right.\\
    &\hspace{25mm} \left.+\underbrace{\frac{4\exp(y_t\Tr(\bM_t \bA\bA^\top))}{(\exp(y_t\Tr(\bM_t \bA\bA^\top))+1)^2}\>(\vc(\bM_t\bA)^\top\otimes\vc(\bM_t\bA))}_{\alpha_2}\right). \label{eqn:hessian}
\end{align}

\paragraph{Intuition} Because of the symmetry in $\bM_t$, $\alpha_1$ is 0 in expectation. We make a concentration bound to show that $\sum_{t\in S}\alpha_1$ is small with high probability. The concentration bound works by, first, observing that the sum over all possible triplets $\oS$ is exactly 0, then by showing that the sum over the sampled triplets $S$ is close. This is made precise in Lemma \ref{lem:alpha_1}.

Since $\alpha_1$ is small, the properties of the Hessian are effectively determined by the properties of $\alpha_2$. To prove our strong convexity property, we have to show a lower bound on $\vc(\widetilde{\bZ})^\top\nabla^2 L(\bA)\vc(\widetilde{\bZ})$, where $\widetilde{\bZ}$ is in the form specified in Theorem \ref{thm:hessian}. This form of matrix is designed to remove the rotational ambiguity in the problem of learning $\bA_\star$.

We show that the first subterm of $\alpha_2$, $\frac{4\exp(y_t\Tr(\bM_t \bA\bA^\top))}{(\exp(y_t\Tr(\bM_t \bA\bA^\top))+1)^2}$, is lower bounded for many $t$ using standard random matrix bounds. Then, we show that the second subterm of $\alpha_2$, $(\vc(\bM_t\bA)^\top\otimes\vc(\bM_t\bA))$, is lower bounded for many $t$ using the Paley-Zygmund inquality and a lower bound on the expectation of this term. The latter is provided by Lemma \ref{lem:expected_abs_value_trace_astar}, which crucially uses the properties of $\widetilde{\bZ}$, which is a solution to an orthogonal procrustes rotation problem \citep{procrustes}. Then, we invoke the pigeonhole principle to conclude that, for many $t$, both subterms of $\alpha_2$ are lower bounded, so $\sum_{t\in S}\alpha_2$ is lower bounded.

Finally, we have to show an upper bound on $\alpha_2$. The first subterm is upper bounded by a constant, so we restrict attention to $(\vc(\bM_t\bA)^\top\otimes\vc(\bM_t\bA))$. We first upper bound the expectation using the moment bounds we assume in the distributional Assumption \ref{assumption:distribution}. Then, we show a concentration bound in Lemma \ref{lem:concentration_tr2_astar}. Since $\bM_t$ has six terms, there are 36 cross terms in $(\vc(\bM_t\bA)^\top\otimes\vc(\bM_t\bA))$. The concentration bound works by bounding each of these cross terms using high-dimensional concentration bounds.

\paragraph{Proof} First, we will show that the contribution from term $\alpha_1$ is small. By Lemma \ref{lem:alpha_1}, for any matrix $\bV\in\R^{p\times r}$, we have 
\begin{align*}
    \left|\frac{1}{|S|}\sum_{t\in S}\vc(\bV)^\top\alpha_1\vc(\bV)\right|\lesssim\frac{1}{n^{1/2}}\|\bV\|_\textnormal{F}^2.
\end{align*}
with probability at least $1-O(\exp(-n^{1/4}))$.

Now we turn to $\alpha_2$, the properties of which will determine the properties of the Hessian.

Towards showing strong convexity, we seek to lower bound
$\vc(\widetilde{\bZ})^T\alpha_2\vc(\widetilde{\bZ})$. By Lemma \ref{lem:alpha_2_lower}, for any matrix $\widetilde{\bZ}\in\R^{p\times r}$ of the form $\widetilde{\bZ}=\bZ\bH_{\bZ}-\bA_\star$,
    \begin{align*}
        \frac{1}{|S|}\sum_{t\in S}\vc(\widetilde{\bZ})^\top\alpha_2\vc(\widetilde{\bZ})&\gtrsim \|\widetilde\bZ\|_\textnormal{F}^2
    \end{align*} with probability at least $1-O(n^{-10})$.

Towards showing smoothness, we want to upper bound
$\vc(\bV)^\top\alpha_2\vc(\bV)$. By Lemma \ref{lem:alpha_2_upper}, for for any matrix $\bV\in\R^{p\times r}$,
\begin{align*}
    \frac{1}{|S|}\sum_{t\in S}(\Tr(\bA^\top\bM_t\bV))^2\lesssim\|\bV\|_\textnormal{F}^2
\end{align*}
with probability at least $1-O(n^{-10})$.

Combining our upper bound on $\alpha_1$ with our lower bound on $\alpha_2$, our lower bound on the Hessian is 
\[\vc(\widetilde{\bZ})^\top\nabla^2L(\bA)\vc(\widetilde{\bZ})\ge C_2\|\widetilde{\bZ}\|_\textnormal{F}^2,\] proving (\ref{eqn:hessian_strong_convexity}).

Similarly, combining our upper bound on $\alpha_1$ with our upper bound on $\alpha_2$ leads to the upper bound on the Hessian
\begin{align*}
    \vc(\bV)^\top\nabla^2 L(\bA)\vc(\bV)\le C_1\|\bV\|_\textnormal{F}^2,
\end{align*} proving (\ref{eqn:hessian_smoothness}).

This completes our proof of Theorem \ref{thm:hessian}.

\subsection{Proof of Theorem \ref{thm:gd_induction}}\label{pf:gd_induction}

We can write
\begin{align}\nonumber
    &\hspace{-5mm}\|\bA_{w+1}\bO_{w+1}-\bA_\star\|_\textnormal{F}\\\nonumber
    &\le\|\bA_{w+1}\bO_w-\bA_\star\|_\textnormal{F}&\text{definition of $\bO_w$}\\\nonumber
    &=\|(\bA_w-\eta\nabla L(\bA_w))\bO_w-\bA_\star\|_\textnormal{F}&\text{definition of $\bA_{w+1}$}\\\nonumber
    &=\|(\bA_w-\eta\nabla L(\bA_w))-\bA_\star\bO_w^\top\|_\textnormal{F}&\text{orthogonal invariance of norm}\\\nonumber
    &\le\|(\bA_w-\eta\nabla L(\bA_w))-(\bA_\star\bO_w^\top-\eta\nabla L(\bA_\star\bO_w^\top))\|_\textnormal{F}+\frac{\eta C_1}{n}&\text{Lemma \ref{lem:empirical_gradient}}\\
    &=\left\|\left(\bI_{pr}-\eta\int_0^1\nabla^2 L(\bA(\tau))\,\dd\tau\right)\vc(\bA_w-\bA_\star\bO_w^\top)\right\|+\frac{\eta C_1}{n}\label{eqn:gd_bound_breakdown}
\end{align}
where the last step is due to the fundamental theorem of calculus and we have defined \[\bA(\tau):=\bA_\star\bO_w^\top+\tau(\bA_w-\bA_\star\bO_w^\top).\]

Define $\bL:=\int_0^1\nabla^2 L(\bA(\tau))\,\dd\tau$. We begin with the first term of \eqref{eqn:gd_bound_breakdown}. Starting by squaring,
\begin{align*}
    &\hspace{-5mm}\left\|\left(\bI_{p}-\eta\bL\right)\vc(\bA_w-\bA_\star\bO_w^\top)\right\|^2\\
    &=\vc(\bA_w-\bA_\star\bO_w^\top)^\top(\bI_{p}-\eta\bL)^2\vc(\bA_w-\bA_\star\bO_w^\top)\\
    &=\vc(\bA_w-\bA_\star\bO_w^\top)^\top(\bI_{p}-2\eta\bL+\eta^2\bL^2)\vc(\bA_w-\bA_\star\bO_w^\top)\\
    &\le\|\bA_w-\bA_\star\bO_w^\top\|^2_\textnormal{F}-2\eta\vc(\bA_w-\bA_\star\bO_w^\top)^\top\bL\vc(\bA_w-\bA_\star\bO_w^\top)+\eta^2\|\bL\|^2\cdot\|\bA_w-\bA_\star\bO_w^\top\|^2_\textnormal{F}
\end{align*}

Since $\|\bA_w-\bA_\star\bO_{w}^\top\|_\textnormal{F}\le C_2\sqrt{\frac{1}{n\log n}}$, we know that $\|\bA(\tau)-\bA_\star\bO_w^\top\|_\textnormal{F}\le C_2\sqrt{\frac{1}{n\log n}}$ for all $\tau\in[0,1]$. The same is true in Euclidean norm (up to a factor of $r$, which is a constant), so we can apply Theorem \ref{thm:hessian} with $\bA$ set to $\bA(\tau)$, $\bZ$ set to $\bA_w$, and $\bA_\star$ set to $\bA_\star\bO_w^\top$. Note that the quantity $\bH_{\bZ}:=\argmin_{\bO\in\mathcal O_{r\times r}}\|\bA_w\bO-\bA_\star\bO_w^\top\|_\textnormal{F}$ in the conclusion of Theorem \ref{thm:hessian} is simply $\boldsymbol{I}_r$ by definition of $\bO_w^\top$.

By Theorem \ref{thm:hessian},
\begin{align*}
    \vc(\bA_w-\bA_\star\bO_w^\top)^\top\bL\vc(\bA_w-\bA_\star\bO_w^\top)\ge C_4\|\bA_w-\bA_\star\bO_w^\top\|^2_\textnormal{F}= C_3\|\bA_w\bO_w-\bA_\star\|^2_\textnormal{F}.
\end{align*}
Also by Theorem \ref{thm:hessian}, we have that $\|\bL\|^2\le C_4.$
Therefore we have
\begin{align*}
    \left\|\left(\bI_{p}-\eta\bL\right)\vc(\bA_w-\bA_\star\bO_w^\top)\right\|^2
    &\le\left(1-2\eta C_3+\eta^2 C_4\right)\|\bA_w\bO_w-\bA_\star\|^2_\textnormal{F}
\end{align*}
As long as $\eta\le C_3/C_4$, we have
\begin{align*}
    \left\|\left(\bI_{p}-\eta\bL\right)\vc(\bA_w-\bA_\star\bO_w^\top)\right\|^2\le\left(1-\frac{C_3^2}{C_4}\right)\|\bA_w\bO_w-\bA_\star\|^2_\textnormal{F}.
\end{align*}
This concludes our analysis of the first term in \eqref{eqn:gd_bound_breakdown}. Considering both terms, we have
\begin{align*}
    \|\bA_{w+1}\bO_{w+1}-\bA_\star\|&\le \sqrt{1-\frac{C_4^2}{C_5}}\|_\textnormal{F}\bA_w\bO_w-\bA_\star\|_\textnormal{F}+\frac{\eta C_1}{n}
\end{align*}

Taking $\rho:=\sqrt{1-C_4^2/C_5}$, we obtain
\begin{align*}
    \|\bA_{w+1}\bO_{w+1}-\bA_\star\|_\textnormal{F}&\le\rho\|\bA_w\bO_w-\bA_\star\|_\textnormal{F}+\frac{\eta C_1}{n},
\end{align*}
concluding the proof of Theorem \ref{thm:gd_induction}.

\section{Key Lemmas}\label{sec:key_lemmas}
\begin{lem}\label{lem:alpha_1}
    Under the conditions of Theorem \ref{thm:hessian}, for any matrix $\bV\in\R^{p\times r}$,
    \begin{align*}
        \left|\frac{1}{|S|}\sum_{t\in S}\vc(\bV)^\top\alpha_1\vc(\bV)\right|\lesssim\frac{1}{n^{1/2}}\|\bV\|_\textnormal{F}^2
    \end{align*}
    with probability at least $1-O(\exp(-n^{1/4}))$.
\end{lem}
\begin{proof}

Recall that
\[\alpha_1=-\frac{2y_t}{{(\exp(y_t\Tr(\bM_t \bA\bA^\top))+1)^2}}\>\bI_r\otimes\bM_t.\] Note that $(\exp(y_t\Tr(\bM_t\bA\bA^\top)^2\ge 0$ and $|y_t|=1$, so \[|\vc(\bV)^\top \alpha_1\vc(\bV)|\le 2|\vc(\bV)^\top(\bI_r\otimes \bM_t)\vc(\bV)|=2\Tr(\bV^\top\bM_t\bV).\]

If we sum over all possible triplets $\oS=\{(i,j,k)\mid i\neq j\neq k\neq i\in[n]\}$, we get
\begin{align*}
    &\frac{1}{|\oS|}\sum_{t\in\oS}\Tr(\bV^\top\bM_t\bV)\\
    &=\frac{1}{|\oS|}\sum_{t\in\oS}(\bx_i^\top\bV\bV^\top\bx_k+\bx_k^\top\bV\bV^\top\bx_i-\bx_i^\top\bV\bV^\top\bx_j-\bx_j^\top\bV\bV^\top\bx_i+\bx_j^\top\bV\bV^\top\bx_j-\bx_k^\top\bV\bV^\top\bx_k).
\end{align*}
By the symmetry of $j$ and $k$, this is 0. 

For a triplet $t$, let $b_t$ be an indicator random variable that is 1 if and only if $t\in S$. By assumption, $b_t\sim\text{Bernoulli}(s)$, independent of all other $b_{t'}$. Then we can write
\begin{align*}
    &\frac{1}{|S|}\sum_{t\in S}\Tr(\bV^\top\bM_t\bV)\\
    &=\frac{1}{|S|}\sum_{t\in\oS}b_t\Tr(\bV^\top\bM_t\bV)\\
    &=\left(\frac{1}{|S|}\sum_{t\in\oS}(b_t-s)\Tr(\bV^\top\bM_t\bV)\right)+\left(\frac{s}{|S|}\sum_{t\in\oS}\Tr(\bV^\top\bM_t\bV)-\frac{1}{|\oS|}\sum_{t\in\oS}\Tr(\bV^\top\bM_t\bV)\right)\\
    &\hspace{10mm}+\frac{1}{|\oS|}\sum_{t\in\oS}\Tr(\bV^\top\bM_t\bV)\\
    &=\underbrace{\left(\frac{1}{|S|}\sum_{t\in\oS}(b_t-s)\Tr(\bV^\top\bM_t\bV)\right)}_{\eta_1}+\underbrace{\left(\frac{s}{|S|}\sum_{t\in\oS}\Tr(\bV^\top\bM_t\bV)-\frac{1}{|\oS|}\sum_{t\in\oS}\Tr(\bV^\top\bM_t\bV)\right).}_{\eta_2}
\end{align*}
Therefore,
\[\left|\frac{1}{|S|}\sum_{t\in S}\Tr(\bV^\top\bM_t\bV)\right|\le|\eta_1|+|\eta_2|.\]
To bound $\eta_1$, we can write
\begin{align*}
    |\eta_1|&\le\frac{1}{|S|}\left(\max_{t\in\oS}\Tr(\bV^\top\bM_t\bV)\right)\left|s|\oS|-\sum_{t\in\oS}b_t\right|\\
    &=\frac{1}{|S|}\left(\max_{t\in\oS}\Tr(\bV^\top\bM_t\bV)\right)\left|s|\oS|-|S|\right|\\
\end{align*}
Since $|S|$ follows a binomial distribution with mean $s|\oS|=s|\oS|$, we have
\begin{equation}
    \Pr\left[\Big||S|-s|\oS|\Big|\ge q_1s|\oS|\right]\lesssim \exp(-q_1^2|\oS|)\lesssim \exp(-q_1^2n^3).\label{eqn:trvmtv_binomial}
\end{equation}
This also means that, with probability at least $1-O(\exp(-q_1^2n^3))$, we have
\begin{align}
    \left|\frac{1}{|S|}-\frac{1}{s|\oS|}\right|&=\left|\frac{s|\oS|-|S|}{|S|s|\oS|}\right|\le\frac{q_1}{|S|}.\label{eqn:trvmtv_sizeS}
\end{align}

By Lemma \ref{lem:trace_bound_nonsymmetric}, we know
\[\Pr\left[\Tr(\bV^\top\bM_t\bV)\ge q_2\|\bV\|_\textnormal{F}^2\right]\lesssim\exp(-q_2).\] Applying a union bound, this means
\begin{equation}\label{eqn:trvmtv_max}
    \Pr\left[\max_{t\in{\oS}}\Tr(\bV^\top\bM_t\bV)\ge q_2\|\bV\|_\textnormal{F}^2\right]\lesssim n^3\exp(-q_2).
\end{equation}

Using the union bound on the bounds (\ref{eqn:trvmtv_binomial}),  (\ref{eqn:trvmtv_sizeS}), and \eqref{eqn:trvmtv_max}, we have with probability at least $1-O(\exp(-q_1^2n^3)+n^3\exp(-q_2))$, we have
\begin{align*}
    |\eta_1|&\le \frac{1}{|S|}\left(\max_{t\in\oS}\Tr(\bV^\top\bM_t\bV)\right)\left|s|\oS|-|S|\right|\\
    &\le \frac{1}{s|\oS|}\left(\max_{t\in\oS}\Tr(\bV^\top\bM_t\bV)\right)\left|s|\oS|-|S|\right|\\
    &\hspace{20mm}+\left(\frac{1}{|S|}-\frac{1}{s|\oS|}\right)\left(\max_{t\in\oS}\Tr(\bV^\top\bM_t\bV)\right)\left|s|\oS|-|S|\right|\\
    &\le q_1q_2\|\bV\|^2_\textnormal{F}+q_1|\eta_1|\\
    (1-q_1)|\eta_1|&\le q_1q_2\|\bV\|^2_\textnormal{F}\\
    |\eta_1|&\le\frac{q_1q_2\|\bV\|^2_\textnormal{F}}{1-q_1}.
\end{align*}
To bound $|\eta_2|$, we have
\begin{align*}
    \eta_2&=\left(\frac{s}{|S|}-\frac{1}{|\oS|}\right)\sum_{t\in\oS}\Tr(\bV^\top\bM_t\bV)=0,
\end{align*}
because $\sum_{t\in\oS}\Tr(\bV^\top\bM_t\bV)=0$.

Combining the bounds on $|\eta_1|$ and $|\eta_2|$, we see that, with probability at least $1-O(\exp(-q_1^2n^3)+n^3\exp(-q_2))$
we have
\begin{align}\nonumber
    \left|\frac{1}{|S|}\sum_{t\in S}\Tr(\bV^\top\bM_t\bV)\right|&\le \frac{q_1q_2\|\bV\|^2_\textnormal{F}}{1-q_1}
\end{align}

Setting $q_1=1/n$, $q_2=n^{1/2}$, we have that with probability at least $1-O(\exp(-n^{1/4}))$,
\begin{align}\nonumber
    \left|\frac{1}{|S|}\sum_{t\in S}\Tr(\bV^\top\bM_t\bV)\right|&\le Cn^{-1/2}\|\bV\|_\textnormal{F}^2
    \lesssim n^{-1/2}\|\bV\|^2_\textnormal{F}.
\end{align}

\end{proof}

\begin{lem}\label{lem:alpha_2_lower}
    Under the conditions of Theorem \ref{thm:hessian}, for any matrix $\widetilde{\bZ}\in\R^{p\times r}$ of the form $\widetilde{\bZ}=\bZ\bH_{\bZ}-\bA_\star$,
    \begin{align*}
        \frac{1}{|S|}\sum_{t\in S}\vc(\widetilde{\bZ})^\top\alpha_2\vc(\widetilde{\bZ})&\gtrsim \|\widetilde\bZ\|_\textnormal{F}^2
    \end{align*} with probability at least $1-O(n^{-10})$.
\end{lem}
\begin{proof}
We can rewrite $\vc(\widetilde{\bZ})^\top\alpha_2\vc(\widetilde{\bZ})$ as
\begin{align}\nonumber
    &\sum_{t\in S}\vc(\widetilde{\bZ})^\top \frac{4\exp(y_t\Tr(\bM_t \bA\bA^\top))}{(\exp(y_t\Tr(\bM_t \bA\bA^\top))+1)^2}\>(\vc(\bM_t\bA)^\top\otimes\vc(\bM_t\bA))\vc(\widetilde{\bZ})\\\nonumber
    &=\sum_{t\in S}\frac{4\exp(y_t\Tr(\bM_t \bA\bA^\top))}{(\exp(y_t\Tr(\bM_t \bA\bA^\top))+1)^2}\vc(\widetilde{\bZ})^\top\vc(\bM_t\bA)\vc(\bM_t\bA)^\top\vc(\widetilde{\bZ})\\
    &=\sum_{t\in S}\frac{4\exp(y_t\Tr(\bM_t \bA\bA^\top))}{(\exp(y_t\Tr(\bM_t \bA\bA^\top))+1)^2}(\Tr(\bA^\top\bM_t\widetilde{\bZ}))^2\label{eqn:hessian_ii_lower}
\end{align}

Now, we show that there exist constants $C_2,C_3,C_4,C_5$ such that, with probability at least $1-O(n^{-10})$,
\begin{enumerate}[(a)]
    \item At least a $C_2\kappa^{-2}$ fraction of the triplets satisfy
    \[(\Tr(\bM_t\widetilde\bZ\bA^\top))^2\ge C_3\|\bA_\star\|^2\|\widetilde\bZ\|^2_\textnormal{F}.\]
    \item At least a $1-\frac{C_2}{2}\kappa^{-2}$ of the triplets satisfy
        \begin{align}\nonumber
            \frac{4\exp(y_t\Tr(\bM_t \bA\bA^\top))}{(\exp(y_t\Tr(\bM_t \bA\bA^\top))+1)^2}\ge \frac{1}{\exp\left(C_4\|\bA_\star\|^2_\textnormal{F}\right)}.
        \end{align}
\end{enumerate}

Condition (a) follows from Lemma \ref{lem:lower_bound_a}. Condition (b) follows from Lemma \ref{lem:lower_bound_b}. %

Using the pigeonhole principle, we conclude that, with probability at least $1-O(n^{-10})$, there are at least a $\frac{C_2}{8}\kappa^{-2}$ fraction of $t\in S$ where (a) and (b) hold.

For such $t$, we have that
\begin{align*}
    \frac{4\exp(y_t\Tr(\bM_t \bA\bA^\top))}{(\exp(y_t\Tr(\bM_t \bA\bA^\top))+1)^2}(\Tr(\bM_t\widetilde\bZ\bA^\top))^2&\ge \frac{C_3\|\bA_\star\|^2\|\widetilde\bZ\|^2_\textnormal{F}}{\exp(C_4\|\bA_\star\|^2_\textnormal{F})}.
\end{align*}
Since $\|\bA_\star\|$ is $\Theta(1)$,
\begin{align*}
    \frac{4\exp(y_t\Tr(\bM_t \bA\bA^\top))}{(\exp(y_t\Tr(\bM_t \bA\bA^\top))+1)^2}(\Tr(\bM_t\widetilde\bZ\bA^\top))^2\ge C_8\|\widetilde\bZ\|_\textnormal{F}^2.
\end{align*}

Since such triplets make up at least a $\frac{C_2}{8}\kappa^{-2}r^{-1}$ fraction of all triplets, we have
\[\frac{1}{|S|}\sum_{t\in S}\frac{4\exp(y_t\Tr(\bM_t \bA\bA^\top))}{(\exp(y_t\Tr(\bM_t \bA\bA^\top))+1)^2}(\Tr(\bM_t\widetilde\bZ\bA^\top))^2\ge \frac{C_2C_8}{8}\kappa^{-2}\|\widetilde{\bZ}\|_\textnormal{F}^2,\] which satisfies the lemma.

\end{proof}

\begin{lem}\label{lem:alpha_2_upper}
    Under the conditions of Theorem \ref{thm:hessian}, for any matrix $\bV\in\R^{p\times r}$,
    \begin{align*}
        \vc(\bV)^\top\alpha_2\vc(\bV)\lesssim\|\bV\|_\textnormal{F}^2.
    \end{align*}
    with probability at least $1-O(n^{-10})$.
\end{lem}
\begin{proof}
We can write
\begin{align}\nonumber
    &\hspace{-10mm}\frac{1}{|S|}\sum_{t\in S}\vc(\bV)^\top\alpha_2\vc(\bV)\\\nonumber
    &=\frac{1}{|S|}\sum_{t\in S}\vc(\bV)^\top \frac{2\exp(y_t\Tr(\bM_t \bA\bA^\top))}{(\exp(y_t\Tr(\bM_t \bA\bA^\top))+1)^2}\>(\vc(\bM_t\bA)^\top\otimes\vc(\bM_t\bA))\vc(\bV)\\\nonumber
    &=\frac{1}{|S|}\sum_{t\in S}\frac{2\exp(y_t\Tr(\bM_t \bA\bA^\top))}{(\exp(y_t\Tr(\bM_t \bA\bA^\top))+1)^2}\vc(\bV)^\top\vc(\bM_t\bA)\vc(\bM_t\bA)^\top\vc(\bV)\\
    &=\frac{1}{|S|}\sum_{t\in S}\frac{2\exp(y_t\Tr(\bM_t \bA\bA^\top))}{(\exp(y_t\Tr(\bM_t \bA\bA^\top))+1)^2}(\Tr(\bA^\top\bM_t\bV))^2. \label{eqn:hessian_ii}
\end{align}
Note that
\[\frac{2\exp(y_t\Tr(\bM_t \bA\bA^\top))}{(\exp(y_t\Tr(\bM_t \bA\bA^\top))+1)^2}\le 1.\]

Therefore, it suffices to upper bound $\frac{1}{|S|}\sum_{t\in S}(\Tr(\bA^\top\bM_t\bV))^2$. By Lemma \ref{lem:sum_sbar_tr_a_upper}, we have, with probability at least $1-O(\exp(-n))$,
\[\frac{1}{|\oS|}\sum_{t\in\oS}(\Tr(\bA\bM_t\bV))^2\le C\|\bA_\star\|^2\|\bV\|_\textnormal{F}^2.\]

By Lemma \ref{lem:sbar_to_s_Astar}, with probability at least $1-O(n^{-10})$,
\[\left|\frac{1}{|S|}\sum_{t\in S}(\Tr(\bM_t\bV\bA^\top))^2-\frac{1}{|\oS|}\sum_{t\in \oS}(\Tr(\bM_t\bV\bA^\top))^2\right|\le \frac{C}{n}\|\bA_\star\|^2\|\bV\|_\textnormal{F}^2.\]
Therefore
\begin{align*}
    \frac{1}{|S|}\sum_{t\in S}(\Tr(\bA^\top\bM_t\bV))^2
    &\le \left(C+\frac{1}{n}\right)\|\bA_\star\|^2\|\bV\|_\textnormal{F}^2\lesssim\|\bV\|_\textnormal{F}^2.
\end{align*} 
with probability at least $1-O(n^{-10})$.
\end{proof}

\begin{lem}\label{lem:empirical_gradient}
    Let $\nabla L(\bA)$ be the empirical gradient of the loss function \eqref{eqn:risk}. Then, $\left\|\nabla L(\bA_\star)\right\|\le1/n$ with probability at least $1-O(n^{-10})$.
\end{lem}
\begin{proof}
    Using the expression \eqref{eqn:gradient}, we can write
    \begin{align*}
        \|\nabla L(\bA_\star)\|&\le2\|\bA_\star\|\left\|\frac{1}{|S|}\sum_{t\in S}\frac{y_t\bM_t}{\exp(y_t\Tr(\bM_t\bA_\star\bA_\star^\top))+1}\right\|_\textnormal{F}.
    \end{align*}
    We will bound $\left\|\frac{1}{|S|}\sum_{t\in S}\frac{y_t\bM_t}{\exp(y_t\Tr(\bM_t\bA_\star\bA_\star^\top))+1}\right\|_\textnormal{F}$ by first conditioning on $\{\bM_t\}_{t\in S}$, allowing us to view the expression as a sum of independent random matrices. The expectation of each one is:
    \begin{align*}
        \EE\left[\frac{y_t\bM_t}{\exp(y_t\Tr(\bM_t\bA_\star\bA_\star^\top))+1}\mid\bM_t\right]&=\Pr[y_t=1\mid\bM_t]\cdot \frac{\bM_t}{\exp(\Tr(\bM_t\bA_\star\bA_\star^\top))+1}\\
        &\hspace{5mm}+\Pr[y_t=-1\mid\bM_t]\cdot \frac{-\bM_t}{\exp(-\Tr(\bM_t\bA_\star\bA_\star^\top))+1}\\
        &=\frac{1}{\exp(-\Tr(\bM_t\bA_\star\bA_\star^\top))+1}\cdot \frac{\bM_t}{\exp(\Tr(\bM_t\bA_\star\bA_\star^\top))+1}\\
        &\hspace{5mm}+\frac{1}{\exp(\Tr(\bM_t\bA_\star\bA_\star^\top))+1}\cdot \frac{-\bM_t}{\exp(-\Tr(\bM_t\bA_\star\bA_\star^\top))+1}\\
        &=0.
    \end{align*}

    To apply the matrix Bernstein inequality, we make two bounds. First, using Lemma \ref{lem:single_norm_mt}, we have \begin{align*}
        \left\|\frac{y_t\bM_t}{\exp(y_t\Tr(\bM_t\bA_\star\bA_\star^\top))+1}\right\|&\le \|\bM_t\|_\textnormal{F}\le C_1p\log^2n
    \end{align*} with probability at least $1-O(n^{-10})$.

    Second, again using Lemma \ref{lem:single_norm_mt}, 
    \begin{align*}
        \left\|\sum_{t\in S}\EE\left(\frac{y_t\bM_t}{\exp(y_t\Tr(\bM_t\bA_\star\bA_\star^\top))+1}\right)^2\right\|&\le \sum_{t\in S}\|\bM_t\|_\textnormal{F}^2\le |S|\cdot C_1p^2\log^4n.
    \end{align*}
    Then, applying the matrix Bernstein inequality \citep[Theorem 5.4.1]{vershynin2025highNEW}, we conclude
    \begin{align*}
        \Pr\left[\frac{1}{|S|}\sum_{t\in S}\frac{y_t\bM_t}{\exp(y_t\Tr(\bM_t\bA_\star\bA_\star^\top))+1}\ge \frac{1}{n}\mid\{\bM_t\}_{t\in S}\right]&\le 2p\exp\left(-\frac{C_2|S|}{n^2p^2\log^4n}\right)\lesssim \exp(-n^{1/4}).
    \end{align*}
    Since this holds uniformly over $t$, the conditional bound suffices to show the lemma.
\end{proof}

\section{Technical Lemmas}\label{sec:technical_lemmas}
\begin{lem}\label{lem:expected_abs_value_trace_astar}
    Suppose $\widetilde{\bZ}$ is of the form given in Theorem \ref{thm:hessian}. Then, \[\EE[|\Tr(\bM_t\widetilde\bZ\bA_\star^\top)|]\ge C\sigma_{\min}(\bA_\star)\|\widetilde{\bZ}\|_\textnormal{F}\] for some constant $C$.
\end{lem}
\begin{proof}
We can write
\begin{align*}
    &\EE[|\Tr(\bM_t\widetilde{\bZ}\bA_\star^\top)|]=\EE[|\langle \bx_i\bx_k^\top+\bx_k\bx_i^\top-\bx_i\bx_j^\top-\bx_j\bx_i^\top+\bx_j\bx_j^\top-\bx_k\bx_k^\top, \widetilde{\bZ}\bA_\star^\top\rangle|].
\end{align*}
Defining $\boldsymbol B:=\widetilde{\bZ}\bA_\star^\top+\bA_\star\widetilde{\bZ}^\top$, we have
\begin{align}\nonumber
    &\hspace{-10mm}\EE\left[\left|\langle \bx_i\bx_k^\top+\bx_k\bx_i^\top-\bx_i\bx_j^\top-\bx_j\bx_i^\top+\bx_j\bx_j^\top-\bx_k\bx_k^\top, \widetilde{\bZ}\bA_\star^\top\rangle\right|\right]\\\nonumber
    &=\EE\left[\left|\bx_i^\top\boldsymbol B\bx_k-\bx_i^\top\boldsymbol B\bx_j+\frac{1}{2}\left(\bx_j^\top\boldsymbol B\bx_j-\bx_k^\top\boldsymbol B\bx_k\right)\right|\right]\\\label{eqn:abs_val_trace_conditioning}
    &=\frac{1}{2}\EE\left[\left|(\bx_i-\bx_j)^\top\boldsymbol{B}(\bx_i-\bx_j)-(\bx_i-\bx_k)^\top\boldsymbol{B}(\bx_i-\bx_k)\right|\right]
\end{align}

By Assumption \ref{assumption:distribution}, this is $\gtrsim\|\boldsymbol B\|_{\textnormal{F}}$.

Now we will compute $\|\bB\|_\textnormal{F}=\|\widetilde{\bZ}\bA_\star^\top+\bA_\star\widetilde{\bZ}^\top\|_\textnormal{F}$.

We can write
\begin{align}\nonumber
    \|\widetilde{\bZ}\bA_\star^\top+\bA_\star\widetilde{\bZ}^\top\|_\textnormal{F}^2&=\Tr((\widetilde{\bZ}\bA_\star^\top+\bA_\star\widetilde{\bZ}^\top)^\top(\widetilde{\bZ}\bA_\star^\top+\bA_\star\widetilde{\bZ}^\top))\\
    &=\|\widetilde{\bZ}\bA_\star^\top\|_{\textnormal{F}}^2+\|\bA_\star\widetilde{\bZ}^\top\|_{\textnormal{F}}^2+2\Tr(\widetilde{\bZ}\bA_\star^\top\widetilde{\bZ}\bA_\star^\top).\label{eqn:f_norm_b}
\end{align}

We will decompose the last term using the definition of $\widetilde\bZ$:
\begin{align}\nonumber
    &\Tr(\widetilde\bZ\bA_\star^\top\widetilde\bZ\bA_\star^\top)\\\nonumber
    &=\Tr((\bZ\bH_{\bZ}-\bA_\star)\bA_\star^\top(\bZ\bH_{\bZ}-\bA_\star)\bA_\star^\top)\\
    &=\Tr(\bZ\bH_{\bZ}\bA_\star^\top\bZ\bH_{\bZ}\bA_\star^\top)+\Tr(\bA_\star\bA_\star^\top\bA_\star\bA_\star^\top)-\Tr(\bZ\bH_{\bZ}\bA_\star^\top\bA_\star\bA_\star^\top)-\Tr(\bA_\star\bA_\star^\top\bZ\bH_{\bZ}\bA_\star^\top).\label{eqn:trace_crossterm}
\end{align}
By \citet[Theorem 2]{procrustes}, if the singular value decomposition of $\bZ^\top\bA_\star$ is $\bU\bLambda\bW^\top$, then $\bH_{\bZ}=\bU\bW^\top$. Therefore
\[\Tr(\bZ\bH_{\bZ}\bA_\star^\top\bZ\bH_{\bZ}\bA_\star^\top)=\Tr(\bZ\bZ^\top\bA_\star\bA_\star^\top).\]
This means that we can write (\ref{eqn:trace_crossterm}) as
\begin{align*}
    &\hspace{-10mm}\Tr(\bZ\bZ^\top\bA_\star\bA_\star^\top)+\Tr(\bA_\star\bA_\star^\top\bA_\star\bA_\star^\top)-\Tr(\bZ\bH_{\bZ}\bA_\star^\top\bA_\star\bA_\star^\top)-\Tr(\bA_\star\bA_\star^\top\bZ\bH_{\bZ}\bA_\star^\top)\\
    &=\Tr(\bZ\bZ^\top\bA_\star\bA_\star^\top)+\Tr(\bA_\star\bA_\star^\top\bA_\star\bA_\star^\top)-\Tr(\bZ\bH_{\bZ}\bA_\star^\top\bA_\star\bA_\star^\top)-\Tr(\bA_\star\bZ^\top\bH_{\bZ}^\top\bA_\star\bA_\star^\top)\\
    &=\Tr(\bZ\bZ^\top\bA_\star\bA_\star^\top+\bA_\star\bA_\star^\top\bA_\star\bA_\star^\top-\bZ\bH_{\bZ}\bA_\star^\top\bA_\star\bA_\star^\top-\bA_\star\bH_{\bZ}^\top\bZ^\top\bA_\star\bA_\star^\top)\\
    &=\Tr((\bZ\bZ^\top+\bA_\star\bA_\star^\top-\bZ\bH_{\bZ}\bA_\star^\top-\bA_\star\bH_{\bZ}^\top\bZ^\top)\bA_\star\bA_\star^\top)\\
    &=\Tr((\bZ\bZ^\top+\bA_\star\bA_\star^\top-\bZ\bU\bW^\top\bA_\star^\top-\bA_\star\bW\bU^\top\bZ^\top)\bA_\star\bA_\star^\top)\\
    &=\Tr((\bZ\bU\bW^\top-\bA_\star)(\bZ\bU\bW^\top-\bA_\star)^\top\bA_\star\bA_\star^\top).
\end{align*}
Since this is a trace of a product of two positive semidefinite matrices, it is non-negative. Therefore we can lower bound (\ref{eqn:trace_crossterm}) by 0.

Recalling (\ref{eqn:f_norm_b}), we conclude that \begin{align*}
    \|\widetilde{\bZ}\bA_\star^\top+\bA_\star\widetilde{\bZ}^\top\|_\textnormal{F}^2&\ge \|\widetilde\bZ\bA_\star^\top\|_\textnormal{F}^2+\|\bA_\star\widetilde\bZ^\top\|_\textnormal{F}^2\\
    \|\widetilde{\bZ}\bA_\star^\top+\bA_\star\widetilde{\bZ}^\top\|_\textnormal{F}&\ge\sqrt{2}\|\bA_\star\widetilde\bZ^\top\|\\
    &\ge\sqrt{2}
    \,\sigma_{\min}(\bA_\star)\|\widetilde\bZ\|_\textnormal{F}.
\end{align*}

Recalling (\ref{eqn:abs_val_trace_conditioning}), we conclude that $\EE[|\Tr(\bM_t\widetilde\bZ\bA^\top)|]\ge C\sigma_{\min}(\bA_\star)\|\widetilde{\bZ}\|_\textnormal{F}$.
\end{proof}

\begin{lem}\label{lem:expectation_tr2_upper_astar}
    For matrices $\bV,\bA_\star\in\R^{p\times r}$, we have
    \[\EE[(\Tr(\bA_\star^\top\bM_t\bV))^2]\le C\|\bA_\star\|^2\|\bV\|_\textnormal{F}^2\] for some constant $C$.
\end{lem}
\begin{proof}
We have
\begin{align*}
    \EE\left[\Tr(\bA_\star^\top\bM_t\bV)^2\right]&=\vc(\bV\bA_\star^\top)^\top\Var(\vc(\bM_t))\vc(\bV\bA_\star^\top),
\end{align*}
so we begin with by computing $\EE[\vc(\bM_t)\vc(\bM_t)^\top]=\Var(\vc(\bM_t))$:
\begin{align*}
    \Var(\vc(\bM_t))=&\Var(\vc(\bx_i\bx_k^\top+\bx_k\bx_i^\top-\bx_i\bx_j^\top-\bx_j\bx_i^\top+\bx_j\bx_j^\top-\bx_k\bx_k^\top))\\
    =&\Var(\vc(\bx_i\bx_k^\top))+\Var(\vc(\bx_k\bx_i^\top))    +\Var(\vc(\bx_i\bx_j^\top))\\ &+\Var(\vc(\bx_j\bx_i^\top))+\Var(\vc(\bx_j\bx_j^\top))+\Var(\vc(\bx_i\bx_i^\top))\\
    &+2\Cov(\vc(\bx_i\bx_k^\top),\vc(\bx_k\bx_i^\top))
    +2\Cov(\vc(\bx_i\bx_j^\top),\vc(\bx_j\bx_i^\top)).
\end{align*}
For the terms of the form $\Var(\vc(\bx_a\bx_b^\top))$ for $a\neq b\in\{i,j,k\}$, we have
\begin{align*}
    \Var(\vc(\bx_a\bx_b^\top))&=\EE[\vc(\bx_a\bx_b^\top)\vc(\bx_a\bx_b^\top)^\top]\\
    &=\bSigma\otimes\bSigma.
\end{align*}

For the terms of the form $\Var(\vc(\bx_a\bx_a^\top))$ for $a\in\{i,j,k\}$, we have
\begin{align*}
    \Var(\vc(\bx_a\bx_a^\top))&=\EE[\vc(\bx_a\bx_a^\top)\vc(\bx_a\bx_a^\top)^\top]-\EE[\vc(\bx_a\bx_a^\top)]\EE[\vc(\bx_a\bx_a^\top)]^\top\\
    &=\EE[\vc(\bx_a\bx_a^\top)\vc(\bx_a\bx_a^\top)^\top]-\vc(\bSigma)\vc(\bSigma)^\top.
\end{align*}

For the terms of the form $\Cov(\vc(\bx_a\bx_b^\top),\vc(\bx_b\bx_a^\top))$ for $a\neq b\in\{i,j,k\}$, we have
\begin{align*}
    \Cov(\vc(\bx_a\bx_b^\top),\vc(\bx_b\bx_a^\top))
    &=\EE[\vc(\bx_a\bx_b^\top)\vc(\bx_b\bx_a^\top)^\top]\\
    &=\bSigma_1\star\bSigma_2
\end{align*}
where $\star$ is the Khatri-Rao product, $\bSigma_1$ is $\bSigma$ partitioned into columns, and $\bSigma_2$ is $\bSigma$ partitioned into rows. 

Therefore
\begin{align*}
    &\hspace{-0mm}\EE\left[\Tr(\bA_\star^\top\bM_t\bV)^2\right]\\
    &=\vc(\bV\bA_\star^\top)^\top\Var(\vc(\bM_t))\vc(\bV\bA_\star^\top)\\
    &=\vc(\bV\bA_\star^\top)^\top(4\bSigma\otimes\bSigma+2\EE[\vc(\bx_a\bx_a^\top)\vc(\bx_a\bx_a^\top)^\top]-2\vc(\bSigma)\vc(\bSigma)^\top+4\bSigma_1\star\bSigma_2)\vc(\bV\bA_\star^\top).
\end{align*}
We will look at the terms one by one. For the $4\bSigma\otimes\bSigma$ term, we get \[4\vc(\bV\bA_\star^\top)^\top\bSigma\otimes\bSigma\vc(\bV\bA_\star^\top)=4\Tr(\bA_\star\bV^\top\bSigma\bV\bA_\star^\top\bSigma)\lesssim\|\bV\bA_\star^\top\|^2_\textnormal{F}.\]
For the $2\EE[\vc(\bx_a\bx_a^\top)\vc(\bx_a\bx_a^\top)^\top]$ term, we rely on Assumption \ref{assumption:distribution} to say that $\|\EE[\vc(\bx_a\bx_a^\top)\vc(\bx_a\bx_a^\top)^\top]\|\lesssim 1$, so \[2\vc(\bV\bA_\star^\top)^\top\EE[\vc(\bx_a\bx_a^\top)\vc(\bx_a\bx_a^\top)^\top]\vc(\bV\bA_\star^\top)\lesssim\|\bV\bA_\star^\top\|^2_\textnormal{F}.\]
For $-2\vc(\bSigma)\vc(\bSigma)^\top,$ we have \[-2\vc(\bV\bA_\star^\top)^\top\vc(\bSigma)\vc(\bSigma)^\top\vc(\bV\bA_\star^\top)=-2\Tr(\bA^\star\bV^\top\bSigma)\Tr(\bSigma\bV\bA_\star^\top)\le 0.\]
For the $4\bSigma_1\star\bSigma_2$ term, we use \citet[Theorem 2.7]{khatrirao} to say that $\|\bSigma_1\star\bSigma_2\|\lesssim\|\bSigma_1\|\|\bSigma_2\|\lesssim 1$, so \[4\vc(\bV\bA_\star^\top)^\top\bSigma_1\star\bSigma_2\vc(\bV\bA_\star^\top)\lesssim \|\bV\bA_\star^\top\|^2_\textnormal{F}.\]

Combining all of the bounds, we get the result.
\end{proof}

\begin{lem}\label{lem:concentration_tr2_astar}
    For any matrices $\bV,\bA_\star\in\R^{p\times r}$, with probability at least $1-O(\exp(-n^{1/8}))$, we have
    \begin{align*}
        \left|\frac{1}{|\oS|}\sum_{t\in\oS}\Tr(\bA_\star^\top\bM_t\bV)^2-\EE\left[\Tr(\bA_\star^\top\bM_t\bV)^2\right]\right|\le\frac{C_1 r\log^2 r}{n^{1/4}}\|\bV\bA_\star^\top\|^2_\textnormal{F}
    \end{align*}
\end{lem}
\begin{proof}

For the remainder of this proof, let $\bU_r\bLambda\bW_r^\top$ be the compact singular value decomposition of $\bV\bA_\star^\top$, where $\bLambda=\text{diag}(\lambda_1,\ldots,\lambda_r)$.

We will write
\begin{align*}
    &\frac{1}{|\oS|}\sum_{t\in\oS}(\Tr(\bM_t\bV\bA_\star^\top))^2\\
    &=\frac{1}{|\oS|}\sum_{t\in\oS}(\bx_i^\top\bV\bA_\star^\top\bx_k-\bx_i^\top\bV\bA_\star^\top\bx_j+\bx_k^\top\bV\bA_\star^\top\bx_i-\bx_j^\top\bV\bA_\star^\top\bx_i+\bx_j^\top\bV\bA_\star^\top\bx_j-\bx_k^\top\bV\bA_\star^\top\bx_k)^2.
\end{align*}
This expansion has 36 terms, each of the form
\[\pm\frac{1}{|\oS|}\sum_{t\in\oS}(\bx_{a_t}^\top\bV\bA_\star^\top\bx_{b_t})(\bx_{c_t}^\top\bV\bA_\star^\top\bx_{d_t})\] where $a_t,b_t,c_t,d_t\in\{i,j,k\}$ (not necessarily distinct). Since $i$, $j$, and $k$ have symmetric roles in $\oS$, we can refer to a term using a pattern $abcd$ (where some letters may be the same). By expanding the square, we observe that there are the following terms:
\begin{center}
\begin{tabular}{|c|r|c|}\hline
    Case  & Coefficient & Form \\\hline
    (i)   & $ 4$ & $abba$ \\
    (ii)  & $ 4$ & $abab$ \\
    (iii) & $-2$ & $abac$ \\
    (iv)  & $-4$ & $abbc$ \\
    (v)   & $ 8$ & $abcc$ \\
    (vi)  & $-4$ & $abbb$ \\
    (vii) & $-2$ & $abcb$ \\
    (iix) & $-4$ & $aaab$ \\
    (ix)  & $ 2$ & $aaaa$ \\
    (x)   & $ 2$ & $aabb$ \\\hline
\end{tabular}
\end{center}

We will show that each term concentrates around its mean. In particular, we will show that each term $\alpha$ satisfies \begin{equation}
    \Pr\left[|\alpha-\EE[\alpha]|\ge\frac{C_1r^2\log^2 r}{n^{1/4}}\|\bV\bA_\star^\top\|_\textnormal{F}^2\right]\lesssim \exp(-n^{1/8})\label{eqn:concentration_condition}
\end{equation} for some constant $C_1$. Applying the union bound, this will imply the lemma.

Throughout the remainder of this proof, without loss of generality, let $a=i$, $b=j$, $c=k$.

\underline{Case (i):} We have
\begin{align*}
    &\frac{1}{|\oS|}\sum_{t\in\oS}(\bx_i^\top\bV\bA_\star^\top\bx_j)(\bx_j^\top\bV\bA_\star^\top\bx_i)\\
    &=\frac{1}{n}\sum_{k}\left(\frac{1}{(n-1)(n-2)}\sum_{i\neq k}\sum_{j\neq k,i}(\bx_i^\top\bV\bA_\star^\top\bx_j)(\bx_j^\top\bV\bA_\star^\top\bx_i)\right)\\
    &=\frac{1}{n}\sum_{k}\left(\frac{1}{n-1}\sum_{i\neq k}\bx_i^\top\bU_r\bLambda\bW_r^\top\left(\frac{1}{n-2}\sum_{j\neq k,i}\bx_j\bx_j^\top\right)\bU_r\bLambda\bW_r^\top\bx_i\right)\\
    &=\frac{1}{n}\sum_k\left(\underbrace{\frac{1}{n-1}\sum_{i\neq k}\bx_i^\top(\bU_r\bLambda\bW_r^\top\bSigma\bU_r\bLambda\bW_r^\top)\bx_i}_{\beta_1}\right.\\
    &\hspace{10mm}\left.+\underbrace{\frac{1}{n-1}\sum_{i\neq k}\bx_i^\top\bU_r\bLambda\bW_r^\top\left(\frac{1}{n-2}\sum_{j\neq k,i}\bx_j\bx_j^\top-\bSigma\right)\bU_r\bLambda\bW_r^\top\bx_i}_{\beta_2}\right).
\end{align*}
The remaining analysis for this case holds for all $k$; therefore, it holds for the average over $k$. 

By Lemma \ref{lem:bernstein} (and using that $\|\bSigma\|\lesssim 1$),
\[\Pr\left[|\beta_1-\EE[\beta_1]|\ge q_1r\|\bV\bA_\star^\top\|_\textnormal{F}^2\right]\lesssim \exp(-n\min\{q_1^2,q_1\}).\]
To bound $\beta_2$, we use Lemma \ref{lem:covariance_estimation} to say that
\[\left\|\bW_r^\top\left(\frac{1}{n-2}\sum_{j\neq i,k}\bx_b\bx_b^\top-\bSigma\right)\bU_r\right\|\lesssim n^{-1/2}\] with probability at least $1-O(\exp(-n^{1/2}))$. Using the union bound, this holds for all $t\in\oS$ with probability $1-O(\exp(-n^{1/2}))$.

We can apply the union bound to this bound and to the conclusion of another application of Lemma \ref{lem:bernstein}, so with probability $1-O(\exp(-n^{1/2})+\exp(-n\min\{q_2^2,q_2\}))$, 
\begin{align*}
    |\beta_2-\EE[\beta_2]|&\le q_2\left(\max_{t\in\oS}\left\|\bU_r\bLambda\bW_r^\top\left(\frac{1}{n-2}\sum_{j\neq k,i}\bx_j\bx_j^\top-\bSigma\right)\bU_r\bLambda\bW_r^\top\right\|\right)\\
    &\le \frac{q_2}{n^{1/2}}\|\bV\bA_\star^\top\|_\textnormal{F}^2.
\end{align*}
Combining our bounds on $\beta_1$ and $\beta_2$, taking $q_1=q_2=n^{-1/2}$, we have that 
\begin{align*}
    &\Pr\left[\left|\frac{1}{|\oS|}\sum_{t\in\oS}(\bx_i^\top\bV\bA_\star^\top\bx_j)(\bx_j^\top\bV\bA_\star^\top\bx_i)-\EE[(\bx_i^\top\bV\bA_\star^\top\bx_j)(\bx_j^\top\bV\bA_\star^\top\bx_i)]\right|\ge \frac{r}{n^{1/2}}\|\bV\bA_\star^\top\|_\textnormal{F}^2\right]\\
    &\hspace{20mm}\lesssim \exp(-n^{1/2})
\end{align*}
which satisfies (\ref{eqn:concentration_condition}).

\underline{Case (ii):} We have
\begin{align*}
    &\frac{1}{|\oS|}\sum_{t\in\oS}(\bx_i^\top\bV\bA_\star^\top\bx_j)^2\\
    &=\frac{1}{n}\sum_{k}\left(\frac{1}{(n-1)(n-2)}\sum_{i\neq k}\sum_{j\neq k,i}(\bx_i^\top\bV\bA_\star^\top\bx_j)(\bx_j^\top\bA_\star\bV^\top\bx_i)\right)\\
    &=\frac{1}{n}\sum_{k}\left(\frac{1}{n-1}\sum_{i\neq k}\bx_i^\top\bU_r\bLambda\bW_r^\top\left(\frac{1}{n-2}\sum_{j\neq k,i}\bx_j\bx_j^\top\right)\bW_r\bLambda\bU_r^\top\bx_i\right)\\
    &=\frac{1}{n}\sum_k\left(\underbrace{\frac{1}{n-1}\sum_{i\neq k}\bx_i^\top(\bU_r\bLambda\bW_r^\top)\bSigma(\bW_r\bLambda\bU_r^\top)\bx_i}_{\beta_1}\right.\\
    &\hspace{10mm}\left.+\underbrace{\frac{1}{n-1}\sum_{i\neq k}\bx_i^\top\bU_r\bLambda\bW_r^\top\left(\frac{1}{n-2}\sum_{j\neq k,i}\bx_j\bx_j^\top-\bSigma\right)\bW_r\bLambda\bU_r^\top\bx_i}_{\beta_2}\right).
\end{align*}
The remaining analysis for this case holds uniformly for all $k$; therefore, it holds for the average over $k$. 

By Lemma \ref{lem:bernstein} (and using that $\|\bSigma\|\lesssim 1$),
\[\Pr\left[|\beta_1-\EE[\beta_1]|\ge q_1r\|\bA_\star\|^2\|\bV\|_\textnormal{F}^2\right]\lesssim \exp(-n\min\{q_1^2,q_1\}).\]
To bound $\beta_2$, we use Lemma \ref{lem:covariance_estimation} to say that
\[\left\|\bW_r^\top\left(\frac{1}{n-2}\sum_{j\neq i,k}\bx_b\bx_b^\top-\bSigma\right)\bW_r\right\|\lesssim n^{-1/2}\] with probability at least $1-O(\exp(-n^{1/2}))$. Using the union bound, this holds for all $t\in\oS$ with probability $1-O(\exp(-n^{1/4}))$.

We can apply the union bound to this bound and to the conclusion of another application of Lemma \ref{lem:bernstein}, so with probability $1-O(\exp(-n^{1/4})+\exp(-n\min\{q_2^2,q_2\}))$, 
\begin{align*}
    |\beta_2-\EE[\beta_2]|&\le q_2\left(\max_{t\in\oS}\left\|\bU_r\bLambda\bW_r^\top\left(\frac{1}{n-2}\sum_{j\neq k,i}\bx_j\bx_j^\top-\bI_p\right)\bW_r\bLambda\bU_r^\top\right\|\right)\le \frac{q_2}{n^{1/2}}\|\bA_\star\|^2\|\bV\|_\textnormal{F}^2.
\end{align*}
Combining our bounds on $\beta_1$ and $\beta_2$ and taking $q_1=q_2=q_3=n^{-1/2}$, we have that 
\begin{align*}
    &\Pr\left[\left|\frac{1}{|\oS|}\sum_{t\in\oS}(\bx_i^\top\bV\bA_\star^\top\bx_j)^2-\EE[(\bx_i^\top\bV\bA_\star^\top\bx_j)^2]\right|\ge \frac{r}{n^{1/2}}\|\bV\bA_\star^\top\|_\textnormal{F}^2\right]\lesssim \exp(-n^{1/4})
\end{align*}
which satisfies (\ref{eqn:concentration_condition}).

\underline{Case (iii):} We can write
\begin{align}\nonumber
    &\left|\frac{1}{|\oS|}\sum_{t\in\oS}(\bx_i^\top\bU_r\bLambda\bW_r^\top\bx_j)(\bx_i^\top\bU_r\bLambda\bW_r^\top\bx_k)\right|\\\nonumber
    &=\left|\frac{1}{|\oS|}\sum_{t\in\oS}(\bx_j^\top\bW_r\bLambda\bU_r^\top\bx_i)(\bx_i^\top\bU_r\bLambda\bW_r^\top\bx_k)\right|\\\nonumber
    &=\left|\frac{1}{|\oS|}\sum_{t\in\oS}(\bW_r^\top\bx_j)^\top(\bLambda\bU_r^\top\bx_i\bx_i^\top\bU_r\bLambda)(\bW_r^\top\bx_k)\right|\\\nonumber
    &=\left|\frac{1}{n}\sum_{j}\left((\bW_r^\top\bx_j)^\top\left(\frac{1}{n-1}\sum_{k\neq j}\bLambda\bU_r^\top\left(\frac{1}{n-2}\sum_{i\neq j,k}\bx_i\bx_i^\top\right)\bU_r\bLambda\right)\left(\bW_r^\top\bx_k\right)\right)\right|\\\nonumber
    &=\underbrace{\left|\frac{1}{n}\sum_{j}\left((\bW_r^\top\bx_j)^\top{\bLambda}\bU_r^\top\bSigma\bU_r\bLambda\left(\frac{1}{n-1}\sum_{k\neq j}\bW_r^\top\bx_k\right)\right)\right|}_{\beta_1}\\\nonumber
    &\hspace{10mm}+\underbrace{\left|\frac{1}{n}\sum_{j}\left((\bW_r^\top\bx_j)^\top\left(\frac{1}{n-1}\sum_{k\neq j}\bLambda\bU_r^\top\left(\frac{1}{n-2}\sum_{i\neq j,k}\bx_i\bx_i^\top-\bSigma\right)\bU_r\bLambda\left(\bW_r^\top\bx_k\right)\right)\right)\right|}_{\beta_2}.
    \label{eqn:iii_expand}
\end{align}

First, we bound $\beta_1$. By Lemma \ref{lem:sum_of_random_vectors}, we have
\begin{equation*}
    \Pr\left[\left\|\frac{1}{n-1}\sum_{k\neq j}\bW_r^\top\bx_k\right\|\ge q\sqrt{\frac{r}{n}}\right]\lesssim \exp(-qn^2),
\end{equation*}
so \begin{equation*}
    \Pr\left[\max_j\left\|\frac{1}{n-1}\sum_{k\neq j}\bW_r^\top\bx_k\right\|\ge q\sqrt{\frac{r}{n}}\right]\lesssim n\exp(-qn^2).
\end{equation*}

Therefore,
\begin{align*}
    \beta_1&\le \left|\frac{1}{n}\sum_j\left(\|\bW_r^\top\bx_j\|\cdot\|{\bLambda}\bU_r^\top\bSigma\bU_r\bLambda\|\cdot\max_j\left\|\frac{1}{n-1}\sum_{k\neq j}\bW_r^\top\bx_k\right\|\right)\right|\\
    &\lesssim q\sqrt{\frac{r}{n}}\|\bLambda\|^2\left(\frac{1}{n}\sum_j\|\bW_r^\top\bx_j\|\right)\lesssim q\sqrt{\frac{rp}{n}}\|\bLambda\|^2
\end{align*}
with probability $1-O(n\exp(-qn^2)+\exp(-n))$, where the last step is by Lemma $\ref{lem:norm}$.

Now we bound $\beta_2$. By Lemma \ref{lem:covariance_estimation}, we have \[\Pr\left[\left\|\bU_r^\top\left(\frac{1}{n-2}\sum_{i\neq j,k}\bx_i\bx_i^\top-\bSigma\right)\bU_r\right\|\ge C_1 n^{-1/2}\right]\lesssim\exp(-n^{1/2}),\] which implies
\begin{equation}\label{eqn:caseiii_beta2}
    \Pr\left[\left\|\bLambda\bU_r^\top\left(\frac{1}{n-2}\sum_{i\neq j,k}\bx_i\bx_i^\top-\bSigma\right)\bU_r\bLambda\right\|\ge C_1\sqrt{\frac{1}{n}}\|\bLambda\|^2\right]\lesssim\exp(-n^{1/2}).
\end{equation}
Using the union bound, the bound is true for all $j$ and all $k$ with probability $1-O(\exp(-n^{1/4}))$.

Therefore we have
\begin{align*}
    \beta_2&\le\frac{1}{n}\sum_j\|\bW_r^\top\bx_j\|\left(\max_{j}\left\|\frac{1}{n-1}\sum_{k\neq j}\bLambda\bU_r^\top\left(\frac{1}{n-2}\sum_{i\neq j,k}\bx_i\bx_i^\top-\bSigma\right)\bU_r\bLambda\left(\bW_r^\top\bx_k\right)\right\|\right)\\
    &\le\frac{1}{n}\|\bW_r^\top\bx_j\|\left(\max_j\frac{1}{n-1}\sum_{k\neq j}\left(\max_{k,j}\left\|\bLambda\bU_r^\top\left(\frac{1}{n-2}\sum_{i\neq j,k}\bx_i\bx_i^\top-\bSigma\right)\bU_r\bLambda\right\|\cdot\|\bW_r^\top\bx_k\|\right)\right)\\
    &\lesssim \sqrt{\frac{p^2}{n}}\|\bLambda\|^2
\end{align*}
with probability at least $1-O(\exp(-n^{1/4}))$, where we have applied Lemma \ref{lem:norm} twice and \ref{eqn:caseiii_beta2} once to get the last line.

Combining the bounds for $\beta_1$ and $\beta_2$, we get that
\begin{align*}
    &\Pr\left[\left|\frac{1}{|\oS|}\sum_{t\in \oS}(\bx_a^\top\bU_r\bLambda\bW_r^\top\bx_b)(\bx_a^\top\bU_r\bLambda\bW_r^\top\bx_c)\right|\ge C_2\left(q\sqrt{\frac{rp}{n}}+\sqrt{\frac{p^2}{n}}\right)\|\bA_\star\|^2\|\bV\|_\textnormal{F}^2\right]\\
    &\hspace{10mm}\lesssim \exp(-n^{1/4})+\exp(-qn^2).
\end{align*} 
Taking $q=1/n$, this satisfies (\ref{eqn:concentration_condition}).

\underline{Case (iv):} We will write
\begin{align*}
    &\frac{1}{n(n-1)(n-2)}\sum_{i}\sum_{j\neq i}\sum_{k\neq i,j}(\bx_i^\top\bV\bA_\star^\top\bx_j)(\bx_j^\top\bV\bA_\star^\top\bx_k)\\
    &=\frac{1}{n}\sum_{k}\left(\frac{1}{n-1}\sum_{i\neq k}\bx_i^\top\bU_r\bLambda\bW_r^\top\left(\frac{1}{n-2}\sum_{j\neq i,k}\bx_j\bx_j^\top\right)\bU_r\bLambda\bW_r^\top\bx_k\right)\\
    &=\underbrace{\frac{1}{n}\sum_{k}\left(\frac{1}{n-1}\sum_{i\neq k}\bx_i^\top(\bU_r\bLambda\bW_r^\top)\bSigma(\bU_r\bLambda\bW_r^\top)\bx_k\right)}_{\beta_1}\\
    &\hspace{10mm}+\underbrace{\frac{1}{n}\sum_{k}\left(\frac{1}{n-1}\sum_{i\neq k}\bx_i^\top\bU_r\bLambda\bW_r^\top\left(\frac{1}{n-2}\sum_{j\neq i,k}\bx_j\bx_j^\top-\bSigma\right)\bU_r\bLambda\bW_r^\top\bx_k\right)}_{\beta_2}.
\end{align*}
By Lemma \ref{lem:bernstein_diff},
\[\Pr\left[\left|\frac{1}{n(n-1)}\sum_k\sum_{i\neq k}\bx_i^\top(\bU_r\bLambda\bW_r^\top)\bSigma(\bU_r\bLambda\bW_r^\top)\bx_k\right|\ge rn^{-1}\|\bA_\star\|^2\|\bV\|_\textnormal{F}^2\right]\lesssim \exp(-n).\]
Therefore, $\beta_1\le rn^{-1}\|\bA_\star\|^2\|\bV\|_\textnormal{F}^2$ with probability $1-O(\exp(-n)$.

To bound $\beta_2$, we use Lemma \ref{lem:covariance_estimation} to say that
\[\left\|\bW_r^\top\left(\frac{1}{n}\sum_{j\neq i,k}\bx_j\bx_j^\top-\bSigma\right)\bU_r\right\|\le n^{-1/2}\] with probability at least $1-O(\exp(-n^{1/2}))$. Using the union bound, this is true for all $i,k$ with probability $1-O(\exp(-n^{1/4}))$.
Then, we can apply Lemma \ref{lem:bernstein_diff} again to say that
\[\Pr[|\beta_2-\EE[\beta_2]|\ge 2rn^{-3/2}\cdot \|\bA_\star\|^2\|\bV\|_\textnormal{F}^2]\lesssim \exp(-n^{1/4}).\]
Combining our bounds on $\beta_1$ and $\beta_2$, we have that
\begin{align*}
    &\Pr\left[\left|\frac{1}{n(n-1)(n-2)}\sum_{t\in\oS}(\bx_i^\top\bV\bA_\star^\top\bx_j)(\bx_j^\top\bV\bA_\star^\top\bx_k)-\EE[(\bx_i^\top\bV\bA_\star^\top\bx_j)(\bx_j^\top\bV\bA_\star^\top\bx_k)]\right|\ge 3rn^{-1}\cdot \|\bA_\star\|\|\bV\|_\textnormal{F}^2\right]\\
    &\hspace{20mm}\lesssim\exp(-n^{1/4})
\end{align*}
which satisfies (\ref{eqn:concentration_condition}).

\underline{Case (v):} First, note that terms of this type are zero-mean. Observe that 
\begin{align*}
    \frac{1}{|\oS|}\sum_{t\in\oS}(\bx_{i}^\top\bV\bA_\star^\top\bx_{j})(\bx_{k}^\top\bV\bA_\star^\top\bx_{k})&=\frac{1}{n(n-1)}\sum_{i}\sum_{j\neq i}\bx_i^\top\bV\bA_\star^\top\bx_j\left(\frac{1}{n-2}\sum_{k\neq i,j}\bx_k^\top\bV\bA_\star^\top\bx_k\right).
\end{align*}

By Lemma \ref{lem:bernstein}, \[\left|\frac{1}{n-2}\sum_{k\neq i,j}\bx_k^\top\bV\bA_\star^\top\bx_k-\EE[\bx_k^\top\bV\bA_\star^\top\bx_k]\right|\le q_1\|\bA_\star\|^2\|\bV\|_\textnormal{F}^2\] with probability $1-O(\exp(-n\min\{q_1^2,q_1\}))$. Take $q_1:=n^{-1/4}$. Applying a union bound, we get
\[\Pr\left[\max_{i,j}\left|\frac{1}{n-2}\sum_{k\neq i,j}\bx_k^\top\bV\bA_\star^\top\bx_k-\EE[\bx_k^\top\bV\bA_\star^\top\bx_k]\right|\ge\frac{1}{n^{1/4}}\|\bA_\star\|^2\|\bV\|_\textnormal{F}^2\right]\lesssim\exp(-n^{1/4})\] %

Applying Lemma \ref{lem:bernstein_diff},
\[\left|\frac{1}{n(n-1)}\sum_{i\neq j}\bx_i^\top\bV\bA_\star^\top\bx_j\right|\le rn^{-1}\|\bA_\star\|^2\|\|\bV\|_\textnormal{F}^2\] with probability $1-O(\exp(-n))$.

Putting these bounds together, we have that 
\[\left|\frac{1}{|\oS|}\sum_{t\in\oS}(\bx_{i}^\top\bV\bA_\star^\top\bx_{j})(\bx_{k}^\top\bV\bA_\star^\top\bx_{k})\right|\le rn^{-1/4}\|\bA_\star\|^2\|\|\bV\|_\textnormal{F}^2\] with probability $1-O(\exp(-n^{-1/4}))$, satisfying \eqref{eqn:concentration_condition}.

\underline{Case (vi):} We have
\begin{align}
    \frac{1}{n(n-1)}\sum_{a\neq b\in[n]}(\bx_a^\top\bV\bA_\star^\top\bx_b)(\bx_b^\top\bV\bA_\star^\top\bx_b)&=\frac{1}{n}\sum_{a\in[n]}\bx_a^\top\bV\bA_\star^\top\left(\frac{1}{n-1}\sum_{b\neq a}\bx_b\bx_b^\top\bV\bA_\star^\top\bx_b\right).\label{eqn:case_vi_expand}
\end{align}
By Lemma \ref{lem:normal_third_power}, we have
\[\Pr\left[\left\|\frac{1}{n-1}\sum_{b\neq a}\bx_a\bx_a^\top\bV\bA_\star^\top\bx_a-\EE\left[\bx_a\bx_a^\top\bV\bA_\star^\top\bx_a\right]\right\|\ge C_3n^{-1/4}\|\bV\bA_\star^\top\|_\textnormal{F}\right]\lesssim \exp(-n^{1/4}).\]
Applying a union bound, we have
\[\Pr\left[\max_{a}\left\|\frac{1}{n-1}\sum_{b\neq a}\bx_a\bx_a^\top\bV\bA_\star^\top\bx_a-\EE\left[\bx_a\bx_a^\top\bV\bA_\star^\top\bx_a\right]\right\|\ge C_3n^{-1/4}\|\bV\bA_\star^\top\|_\textnormal{F}\right]\lesssim \exp(-n^{1/8}).\]

By Lemma \ref{lem:sum_of_random_vectors}, we have that 
\[\Pr\left[\left\|\frac{1}{n}\sum_{a\in[n]}\bV\bA_\star^\top\bx_a\right\|\ge q\sqrt{\frac{r}{n}}\|\bV\bA_\star^\top\|\right]\lesssim\exp(-nq^2).\]

Also, we have
\begin{align*}
    \EE\left[\left(\bx_a\bx_a^\top\bV\bA_\star^\top\bx_a\right)_\ell\right]&=\sum_{\ell_1,\ell_2\in[p]}(\bV\bA_\star^\top)_{\ell_1,\ell_2}\EE[x_{a,\ell}x_{a,\ell_1}x_{a,\ell_2}]\\
    &=\sqrt{p}\|\bV\bA_\star^\top\|_\textnormal{F}
\end{align*}
(where the last line is because each of $x_{a,\ell},x_{a,\ell_1},x_{a,\ell_2}$ are subgaussian).

Therefore we can write
\begin{align*}
    &\hspace{-0mm}\frac{1}{n}\sum_{a\in[n]}\bx_a^\top\bV\bA_\star^\top\left(\frac{1}{n-1}\sum_{b\neq a}\bx_b\bx_b^\top\bV\bA_\star^\top\bx_b\right)\\
    &=\underbrace{\frac{1}{n}\sum_{a\in[n]}\bx_a^\top\bV\bA_\star^\top\EE\left[\bx_b\bx_b^\top\bV\bA_\star^\top\bx_b\right]}_{\beta_1}+\underbrace{\frac{1}{n}\sum_{a\in[n]}\bx_a^\top\bV\bA_\star^\top\left(\frac{1}{n-1}\sum_{b\neq a}\bx_b\bx_b^\top\bV\bA_\star^\top\bx_b-\EE\left[\bx_b\bx_b^\top\bV\bA_\star^\top\bx_b\right]\right)}_{\beta_2}
\end{align*}

$|\beta_1|\le q\sqrt{\frac{pr}{n}}\|\bV\bA_\star^\top\|_\textnormal{F}^2$ with probability at least $1-O(\exp(-nq^2))$. $|\beta_2|\le C_3qr^{1/2}n^{-3/4}\|\bV\bA_\star^\top\|_\textnormal{F}^2$ with probability at least $1-O(\exp(-nq^2)+\exp(-n^{1/8}))$. Setting $q=n^{-1/4}$, we obtain the bound 
\[\frac{1}{n}\sum_{a\in[n]}\bx_a^\top\bV\bA_\star^\top\left(\frac{1}{n-1}\sum_{b\neq a}\bx_b\bx_b^\top\bV\bA_\star^\top\bx_b\right)\le C_4\sqrt{\frac{r}{n}}\|\bV\bA_\star^\top\|_\textnormal{F}^2\] with probability $1-O(\exp(-n^{1/8}))$,
which satisfies (\ref{eqn:concentration_condition}).

\underline{Case (vii):} We can write
\begin{align*}
    &\left|\frac{1}{|\oS|}\sum_{t\in\oS}(\bx_i^\top\bU_r\bLambda\bW_r^\top\bx_j)(\bx_k^\top\bU_r\bLambda\bW_r^\top\bx_j)\right|=\left|\frac{1}{|\oS|}\sum_{t\in\oS}(\bx_i^\top\bU_r\bLambda\bW_r^\top\bx_j)(\bx_j^\top\bW_r\bLambda\bU_r^\top\bx_k)\right|.
\end{align*}
Then, by the same arguments as Case (iii) with $\bU_r$ and $\bW_r$ swapping roles, we have
\begin{align*}
    \Pr\left[\left|\frac{1}{|\oS|}\sum_{t\in\oS}(\bx_i^\top\bU_r\bLambda\bW_r^\top\bx_j)(\bx_k^\top\bU_r\bLambda\bW_r^\top\bx_j)\right|\ge\frac{r}{n^{1/2}}\|\bV\bA_\star^\top\|_\textnormal{F}^2\right]\lesssim\exp(-n^{1/4}),
\end{align*}
which satisfies (\ref{eqn:concentration_condition}).

\underline{Case (iix):} We have
\begin{align*}
    \frac{1}{n(n-1)}\sum_{a\neq b}(\bx_a^\top\bV\bA_\star^\top\bx_a)(\bx_a^\top\bV\bA_\star^\top\bx_b)&=\frac{1}{n(n-1)}\sum_{a\neq b}(\bx_b^\top\bA_\star\bV^\top\bx_a)(\bx_a^\top\bA_\star\bV^\top\bx_a).
\end{align*}
Then, we can repeat the arguments from Case (vi), replacing $\bV\bA_\star^\top$ with $\bA_\star\bV^\top$, to see that \[\Pr\left[\left|\frac{1}{n(n-1)}\sum_{a\neq b}(\bx_a^\top\bV\bA_\star^\top\bx_a)(\bx_a^\top\bV\bA_\star^\top\bx_b)\right|\ge C_4\sqrt{\frac{r}{n}}\|\bV\bA_\star^\top\|_\textnormal{F}^2\right]\lesssim\exp(-n^{1/8}),\]
which satisfies (\ref{eqn:concentration_condition}).

\underline{Case (ix):} By Lemma \ref{lem:sub_weibull},
\begin{align*}
    \Pr\left[\left|\frac{1}{n}\sum_{a=1}^n (\bx_a^\top\bV\bA_\star^\top\bx_a)^2-\EE[(\bx_a^\top\bV\bA_\star^\top\bx_a)^2]\right|\ge C_5n^{-3/8}\|\bV\bA_\star^\top\|_\textnormal{F}^2\right]\le\exp(-n^{1/4}),
\end{align*}
which satisfies (\ref{eqn:concentration_condition}).

\underline{Case (x):} We have
\begin{align*}
    \frac{1}{n(n-1)}\sum_{a\neq b}(\bx_a^\top\bV\bA_\star^\top\bx_a)(\bx_b^\top\bV\bA_\star^\top\bx_b)&=\frac{1}{n}\sum_{a\in[n]}(\bx_a^\top\bV\bA_\star^\top\bx_a)\frac{1}{n-1}\sum_{b\neq a}(\bx_b^\top\bV\bA_\star^\top\bx_b)\\
    &=\underbrace{\frac{1}{n}\sum_{a\in[n]}(\bx_a^\top\bV\bA_\star^\top\bx_a)\EE[\bx_b^\top\bV\bA_\star^\top\bx_b]}_{\beta_1}\\
    &\hspace{0mm}+\underbrace{\frac{1}{n}\sum_{a\in[n]}(\bx_a^\top\bV\bA_\star^\top\bx_a)\left(\frac{1}{n-1}\sum_{b\neq a}\bx_b^\top\bV\bA_\star^\top\bx_b-\EE[\bx_b^\top\bV\bA_\star^\top\bx_b]\right).}_{\beta_2}
\end{align*}

By Lemma \ref{lem:bernstein}, \[\Pr[\frac{1}{n-1}\sum_{b\neq a}\bx_b^\top\bV\bA_\star^\top\bx_b-\EE[\bx_b^\top\bV\bA_\star^\top\bx_b]\ge n^{-1/4}\|\bV\bA_\star^\top\|_\textnormal{F}]\lesssim\exp(-n^{1/2}).\] Using a union bound, 
\[\Pr[\max_a\left(\frac{1}{n-1}\sum_{b\neq a}\bx_b^\top\bV\bA_\star^\top\bx_b-\EE[\bx_b^\top\bV\bA_\star^\top\bx_b]\right)\ge n^{-1/4}\|\bV\bA_\star^\top\|_\textnormal{F}]\lesssim\exp(-n^{1/4}).\]

Again applying Lemma \ref{lem:bernstein}, we have
\[\Pr[\frac{1}{n}\sum_{a\in[n]}\bx_a^\top\bV\bA_\star^\top\bx_a-\EE[\bx_a^\top\bV\bA_\star^\top\bx_a]\ge n^{-1/4}\|\bV\bA_\star^\top\|_\textnormal{F}]\lesssim\exp(-n^{1/2})\]

We also have $\EE[\bx_a^\top\bV\bA_\star^\top\bx_a]\lesssim\|\bV\bA_\star^\top\|$ by Lemma \ref{lem:exp_xVAx}. Therefore,
\begin{align*}
    |\beta_1-\EE[\beta_1]|\le n^{-1/4}\|\bV\bA_\star^\top\|_\textnormal{F}^2
\end{align*}
with probability at least $1-O(\exp(-n^{1/2}))$, and
\begin{align*}
    |\beta_2-\EE[\beta_2]|\le n^{-1/2}\|\bV\bA_\star^\top\|_\textnormal{F}^2
\end{align*} with probability at least $1-O(\exp(-n^{1/4}))$. The combination of these bounds satisfies (\ref{eqn:concentration_condition}). 

\end{proof}

\begin{lem}\label{lem:trA_close_to_exp_trAstar}
    Suppose $\|\bA-\bA_\star\|\le C\sqrt{\frac{1}{n\log n}}$. Then,
    \[\left|\frac{1}{|\oS|}\sum_{t\in\oS}(\Tr(\bA^\top\bM_t\bV))^2-\EE\left[(\Tr(\bA_\star^\top\bM_t\bV))^2\right]\right|\le f(n)\cdot\|\bV\|_\textnormal{F}^2\] with probability at least $1-O(\exp(-n))$, where $f(n)$ is $o(1)$.
\end{lem}
\begin{proof}
    First, by Lemma \ref{lem:trAstar2_close_to_trA2}, we have
    \begin{align}\label{eqn:tr2_a_to_astar}
        \max_{t\in S}\left|(\Tr(\bA^\top\bM_t\widetilde{\bZ}))^2-(\Tr(\bA_\star^\top\bM_t\widetilde{\bZ}))^2\right|\lesssim \sqrt{\frac{p^4\log n}{n}}\|\widetilde{\bZ}\|_\textnormal{F}^2.
    \end{align}

    By Lemma \ref{lem:concentration_tr2_astar}, we know that 
    \begin{align}\nonumber
        \left|\frac{1}{|\oS|}\sum_{t\in\oS}\Tr((\bA_\star^\top\bM_t\bV))^2-\EE\left[(\Tr(\bA_\star^\top\bM_t\bV))^2\right]\right|&\le\frac{C r\log^2 r}{n^{1/4}}\|\bV\bA_\star^\top\|^2_\textnormal{F}\\
        &\le \frac{Cr\log^2 r}{n^{1/4}}\|\bA_\star\|^2\|\bV\|_\textnormal{F}^2\label{eqn:concentration_tr2_astar}
    \end{align} with probability at least $1-O(\exp(-n^{1/8}))$.

    Combining \eqref{eqn:tr2_a_to_astar} and \eqref{eqn:concentration_tr2_astar}, we conclude that, with probability at least $1-O(\exp(-n))$, 
    \[\left|\frac{1}{|\oS|}\sum_{t\in\oS}(\Tr(\bA^\top\bM_t\bV))^2-\EE\left[(\Tr(\bA_\star^\top\bM_t\bV))^2\right]\right|\le C\left(\frac{r\log^2 r}{n^{1/4}}+\sqrt{\frac{p^4\log^3 n}{n}}\right)\|\bA_\star\|^2\|\bV\|_\textnormal{F}^2.\] This is $o(1)\cdot\|\bV\|_\textnormal{F}$.
\end{proof}

\begin{lem}\label{lem:sum_sbar_tr_a_upper}
    Suppose $\|\bA-\bA_\star\|\le C_1\sqrt{\frac{1}{n\log n}}$. Then with probability at least $1-O(\exp(-n))$,
    \[\frac{1}{|\oS|}\sum_{t\in\oS}(\Tr(\bA\bM_t\bV))^2\le C_2\|\bA_\star\|^2\|\bV\|_\textnormal{F}^2\] for some constant $C_2$.
\end{lem}
\begin{proof}
    By Lemma \ref{lem:expectation_tr2_upper_astar},
    \begin{equation}
        \EE[(\Tr(\bA_\star^\top\bM_t\bV))^2]\le C\|\bA_\star\|^2\|\bV\|_\textnormal{F}^2.\label{eqn:exp_upper}
    \end{equation}

    Combining \eqref{eqn:exp_upper} with Lemma \ref{lem:trA_close_to_exp_trAstar}, we obtain the upper bound
    \begin{align*}
        \frac{1}{|\oS|}\sum_{t\in\oS}(\Tr(\bA\bM_t\bV))^2&\le C'\|\bA_\star\|^2\|\bV\|_\textnormal{F}^2
    \end{align*}
    with probability at least $1-O(\exp(-n))$.
\end{proof}
\begin{lem}\label{lem:trAstar2_close_to_trA2}
    If $\|\bA-\bA_\star\|\le C\sqrt{\frac{1}{n\log n}}$, then
    \begin{align}\nonumber
        \max_{t\in S}\left|(\Tr(\bA^\top\bM_t\widetilde{\bZ}))^2-(\Tr(\bA_\star^\top\bM_t\widetilde{\bZ}))^2\right|\lesssim \sqrt{\frac{p^4\log^3 n}{n}}\|\widetilde{\bZ}\|_\textnormal{F}^2.
    \end{align}
\end{lem}
\begin{proof}
\begin{align}\nonumber
    \max_{t\in S}\left|(\Tr(\bA^\top\bM_t\widetilde{\bZ}))^2-(\Tr(\bA_\star^\top\bM_t\widetilde{\bZ}))^2\right|
    &=\max_{t\in S}\left|(\Tr((\bA-\bA_\star)^\top\bM_t\widetilde{\bZ}))(\Tr((\bA+\bA_\star)^\top\bM_t\widetilde{\bZ}))\right|\\\nonumber
    &\le \max_{t\in S}\|\bA-\bA_\star\|_\textnormal{F}\|\bA+\bA_\star\|_\textnormal{F}\|\bM_t\|_\textnormal{F}^2\|\widetilde{\bZ}\|_\textnormal{F}^2\\\nonumber
    &\lesssim\sqrt{\frac{1}{n\log n}}(\max_{t\in S}\|\bM_t\|_\textnormal{F}^2)\|\widetilde{\bZ}\|_\textnormal{F}^2.
\end{align}
Applying Lemma \ref{lem:single_norm_mt} to all $t$ uniformly, we conclude that, with probability at least $1-O(n^{-10})$, $\max_{t\in S}\|\bM_t\|_\textnormal{F}^2\lesssim p^2(\log n)^2$. This means 
\begin{align}\nonumber
    \max_{t\in S}\left|(\Tr(\bA^\top\bM_t\widetilde{\bZ}))^2-(\Tr(\bA_\star^\top\bM_t\widetilde{\bZ}))^2\right|
    &\lesssim \sqrt{\frac{p^4\log^3 n}{n}}\|\widetilde{\bZ}\|_\textnormal{F}^2.
\end{align}
\end{proof}

\begin{lem}\label{lem:sbar_to_s_Astar}
    For some constant $C$,
    \begin{align*}
        &\Pr\left[\left|\frac{1}{|S|}\sum_{t\in S}(\Tr(\bM_t\bV\bA_\star^\top))^2-\frac{1}{|\oS|}\sum_{t\in \oS}(\Tr(\bM_t\bV\bA_\star^\top))^2\right|\ge \frac{C}{n}\|\bV\|_\textnormal{F}^2\right]\lesssim\exp(-n^{1/8}).
    \end{align*}
\end{lem}
\begin{proof}
For a triplet $t$, let $b_t$ be an indicator random variable that is 1 if and only if $t\in S$. By assumption, $b_t\sim\text{Bernoulli}(s)$, independent of all other $b_{t'}$. Then we can write
\begin{align*}
    &\frac{1}{|S|}\sum_{t\in S}(\Tr(\bM_t\bV\bA_\star^\top))^2\\
    &=\frac{1}{|S|}\sum_{t\in\oS}b_t(\Tr(\bM_t\bV\bA_\star^\top))^2\\
    &=\underbrace{\left(\frac{1}{|S|}\sum_{t\in\oS}(b_t-s)(\Tr(\bM_t\bV\bA_\star^\top))^2\right)}_{\eta_1}+\underbrace{\left(\frac{s}{|S|}\sum_{t\in\oS}(\Tr(\bM_t\bV\bA_\star^\top))^2-\frac{1}{|\oS|}\sum_{t\in\oS}(\Tr(\bM_t\bV\bA_\star^\top))^2\right)}_{\eta_2}\\
    &\hspace{10mm}+\frac{1}{|\oS|}\sum_{t\in\oS}(\Tr(\bM_t\bV\bA_\star^\top))^2.
\end{align*}
Therefore,
\[\left|\frac{1}{|S|}\sum_{t\in S}(\Tr(\bM_t\bV\bA_\star^\top))^2-\frac{1}{|\oS|}\sum_{t\in \oS}(\Tr(\bM_t\bV\bA_\star^\top))^2\right|\le|\eta_1|+|\eta_2|.\]
To bound $\eta_1$, we can write
\begin{align*}
    |\eta_1|&\le\frac{1}{|S|}\left(\max_{t\in\oS}(\Tr(\bM_t\bV\bA_\star^\top))^2\right)\left|s|\oS|-\sum_{t\in\oS}b_t\right|=\frac{1}{|S|}\left(\max_{t\in\oS}(\Tr(\bM_t\bV\bA_\star^\top))^2\right)\left|s|\oS|-|S|\right|\\
\end{align*}
Since $|S|$ follows a binomial distribution with mean $s|\oS|$, we have
\begin{equation}
    \Pr\left[\Big||S|-s|\oS|\Big|\ge q_1s|\oS|\right]\lesssim \exp(-q_1^2|\oS|)= \exp(-q_1^2n^3).\label{eqn:lemma_s_to_sbar_trmtva_binomial}
\end{equation}
This also means that, with probability at least $1-O(\exp(-q_1^2n^3))$, we have
\begin{align}
    \left|\frac{1}{|S|}-\frac{1}{s|\oS|}\right|&=\left|\frac{s|\oS|-|S|}{|S|s|\oS|}\right|\le\frac{q_1}{|S|}.\label{eqn:lemma_s_to_sbar_trmtva_sizeS}
\end{align}

By Lemma \ref{lem:trace_bound_nonsymmetric}, we know
\[\Pr\left[\Tr(\bM_t\bV\bA_\star^\top)\ge q_2\|\bV\|_\textnormal{F}\|\bA_\star\|\right]\lesssim\exp(-q_2).\] Applying a union bound, this means
\begin{equation}\label{eqn:lemma_s_to_sbar_trmtva_trmtva}
    \Pr\left[\max_{t\in{\oS}}\Tr(\bM_t\bV\bA_\star^\top)\ge q_2\|\bV\|_\textnormal{F}\|\bA_\star\|\right]\lesssim n^3\exp(-q_2).
\end{equation}

Using the union bound on the bounds (\ref{eqn:lemma_s_to_sbar_trmtva_binomial}),  (\ref{eqn:lemma_s_to_sbar_trmtva_sizeS}), and \eqref{eqn:lemma_s_to_sbar_trmtva_trmtva}, we have with probability at least $1-O(\exp(-q_1^2n^3)+n^3\exp(-q_2))$, we have
\begin{align*}
    |\eta_1|&\le \frac{1}{|S|}\left(\max_{t\in\oS}(\Tr(\bM_t\bV\bA_\star^\top))^2\right)\left|s|\oS|-|S|\right|\\
    &\le \frac{1}{s|\oS|}\left(\max_{t\in\oS}(\Tr(\bM_t\bV\bA_\star^\top))^2\right)\left|s|\oS|-|S|\right|\\
    &\hspace{20mm}+\left(\frac{1}{|S|}-\frac{1}{s|\oS|}\right)\left(\max_{t\in\oS}(\Tr(\bM_t\bV\bA_\star^\top))^2\right)\left|s|\oS|-|S|\right|\\
    &\le q_1q_2^2\|\bA_\star\|^2\|\bV\|^2_\textnormal{F}+q_1|\eta_1|\\
    (1-q_1)|\eta_1|&\le q_1q_2^2\|\bA_\star\|^2\|\bV\|^2_\textnormal{F}\\
    |\eta_1|&\le\frac{q_1q_2^2\|\bA_\star\|^2\|\bV\|^2_\textnormal{F}}{1-q_1}.
\end{align*}
To bound $|\eta_2|$, we have
\begin{align*}
    \eta_2&=\left(\frac{s}{|S|}-\frac{1}{|\oS|}\right)\sum_{t\in\oS}(\Tr(\bM_t\bV\bA_\star^\top))^2=\left(\frac{s|\oS|}{|S|}-1\right)\frac{1}{|\oS|}\sum_{t\in\oS}(\Tr(\bM_t\bV\bA_\star^\top))^2.
\end{align*}
First, we will bound $\left|\frac{s|\oS|}{|S|}-1\right|$. Using (\ref{eqn:lemma_s_to_sbar_trmtva_binomial}) with $q_1=1/2$, we know that
\begin{align*}
    \Pr\left[\left||S|-s|\oS|\right|\ge \frac{1}{2}s|\oS|\right]\lesssim \exp(-n^3).
\end{align*} 
This means that, with probability at least $1-O(\exp(-n^3))$, we have
$\frac{|S|}{s|\oS|}\in\left[\frac{1}{2},\frac{3}{2}\right]$, implying $\frac{s|\oS|}{|S|}\in\left[\frac{2}{3},2\right]$.
If this holds,
\[\left|\frac{s|\oS|}{|S|}-1\right|\le2\left|\frac{|S|}{s|\oS|}-1\right|=2\left|\frac{|S|-s|\oS|}{s|\oS|}\right|.\]
Again applying (\ref{eqn:lemma_s_to_sbar_trmtva_binomial}), we have that 
\[\Pr\left[\left|\frac{|S|-s|\oS|}{s|\oS|}\right|\ge \frac{q_1}{2}\right]\lesssim \exp(-q_1^2n^3).\]
Taking a union bound over the events $\frac{s|\oS|}{|S|}\in\left[\frac{2}{3},2\right]$ and $\left|\frac{|S|-s|\oS|}{s|\oS|}\right|\ge q_1/2$, we conclude that
\begin{equation}
    \Pr\left[\left|\frac{s|\oS|}{|S|}-1\right|\le q_1\right]\lesssim \exp(-n^3)+\exp(-q_1^2n^3).\label{eqn:lemma_s_to_sbar_trmtva_Ssize_2}
\end{equation}

Now we turn to the rest of $\eta_2$. By Lemma \ref{lem:sum_sbar_tr_a_upper},
\begin{equation}
    \frac{1}{|\oS|}\sum_{t\in\oS}(\Tr(\bM_t\bV\bA_\star^\top))^2\le C\|\bV\|_\textnormal{F}^2\label{eqn:lemma_6_sum_trace_bound}
\end{equation}
with probability at least $1-O(\exp(-n))$.

Combining (\ref{eqn:lemma_s_to_sbar_trmtva_Ssize_2}) and (\ref{eqn:lemma_6_sum_trace_bound}) using the union bound, we have
\begin{align}
    &\Pr\left[|\eta_2|\ge Cq_1\|\bV\|_\textnormal{F}^2\right]\lesssim \exp(-n)+\exp(-q_1^2n^3)+\exp(-n^3). %
\end{align}
Combining the bounds on $|\eta_1|$ and $|\eta_2|$, we see that, with probability at least
\[1-O\left(\exp(-n)+\exp(-q_1^2n^3)+n^3\exp(-q_2)+\exp(-n^3)\right),\]
we have
\begin{align}\nonumber
    &\left|\frac{1}{|S|}\sum_{t\in S}(\Tr(\bM_t\bV\bA_\star^\top))^2-\frac{1}{|\oS|}\sum_{t\in\oS}(\Tr(\bM_t\bV\bA_\star^\top))^2\right|\\\nonumber
    &\hspace{10mm}\le \frac{q_1q_2^2\|\bA_\star\|^2\|\bV\|^2_\textnormal{F}}{1-q_1}+Cq_1\|\bV\|_\textnormal{F}^2\label{eqn:lemma_s_to_sbar_trmtva_bound_mtv}
\end{align}

Setting $q_1=1/n$, $q_2=n^{1/4}$, we have that with probability at least $1-O(\exp(-n^{1/8}))$,
\begin{align}\nonumber
    &\left|\frac{1}{|S|}\sum_{t\in S}(\Tr(\bM_t\bV\bA_\star^\top))^2-\frac{1}{|\oS|}\sum_{t\in\oS}(\Tr(\bM_t\bV\bA_\star^\top))^2\right|\le \frac{C}{n}\|\bV\|_\textnormal{F}^2.
\end{align}
\end{proof}

\begin{lem}\label{lem:single_norm_mt}
    \[\Pr\left[\|\bM_{t}\|_\textnormal{F}\ge Cp\log n\right]\lesssim n^{-13}\]
    for some constant $C$.
\end{lem}
\begin{proof}
We have
\begin{align*}
    \|\bM_{t}\|_\textnormal{F}&\le 2\|\bx_{i}\bx_{k}^\top\|_\textnormal{F}+2\|\bx_{i}\bx_{j}^\top\|_\textnormal{F}+\|\bx_{j}\bx_{j}^\top\|_\textnormal{F}+\|\bx_{k}\bx_{k}^\top\|_\textnormal{F}\\
    &\le 2\|\bx_i\|\|\bx_k\|+2\|\bx_i\|\|\bx_j\|+\|\bx_j\|^2+\|\bx_k\|^2
\end{align*}
For $a\in[n]$, $\Pr\left[\|\bx_a\|\ge C_1(1+q)\sqrt{p}\right]\le\exp(-q^2)$ by \citet[Proposition 6.2.1]{vershynin2025highNEW}. Setting $q:=14\sqrt{\log n}$, we get $\Pr\left[\|\bx_a\|\ge C_2\sqrt{p\log n}\right]\le n^{-14}$. Taking the union bound over $a\in\{i,j,k\}$, we find that $\Pr\left[\max_{a\in\{i,j,k\}}\|\bx_a\|\ge C_2\sqrt{p\log n}\right]\le n^{-13}$.

As long as this holds, $\|\bM_t\|_\textnormal{F}\le 6C_2^2p\log n$. Taking $C:=6C_2^2$, we obtain the lemma.

\end{proof}

\begin{lem}\label{lem:lower_bound_a}
    There exists a constant $C_1$ and a constant $0<C_2<1$ such that
    \[\frac{1}{|S|}\sum_{t\in S}\mathbb I\left\{(\Tr(\bM_t\widetilde\bZ\bA^\top))^2\ge C_1\|\bA_\star\|^2\|\widetilde\bZ\|^2_\textnormal{F}\right\}\ge C_2\kappa^{-2}\] with probability at least $1-O(n^{-10})$.
\end{lem}
\begin{proof}
We will start by lower bounding $(\Tr(\bM_t\bV\bA_\star^\top))^2$ for a single $t$ using the Paley-Zygmund inequality. This says 
\[\Pr\left[|\Tr(\bA_\star^\top\bM_t\widetilde{\bZ})|>\theta \EE[|\Tr(\bA_\star^\top\bM_t\widetilde{\bZ})|]\right]\ge  (1-\theta)^2\frac{(\EE[|\Tr(\bA_\star^\top\bM_t\widetilde{\bZ})|])^2}{\EE[\Tr(\bA_\star^\top\bM_t\widetilde{\bZ})^2]}.\]

An upper bound on $\EE[(\Tr(\bA_\star^\top\bM_t\widetilde{\bZ}))^2]$ is provided by Lemma \ref{lem:expectation_tr2_upper_astar}, which says
\begin{align*}
    \EE[(\Tr(\bA_\star^\top\bM_t\widetilde{\bZ}))^2]&\le C\|\bA_\star\|^2\|\widetilde\bZ\|_\textnormal{F}^2.
\end{align*}

A lower bound on $\EE[|\Tr(\bA_\star^\top\bM_t\widetilde{\bZ})|]$ is provided by Lemma \ref{lem:expected_abs_value_trace_astar}, which says that
\begin{align*}
    \EE[|\Tr(\bA_\star^\top\bM_t\widetilde{\bZ})|]&\ge C'\sigma_{\min}(\bA_\star)\|\widetilde\bZ\|_\textnormal{F}= C'\kappa^{-1}\|\bA_\star\|\|\widetilde\bZ\|_{\textnormal{F}}.
\end{align*}

Plugging these bounds into the Paley-Zygmund inequality, we get
\begin{align*}
    \Pr\left[|\Tr(\bA_\star^\top\bM_t\widetilde{\bZ})|> C'\theta\kappa^{-1}\|\bA_\star\|\|\widetilde\bZ\|_{\textnormal{F}}\right]
    &\ge (1-\theta)^2\frac{\left(C'\kappa^{-1}\|\bA_\star\|\|\widetilde\bZ\|_{\textnormal{F}}\right)^2}{C\|\bA_\star\|^2\|\widetilde\bZ\|_\textnormal{F}^2}\\
    &\ge \frac{C'}{C}(1-\theta)^2\kappa^{-2}.
\end{align*}
Taking $\theta=1/2$, we conclude that 
\begin{align}\label{eqn:prob_of_indicator}
    \Pr\left[|\Tr(\bA_\star^\top\bM_t\widetilde{\bZ})|>C_1'\kappa^{-1}\|\bA_\star\|\|\widetilde\bZ\|_\textnormal{F}\right]\ge C_2'\kappa^{-2}.
\end{align}
for some constant $C_1'$ and some constant $0<C_2'<1$ (we need the the upper bound of one for technical reasons; it is without loss of generality by taking the minimum of $C_2'$ and 1).

Now, we will use this to show that
\begin{equation}
    \frac{1}{|S|}\sum_{t\in S}\mathbb I\left\{(\Tr(\bM_t\widetilde{\bZ}\bA_\star^\top))^2\ge {C_1'}^2\kappa^{-2}\|\bA_\star\|^2\|\widetilde\bZ\|_\textnormal{F}^2\right\}\ge C_2'\kappa^{-2}\label{eqn:sum_of_indicators}
\end{equation} with high probability.

Note that
\begin{equation}
    \EE\left[\mathbb I\left\{(\Tr(\bM_t\widetilde{\bZ}\bA_\star^\top))^2\ge {C_1'}^2\kappa^{-2}\|\bA_\star\|^2\|\widetilde\bZ\|_\textnormal{F}^2\right\}\right]\ge C'_2\kappa^{-2}\label{eqn:expectation_of_indicator}
\end{equation} by (\ref{eqn:prob_of_indicator}). 

We will show a concentration of the sum in (\ref{eqn:sum_of_indicators}) around its mean. We will do this in two steps. First, we will show that a sum over the set $\oS:=\{(i,j,k)\mid i\neq j\neq k\neq i\in[n]\}$ is close to the mean. 

Using the results in \citet[Section 5a]{hoeffding} with
\[g(\bx_i,\bx_j,\bx_k):=\mathbb I\{(\Tr(\bM_t\widetilde\bZ\bA_\star^\top))^2\ge {C_1'}^2\kappa^{-2}\|\bA_\star\|^2\|\widetilde\bZ\|_\textnormal{F}^2\},\] we conclude that
\begin{align*}
    \Pr\left[\frac{1}{|\oS|}\sum_{t\in\oS}g(\bx_i,\bx_j,\bx_k)-\EE[g(\bx_i,\bx_j,\bx_k)]\ge q\right]\lesssim\exp(-nq^2).
\end{align*}

We can obtain a matching lower bound using $g'=-g$. We can combine both bounds and take $q=n^{-1/4}$ to conclude that
\begin{align}
    \Pr\left[\left|\frac{1}{|\oS|}\sum_{t\in\oS}g(\bx_i,\bx_j,\bx_k)-\EE[g(\bx_i,\bx_j,\bx_k)]\right|\ge \frac{1}{n^{1/4}}\right]\lesssim \exp(-n^{1/2}).\label{eqn:u_stat_concentration}
\end{align}

Now we'll show that the sum over $\oS$ is close to the sum over $S$. Using Lemma \ref{lem:sbar_to_s_indicator}, we have
\begin{equation}
    \Pr\left[\left|\frac{1}{|S|}\sum_{t\in S}g(\bx_i,\bx_j,\bx_k)-\frac{1}{|\oS|}\sum_{t\in \oS}g(\bx_i,\bx_j,\bx_k)\right|\ge \frac{1}{n}\right]\lesssim \exp(-n).\label{eqn:sbar_to_s_indicator}
\end{equation}
Combining (\ref{eqn:expectation_of_indicator}), (\ref{eqn:u_stat_concentration}), and (\ref{eqn:sbar_to_s_indicator}) we get
\begin{align}\label{eqn:lower_bd_on_large_fraction_of_terms}
    &\frac{1}{|S|}\sum_{t\in S}\mathbb I\left\{(\Tr(\bM_t\widetilde{\bZ}\bA_\star^\top))^2\ge {C_1'}^2\kappa^{-2}\|\bA_\star\|^2\|\widetilde\bZ\|_\textnormal{F}^2\right\}\ge C_2'\kappa^{-2}-\frac{1}{n^{1/4}}-\frac{1}{n}\ge \frac{C_2'}{2}\kappa^{-2}
\end{align} with probability at least $1-O(\exp(-n))$, for sufficiently large $n$. 

Finally, to move from $(\Tr(\bA_\star^\top\bM_t\widetilde{\bZ}))^2$ to $(\Tr(\bA^\top\bM_t\widetilde{\bZ}))^2$, we use Lemma \ref{lem:trAstar2_close_to_trA2}, which says that
\begin{align}\nonumber
    \max_{t\in S}\left|(\Tr(\bA^\top\bM_t\widetilde{\bZ}))^2-(\Tr(\bA_\star^\top\bM_t\widetilde{\bZ}))^2\right|\lesssim \sqrt{\frac{p^4\log^3 n}{n}}\|\widetilde{\bZ}\|_\textnormal{F}^2.
\end{align}
with probability at least $1-O(n^{-10})$. The term $\sqrt{\frac{p^4\log^3 n}{n}}$ is $o(1)$, so by combining this with \eqref{eqn:lower_bd_on_large_fraction_of_terms}, we conclude that
\begin{align*}
    \frac{1}{|S|}\sum_{t\in S}\mathbb I\left\{(\Tr(\bM_t\widetilde{\bZ}\bA_\star^\top))^2\ge \frac{{C_1'}^2}{2}\kappa^{-2}\|\bA_\star\|^2\|\widetilde\bZ\|^2_\textnormal{F}\right\}\ge\frac{C_2'}{2}\kappa^{-2}
\end{align*}
with probability at least $1-O(n^{-10})$.

Taking $\frac{{C_1'}^2}{2}\kappa^{-2}$ to be a new constant $C_1$ and $\frac{C'_2}{2}$ to be a new constant $0<C_2<1$, we obtain the lemma.
\end{proof}

\begin{lem}\label{lem:lower_bound_b}
    For any constant $0<C_2<1$, there exists a constant $C_1$ such that
    \begin{align}\nonumber
        \Pr\left[\frac{1}{|S|}\sum_{t\in S}\mathbb I\left\{\frac{4\exp(y_t\Tr(\bM_t \bA\bA^\top))}{(\exp(y_t\Tr(\bM_t \bA\bA^\top))+1)^2}\ge \frac{1}{\exp\left(C_1\|\bA_\star\|^2_\textnormal{F}\right)}\right\}\le1-C_2\kappa^{-2}\right]\lesssim n^{-10}.\label{eqn:coeff_lower_bound_count}
    \end{align}
\end{lem}
\begin{proof}
Observe that
\begin{equation}
    \frac{4\exp(y_t\Tr(\bM_t \bA\bA^\top))}{(\exp(y_t\Tr(\bM_t \bA\bA^\top))+1)^2}\ge\frac{1}{\exp(|\Tr(y_t\bM_t\bA\bA^\top)|)}.\label{eqn:coeff_simplify}
\end{equation}

First, we move to $\bA_\star$. We have
\begin{align}\nonumber
    &\left|\Tr(y_t\bM_t\bA\bA^\top)-\Tr(y_t\bM_t\bA_\star\bA_\star^\top)\right|\\\nonumber
    &=\underbrace{\left|\Tr(y_t\bM_t\bA(\bA-\bA_\star)^\top)\right|}_{\eta_1}+\underbrace{\left|\Tr(y_t\bM_t(\bA-\bA_\star)\bA_\star^\top)\right|}_{\eta_2}.
\end{align}

$\eta_1$ can be upper bounded as follows:
\begin{align*}
    \eta_1&\le \|\bM_t\|_\textnormal{F}\|\bA\|_\textnormal{F}\|\bA-\bA_\star\|_\textnormal{F}.
\end{align*}
By Lemma~\ref{lem:single_norm_mt} applied uniformly across all $t$, $\|\bM_t\|_\textnormal{F}\le C_1p\log n$ for some constant $C_1$, with probability at least $1-O(n^{-10})$. By assumption, $\|\bA-\bA_\star\|_\textnormal{F}\le C_2\sqrt{\frac{1}{n\log n}}$. This also implies $\|\bA\|_\textnormal{F}\le 2\|\bA_\star\|_\textnormal{F}$. Putting these together, we conclude
\[\eta_1\le C_3\sqrt{\frac{p^2\log n}{n}}\|\bA_\star\|_\textnormal{F}\] with probability at least $1-O(n^{-10})$.

The bound on $\eta_2$ proceeds similarly. Combining both bounds, we have
\begin{align}
    \left|\Tr(y_t\bM_t\bA\bA^\top)-\Tr(y_t\bM_t\bA_\star\bA_\star^\top)\right|&\le C_4r\sqrt{\frac{p^2\log n}{n}}\|\bA_\star\|_\textnormal{F}.\label{eqn:trmaa_to_trmastarastar}
\end{align} with probability at least $1-O(n^{-10})$. This bound is $o(1)$.

By Lemma \ref{lem:tr_mtaa_indicator},
\[\Pr\left[\frac{1}{|S|}\sum_{t\in S}\mathbb I\{|\Tr(\bM_t\bA_\star\bA_\star^\top)|\le C_{4}\|\bA_\star\|_{\textnormal{F}}^2\}\le 1-\frac{C_2\kappa^{-2}}{4}\right]\lesssim\exp(-n).\]

Combining this with the bound \eqref{eqn:trmaa_to_trmastarastar}, we conclude
\begin{equation}
    \Pr\left[\frac{1}{|S|}\sum_{t\in S}\mathbb I\{|\Tr(\bM_t\bA\bA^\top)|\le 2C_4\|\bA_\star\|_\textnormal{F}^2\}\le1-\frac{C_2\kappa^{-2}}{4}\right]\lesssim n^{-10}.
\end{equation}

Recalling (\ref{eqn:coeff_simplify}), we have 
\begin{align}\nonumber
    \Pr\left[\frac{1}{|S|}\sum_{t\in S}\mathbb I\left\{\frac{4\exp(y_t\Tr(\bM_t \bA\bA^\top))}{(\exp(y_t\Tr(\bM_t \bA\bA^\top))+1)^2}\ge \frac{1}{\exp\left(2C_4\|\bA_\star\|^2_\textnormal{F}\right)}\right\}\le1-\frac{C_2\kappa^{-2}}{4}\right]\lesssim n^{-10}.
\end{align}
\end{proof}

\begin{lem}\label{lem:tr_mtaa_indicator}
    For any constant $0<C_2<1$, there exists a constant $C_1$ such that
    \[\Pr\left[\frac{1}{|S|}\sum_{t\in S}\mathbb I\{|\Tr(\bM_t\bA_\star\bA_\star^\top)|\le C_{1}\|\bA_\star\|_{\textnormal{F}}^2\}\le 1-C_2\kappa^{-2}\right]\lesssim\exp(-n).\]
\end{lem}
\begin{proof}
Let $\bU\bLambda\bU^\top$ be the singular value decomposition of $\bA_\star\bA_\star^\top$. First, we'll show the following:
\[\Pr\left[\frac{1}{n}\sum_{i\in[n]}\mathbb I\left\{\left|\left\|\bLambda^{1/2}\bU^\top\bx_i\right\|-\|\bLambda^{1/2}\|_\textnormal{F}\right|\ge C_3\|\bLambda^{1/2}\|_\textnormal{F}\right\}\ge\left(1-C_2\kappa^{-2}r^{-1}\right)^{1/3}\right]\lesssim\exp(-n)\] for some constant $C_3$.

By \citet[Exercise 6.10]{vershynin2025highNEW}, we have

\[\Pr\left[\|\bLambda^{1/2}\bU^\top\bx\|\ge (C_4+q)\|\bLambda^{1/2}\|_{\textnormal{F}}\right]\le C_5\exp(-q^2).\] 

Setting \begin{equation}
    q=\left(-\log\left(\frac{1}{C_5}\left(1-\left(1-C_2\kappa^{-2}\right)^{1/5}\right)\right)\right)^{1/2}=:\alpha,\label{eqn:alpha_def}
\end{equation} we get
\[\Pr\left[\|\bLambda^{1/2}\bU^\top\bx\|\ge (C_4+\alpha)\|\bLambda^{1/2}\|_{\textnormal{F}}\right]\le1-\left(1-C_{2}\kappa^{-2}\right)^{1/5}).\]

Therefore
\[u_i:=\mathbb I\left\{\left\|\bLambda^{1/2}\bU^\top\bx_i\right\|\le (C_4+\alpha)\|\bLambda^{1/2}\|_\textnormal{F}\right\}\] is a Bernoulli random variable with \[\Pr[u_i=1]\ge\left(1-C_{2}\kappa^{-2}\right)^{1/5},\] and $\{u_i\}_i$ are independent. 

Since $1-C_2\kappa^{-2}$ is $\Theta(1)$ and is between $0$ and $1$, we have $\left|\left(1-C_2\kappa^{-2}\right)^{1/4}-\left(1-C_2\kappa^{-2}\right)^{1/5}\right|\ge C_6$ for some constant $C_6$. Therefore 
\[\Pr\left[\frac{1}{n}\sum_{i\in[n]}u_i\ge\left(1-C_2\kappa^{-2}\right)^{1/4}\right]\lesssim \exp(-n)\] by Hoeffding's inequality.

Since 
\[\left\|\bLambda^{1/2}\bU^\top\bx_i\right\|\le (C_4+\alpha)\|\bLambda^{1/2}\|_\textnormal{F}\] holds for a $\left(1-C_2\kappa^{-2}\right)^{1/4}$ fraction of the $\bx_i$'s, it also holds for all the $\bx$'s in 
a $\left(1-C_2\kappa^{-2}\right)^{3/4}$ fraction
of all possible triplets. 

Since $\left|\left(1-C_2\kappa^{-2}\right)^{3/4}-\left(1-C_2\kappa^{-2}\right)\right|\ge C_7$ for a constant $C_7$, the probability that the bound holds for all the $\bx$'s in a $1-C_2\kappa^{-2}$ fraction
of the triplets in $S$ is at least $1-O(\exp(-n))$ by Hoeffding's inequality.

Now we return to bounding
\[\Pr\left[\frac{1}{|S|}\sum_{t\in S}\mathbb I\{|\Tr(\bM_t\bA_\star\bA_\star^\top)|\le C_1\|\bA_\star\|_\textnormal{F}^2\}\le1-C_2\kappa^{-2}\right].\]
We can expand
\begin{align*}
    &|\Tr(\bM_t\bA_\star\bA_\star^\top)|\\
    &=|\bx_i^\top\bA_\star\bA_\star^\top\bx_k+\bx_k^\top\bA_\star\bA_\star^\top\bx_i-\bx_i^\top\bA_\star\bA_\star^\top\bx_j-\bx_j^\top\bA_\star\bA_\star^\top\bx_i+\bx_j^\top\bA_\star\bA_\star^\top\bx_j-\bx_k^\top\bA_\star\bA_\star^\top\bx_k|\\
    &\le2|\bx_i^\top\bA_\star\bA_\star^\top\bx_k|+2|\bx_i^\top\bA_\star\bA_\star^\top\bx_j|+|\bx_j^\top\bA_\star\bA_\star^\top\bx_j|+|\bx_k^\top\bA_\star\bA_\star^\top\bx_k|.
\end{align*}
Suppose we are in a triplet such that
\[\left\|\bLambda^{1/2}\bU^\top\bx_a\right\|\le (C_4+\alpha)\|\bLambda^{1/2}\|_\textnormal{F}\]
for all $a\in\{i,j,k\}$. Then, by Cauchy-Schwarz,
\begin{align*}
    |\bx_a^\top\bA_\star\bA_\star^\top\bx_b|
    &\le\|\bLambda^{1/2}\bU^\top\bx_a\|\cdot\|\bLambda^{1/2}\bU^\top\bx_b\|\\
    &\le (C_4+\alpha)^2\|\bLambda^{1/2}\|_\textnormal{F}^2\\
    &\le (C_4+\alpha)^2\|\bLambda\|_\textnormal{F}\\
    &=(C_4+\alpha)^2\|\bA_\star\bA_\star^\top\|_\textnormal{F}\\
    &\lesssim \|\bA_\star\bA_\star^\top\|_\textnormal{F}=\|\bA_\star\|_\textnormal{F}^2.
\end{align*}

Therefore, we conclude that 
\begin{equation*}
    \Pr\left[\frac{1}{|S|}\sum_{t\in S}\mathbb I\{|\Tr(\bM_t\bA_\star\bA_\star^\top)|\le C_1\|\bA_\star\|_\textnormal{F}^2\}\le1-C_2\kappa^{-2}\right]\lesssim\exp(-n)
\end{equation*}
for some constant $C_1$.
\end{proof}

\begin{lem}\label{lem:sbar_to_s_indicator}
    Let $\oS=\{t=(i,j,k)\mid i\neq j\neq k\in[n]\}$ and $S\subset\oS$ be selected by randomly including each triplet $t$ independently with probability $s$. Let $g(t)$ be a function of $t$ taking values in $[0,1]$. Then
    \begin{align}\nonumber
        \Pr\left[\left|\frac{1}{|S|}\sum_{t\in S}g(t)-\frac{1}{|\oS|}\sum_{t\in\oS}g(t)\right|\ge\frac{1}{n}\right]\lesssim \exp(-n)
    \end{align}
\end{lem}
\begin{proof}
For a triplet $t$, let $b_t$ be an indicator random variable that is 1 if $t\in S$. By assumption, $b_t\sim\text{Bernoulli}(s)$, independent of all other $b_{t'}$. Then we can write
\begin{align*}
    \frac{1}{|S|}\sum_{t\in S}g(t)
    &=\frac{1}{|S|}\sum_{t\in\oS}b_tg(t)\\
    &=\underbrace{\left(\frac{1}{|S|}\sum_{t\in\oS}(b_t-s)g(t)\right)}_{\eta_1}+\underbrace{\left(\frac{s}{|S|}\sum_{t\in\oS}g(t)-\frac{1}{|\oS|}\sum_{t\in\oS}g(t)\right)}_{\eta_2}+\frac{1}{|\oS|}\sum_{t\in\oS}g(t).
\end{align*}
Therefore,
\[\left|\frac{1}{|S|}\sum_{t\in S}g(t)-\frac{1}{|\oS|}\sum_{t\in \oS}g(t)\right|\le|\eta_1|+|\eta_2|.\]
To bound $\eta_1$, we can write
\begin{align*}
    |\eta_1|&\le\frac{1}{|S|}\left(\max_{t\in\oS}g(t)\right)\left|s|\oS|-\sum_{t\in\oS}b_t\right|\\
    &\le \frac{1}{|S|}\left|s|\oS|-\sum_{t\in\oS}b_t\right|
\end{align*}
    
Since $|S|$ follows a binomial distribution with mean $s|\oS|$, we have
\begin{equation}
    \Pr\left[\Big||S|-s|\oS|\Big|\ge q_1s|\oS|\right]\lesssim \exp(-q_1^2s|\oS|)=\exp(-q_1^2sn^3).\label{eqn:lemma_6_binomial_indicator}
\end{equation}
Therefore, with probability at least $1-O(\exp(-q_1^2n^3))$ we have
\begin{align}
    \left|\frac{1}{|S|}-\frac{1}{s|\oS|}\right|&=\left|\frac{s|\oS|-|S|}{|S|s|\oS|}\right|\le\frac{q_1}{|S|}.\label{eqn:concentration_alpha2_sizeS_indicator}
\end{align}

Using a Chernoff bound, we know
\begin{equation}\label{eqn:concentration_alpha2_chernoff_indicator}
    \Pr\left[\left|s|\oS|-\sum_{t\in\oS}b_t\right|\ge q_2 |\oS|\right]\lesssim \exp\left(-q_2^2|\oS|\right)=\exp(-q_2^2n^3).
\end{equation}

Using the union bound on the bounds (\ref{eqn:concentration_alpha2_sizeS_indicator}) and (\ref{eqn:concentration_alpha2_chernoff_indicator}), we have that with probability at least $1-O(\exp(-q_1^2n^3)+\exp(-q_2^2n^3))$
\begin{align*}
    \frac{1}{|S|}\left|s|\oS|-\sum_{t\in\oS}b_t\right|&\le \frac{1}{s|\oS|}\left|s|\oS|-\sum_{t\in\oS}b_t\right|+\left|\frac{1}{|S|}-\frac{1}{s|\oS|}\right|\left|s|\oS|-\sum_{t\in\oS}b_t\right|\\
    &\le \frac{q_2}{s}+q_1\frac{1}{|S|}\left|s|\oS|-\sum_{t\in\oS}b_t\right|\\
    (1-q_1)\frac{1}{|S|}\left|s|\oS|-\sum_{t\in\oS}b_t\right|&\le \frac{q_2}{s}\\
    \frac{1}{|S|}\left|s|\oS|-\sum_{t\in\oS}b_t\right|&\le \frac{q_2}{(1-q_1)s}\lesssim\frac{q_2}{(1-q_1)},
\end{align*}
so 
\begin{equation}
    \Pr\left[|\eta_1|\ge \frac{q_2}{(1-q_1)s}\right]\lesssim \exp(-q_1^2n^3)+\exp(-q_2^2n^3).
    \label{eqn:lemma_6_eta1_bound_indicator}
\end{equation}
To bound $|\eta_2|$, we have
\begin{align*}
    \eta_2&=\left(\frac{s}{|S|}-\frac{1}{|\oS|}\right)\sum_{t\in\oS}g(t)
    =\left(\frac{s|\oS|}{|S|}-1\right)\frac{1}{|\oS|}\sum_{t\in\oS}g(t).
\end{align*}
By Lemma \ref{lem:size_S_ratio}, we have $\Pr\left[\left|\frac{s|\oS|}{|S|}-1\right|\ge q_3\right]\lesssim \exp(-q_3^2sn^3)$. Finally, we know $\frac{1}{|\oS|}\sum_{t\in\oS}g(t)\le 1$. Therefore, we have
\begin{align}
    \Pr[|\eta_2|\ge q_3]\lesssim \exp(-q_3^2sn^3).\label{eqn:lemma_6_eta2_indicator}
\end{align}
Combining the bounds (\ref{eqn:lemma_6_eta1_bound_indicator}) and (\ref{eqn:lemma_6_eta2_indicator}), we see that, with probability at least
\[1-O(\exp(-q_1^2n^3)+\exp(-q_2^2n^3)+\exp(-q_3^2sn^3)),\]
we have
\begin{align}
    \left|\frac{1}{|S|}\sum_{t\in S}g(t)-\frac{1}{|\oS|}\sum_{t\in\oS}g(t)\right|
    &\le \frac{q_2}{(1-q_1)}+q_3.\label{eqn:lemma_6_bound_indicator}
\end{align}

Setting $q_1=1/2$ and $q_2=q_3=n^{-1}$, (\ref{eqn:lemma_6_bound_indicator}) becomes
\begin{align}\nonumber
    \Pr\left[\left|\frac{1}{|S|}\sum_{t\in S}g(t)-\frac{1}{|\oS|}\sum_{t\in\oS}g(t)\right|\ge\frac{1}{n}\right]\lesssim \exp(-n).
\end{align}
\end{proof}

\begin{lem}\label{lem:size_S_ratio}
    Let $\oS$ be of size $|\oS|\ge n^3/2$ and let $S$ be randomly chosen from $\oS$ by including each element randomly with probability $s$. Then,
    \begin{equation}
        \Pr\left[\left|\frac{s|\oS|}{|S|}-1\right|\le q\right]\lesssim \exp(-q^2sn^3)
    \end{equation}
    for $0\le q\le 1$.
\end{lem}	
\begin{proof}
Since $|S|$ follows a binomial distribution with mean $s|\oS|$, we have
\begin{equation}
    \Pr\left[\Big||S|-s|\oS|\Big|\ge qs|\oS|\right]\lesssim \exp(-q^2s|\oS|)\lesssim \exp(-q^2sn^3).\label{eqn:binom_size_concentration}
\end{equation}
Therefore,
\begin{align*}
    \Pr\left[\left||S|-s|\oS|\right|\ge s|\oS|/2\right]\lesssim \exp(-sn^3).
\end{align*}
This means that, with probability at least $1-O(\exp(-sn^3))$, we have
$|S|/(s|\oS|)\in\left[\frac{1}{2},\frac{3}{2}\right]$, implying $s|\oS|/|S|\in\left[\frac{2}{3},2\right]$.
If this holds,
\[\left|\frac{s|\oS|}{|S|}-1\right|\le2\left|\frac{|S|}{s|\oS|}-1\right|.\]
Therefore, assuming that $\frac{s|\oS|}{|S|}\in\left[\frac{2}{3},2\right]$, we have
\begin{align*}
    \left|\frac{s|\oS|}{|S|}-1\right|&\le 2\left|\frac{|S|}{s|\oS|}-1\right|=2\left|\frac{|S|-s|\oS|}{s|\oS|}\right|.
\end{align*}
Again applying (\ref{eqn:binom_size_concentration}), we have that 
\[\Pr\left[\left|\frac{|S|-s|\oS|}{s|\oS|}\right|\ge q\right]\lesssim \exp(-q^2sn^3).\].
Taking a union bound over the events $s|\oS|/|S|\in\left[\frac{2}{3},2\right]$ and $\left|\frac{|S|-s|\oS|}{s|\oS|}\right|\ge q$, we conclude that
\begin{equation}
    \Pr\left[\left|\frac{s|\oS|}{|S|}-1\right|\lesssim q\right]\lesssim \exp(-sn^3)+\exp(-q^2sn^3)\lesssim \exp(-q^2sn^3).
\end{equation}
Taking $q:=1/n$ completes the proof.
\end{proof}

\begin{lem}\label{lem:trace_bound_nonsymmetric}
    Let $\bM_t:=\bx_i\bx_k^\top+\bx_k\bx_i^\top-\bx_i\bx_j^\top-\bx_j\bx_i^\top+\bx_j\bx_j^\top-\bx_k\bx_k^\top$ and let $\bA$ be independent from $\bM_t$. Then
    \begin{align*}
    	\Pr[|\Tr(\bM_t\bV\bA^\top)|>q]&\lesssim \exp(-q/\|\bV\bA^\top\|_\textnormal{F}).
    \end{align*}
\end{lem}
\begin{proof}

By definition,
\begin{align*}\nonumber
    \Tr(\bM_t\bV\bA^\top)&=\bx_i^\top\bV\bA^\top\bx_k+\bx_k^\top\bV\bA^\top\bx_i-\bx_i^\top\bV\bA^\top\bx_j-\bx_j^\top\bV\bA^\top\bx_i+\bx_j^\top\bV\bA^\top\bx_j-\bx_k^\top\bV\bA^\top\bx_k.\label{eqn:nonsymmetric_trace_six_terms}
\end{align*}
By Assumption \ref{assumption:distribution}, each term is subexponential with subexponential norm $\lesssim\|\bV\bA^\top\|_\textnormal{F}$. Therefore, $\Tr(\bM_t\bV\bA^\top)$ is also subexponential and $\|\Tr(\bM_t\bV\bA^\top)\|_{\psi_1}\lesssim\|\bV\bA^\top\|_\textnormal{F}$. Therefore, by \citet[Proposition 2.7.1]{vershynin2020high},
\begin{align*}
	\Pr\left[\left|\Tr(\bM_t\bV\bA^\top)\right|>q\right]
    &\le \exp(-q/\|\bV\bA^\top\|_\textnormal{F}).
\end{align*}

\end{proof}
\begin{lem}\label{lem:norm}
    Let $T\subset [n]$ have size $\Omega(n)$. Then,
    \begin{align*}
        \Pr\left[\left|\frac{1}{|T|}\sum_{a\in T}\|\bx_a\|-\EE[\|\bx_a\|]\right|\ge q\EE[\|\bx_a\|]\right]\lesssim \exp(-nq^2).
    \end{align*}
\end{lem}
\begin{proof}
    By Assumption \ref{assumption:distribution}, we have $\left\|\|\bx_a\|-\EE[\|\bx_a\|]\right\|_{\psi_2}\lesssim 1$. Then, the lemma follows from \citet[Theorem 2.7.3]{vershynin2025highNEW}.
\end{proof}

\begin{lem}\label{lem:bernstein}
    Let $T\subset [n]$ have size $\Omega(n)$. If $\boldsymbol{B}$ is a matrix, we have
    \begin{align*}
        \Pr\left[\left|\frac{1}{|T|}\sum_{a\in T}\bx_a^\top\boldsymbol{B}\bx_a-\EE[\bx_a^\top\boldsymbol{B}\bx_a]\right|\ge q\|\boldsymbol{B}\|_\textnormal{F}\right]\lesssim \exp(-n\min\{q,q^2\}).
    \end{align*}
\end{lem}
\begin{proof}

     By Assumption \ref{assumption:distribution}, $\|\bx_a^\top\boldsymbol{B}\bx_a\|_{\psi_1}\lesssim\|\boldsymbol{B}\|_\textnormal{F}$. Then, we can apply the Bernstein inequality \citep[Theorem 2.8.1]{vershynin2020high} to get 
    \begin{align*}
        \Pr\left[\left|\frac{1}{|T|}\sum_{a\in T}\bx_a^\top\boldsymbol{B}\bx_a-\EE[\bx_a^\top\boldsymbol{B}\bx_a]\right|\ge q\|\boldsymbol{B}\|_\textnormal{F}\right]\lesssim \exp(-n\min\{q,q^2\}).
    \end{align*}
\end{proof}

\begin{lem}\label{lem:bernstein_diff}
    If $\boldsymbol{B}$ has rank $r$, we have
    \begin{align*}
        \Pr\left[\left|\frac{1}{n(n-1)}\sum_{i\neq j\in[n]}\bx_i^\top\boldsymbol{B}\bx_j\right|\ge\frac{r\|\boldsymbol{B}\|}{n}\right]\lesssim \exp(-n).
    \end{align*}
\end{lem}
\begin{proof}
We will rewrite as follows:
\begin{align*}
    \frac{1}{n(n-1)}\sum_{i\neq j\in[n]}\bx_i^\top\boldsymbol{B}\bx_j&=\frac{n}{n-1}\underbrace{\left(\frac{1}{n^2}\sum_{i,j\in[n]}\bx_i^\top\boldsymbol{B}\bx_j\right)}_{\beta_1}+\frac{1}{n-1}\underbrace{\left(\frac{1}{n}\sum_{i\in[n]}\bx_i^\top\boldsymbol{B}\bx_i\right).}_{\beta_2}
\end{align*}
Writing the compact singular value decomposition of $\boldsymbol{B}$ as $\bU_r\bLambda\bW_r^\top$, we can write \[\beta_1=\left(\frac{1}{n}\sum_{i\in[n]}\bx_i^\top\bU_r\right)\bLambda\left(\frac{1}{n}\sum_{i\in[n]}\bx_i^\top\bU_r\right)^\top.\] Applying Cauchy-Schwarz and using Lemma \ref{lem:sum_of_random_vectors} with $q:=\sqrt{q_1}$, this is bounded by $q_1r\|\boldsymbol{B}\|/n$ with probability at least $1-O(\exp(-nq_1))$.

For $\beta_2$, by Lemma \ref{lem:bernstein}, this is bounded by $q_2\|\boldsymbol{B}\|$ with probability at least $1-O(\exp(-n\min\{q_2,q_2^2\}))$.

Therefore we have
\begin{align*}
    \left|\frac{1}{n(n-1)}\sum_{i\neq j\in[n]}\bx_i^\top\boldsymbol{B}\bx_j\right|&\le\frac{n}{n-1}|\beta_1|-\frac{1}{n-1}|\beta_2|\\
    &\le\frac{1}{n-1}q_1r\|\boldsymbol{B}\|-\frac{1}{n-1}q_2\|\boldsymbol{B}\|.
\end{align*}
Setting $q_1=q_2=1/2$, we achieve the claimed bound with probability at least $1-O(\exp(-n))$.
\end{proof}

\begin{lem}\label{lem:exp_xVAx}
    For any matrices $\bV,\bA\in\R^{p\times r}$,
    \(\EE[\bx^\top\bV\bA^\top\bx]\le\|\bV\bA^\top\|.\)
\end{lem}
\begin{proof}
    \begin{align*}
        \EE[\bx^\top\bV\bA^\top\bx]&=\EE[\Tr(\bx^\top\bV\bA^\top\bx)]\\
        &=\EE[\Tr(\bV\bA^\top\bx\bx^\top)]\\
        &=\Tr(\bV\bA^\top\EE[\bx\bx^\top])\\
        &=\Tr(\bV\bA^\top\bSigma)\\
        &\lesssim\|\bV\bA^\top\|.
    \end{align*}
\end{proof}

\begin{lem}\label{lem:covariance_estimation}
    Let $\bU_r,\bW_r$ be $p\times r$ semi-orthogonal matrices and let $T\subset[n]$ be of size $|T|\gtrsim n$. Then, \[\left\|\bW_r^\top\left(\frac{1}{|T|}\sum_{b\in T}\bx_b\bx_b^\top-\bSigma\right)\bU_r\right\|\lesssim n^{-1/2}\] with probability at least $1-O(\exp(-n^{1/2}))$.
\end{lem}
\begin{proof}
    Define $\mathbb S^{r-1}:=\{\bv\in\R^r\mid\|\bv\|=1\}$. Then,
    \begin{align*}
        \left\|\bW^\top_r\left(\frac{1}{|T|}\sum_{b\in T}\bx_b\bx_b^\top-\bSigma\right)\bU_r\right\|&=\sup_{\bv_1,\bv_2\in\mathbb S^{r-1}}(\bW_r\bv_1)^\top\left(\frac{1}{|T|}\sum_{b\in T}\bx_b\bx_b^\top-\bSigma\right)(\bU_r\bv_2).
    \end{align*}
    We will discretize all possible $\bv_1,\bv_2$ in an appropriate way. Let $\mathcal{E}_{1/3}$ be a $1/3$-net for $\mathbb S^{r-1}$. We can take this net such that $|\mathcal{E}_{1/3}|\le C_1^r$ for some constant $C_1$. 
    Since $\mathbb S^{r-1}$ is compact, we can define the vectors $\bv_1^\star,\bv_2^\star$ where the supremum
    \begin{align*}
        \sup_{\bv_1,\bv_2\in\mathbb S^{r-1}}(\bW_r\bv_1)^\top\left(\frac{1}{|T|}\sum_{b\in T}\bx_b\bx_b^\top-\bSigma\right)(\bU_r\bv_2)
    \end{align*}
    is attained. By the definition of a $1/3$-net, there exist points $\hat\bv_1,\hat\bv_2\in\mathcal E_{1/3}$ such that $\|\hat\bv_1-\bv_1^\star\|\le 1/3$ and $\|\hat\bv_2-\bv_2^\star\|\le 1/3$. 
    Then, 
    \begin{align*}
        &\left\|\bW^\top_r\left(\frac{1}{|T|}\sum_{b\in T}\bx_b\bx_b^\top-\bSigma\right)\bU_r\right\|\\
        &=\sup_{\bv_1,\bv_2\in\mathbb S^{r-1}}(\bW_r\bv_1)^\top\left(\frac{1}{|T|}\sum_{b\in T}\bx_b\bx_b^\top-\bSigma\right)(\bU_r\bv_2)\\
        &=(\bW_r\bv_1^\star)^\top\left(\frac{1}{|T|}\sum_{b\in T}\bx_b\bx_b^\top-\bSigma\right)(\bU_r\bv_2^\star)\\
        &=(\bW_r\hat\bv_1)^\top\left(\frac{1}{|T|}\sum_{b\in T}\bx_b\bx_b^\top-\bSigma\right)(\bU_r\hat\bv_2)\\
        &\hspace{10mm}+(\bW_r(\bv_1^\star-\hat\bv_1))^\top\left(\frac{1}{|T|}\sum_{b\in T}\bx_b\bx_b^\top-\bSigma\right)(\bU_r\hat\bv_2)+(\bW_r\bv_1^\star)^\top\left(\frac{1}{|T|}\sum_{b\in T}\bx_b\bx_b^\top-\bSigma\right)(\bU_r(\bv_2^\star-\hat\bv_2))\\
        &\le(\bW_r\hat\bv_1)^\top\left(\frac{1}{|T|}\sum_{b\in T}\bx_b\bx_b^\top-\bSigma\right)(\bU_r\hat\bv_2)+\frac{1}{3}\left\|\frac{1}{|T|}\sum_{b\in T}\bx_b\bx_b^\top-\bSigma\right\|+\frac{1}{3}\left\|\frac{1}{|T|}\sum_{b\in T}\bx_b\bx_b^\top-\bSigma\right\|.
    \end{align*}
    Since $\left\|\frac{1}{|T|}\sum_{b\in T}\bx_b\bx_b^\top-\bSigma\right\|=\left\|\bW^\top_r\left(\frac{1}{|T|}\sum_{b\in T}\bx_b\bx_b^\top-\bSigma\right)\bU_r\right\|$, this implies
    \begin{align*}
        \frac{1}{3}\left\|\bW^\top_r\left(\frac{1}{|T|}\sum_{b\in T}\bx_b\bx_b^\top-\bSigma\right)\bU_r\right\|&\le(\bW_r\hat\bv_1)^\top\left(\frac{1}{|T|}\sum_{b\in T}\bx_b\bx_b^\top-\bSigma\right)(\bU_r\hat\bv_2)\\
        &\le\max_{\bv_1,\bv_2\in\mathcal E_{1/3}}(\bW_r\bv_1)^\top\left(\frac{1}{|T|}\sum_{b\in T}\bx_b\bx_b^\top-\bSigma\right)(\bU_r\bv_2)\\
        \left\|\bW^\top_r\left(\frac{1}{|T|}\sum_{b\in T}\bx_b\bx_b^\top-\bSigma\right)\bU_r\right\|&=3\max_{\bv_1,\bv_2\in\mathcal E_{1/3}}(\bW_r\bv_1)^\top\left(\frac{1}{|T|}\sum_{b\in T}\bx_b\bx_b^\top-\bSigma\right)(\bU_r\bv_2).
    \end{align*}
    For fixed $\bv_1,\bv_2$, we have
    \begin{align*}
        (\bW_r\bv_1)^\top\left(\frac{1}{|T|}\sum_{b\in T}\bx_b\bx_b^\top-\bSigma\right)(\bU_r\bv_2)&\le\left\|\frac{1}{|T|}\sum_{b\in T}\bx_b\bx_b^\top-\bSigma\right\|.
    \end{align*}
    By \citet[Exercise 4.7.3]{vershynin2020high} with probability at least $1-O(\exp(-n^{1/2}))$, this is bounded by $C_2n^{-1/2}\|\bSigma\|\lesssim C_2n^{-1/2}$ (by Assumption \ref{assumption:distribution}).
    
    We can apply the union bound over $\mathcal E_{1/3}$ to say that
    \begin{align*}\Pr\left[\exists\bv_1,\bv_2\in\mathcal E_{1/3},\>\>(\bW_r\bv_1)^\top\left(\frac{1}{|T|}\sum_{b\in T}\bx_b\bx_b^\top-\bSigma\right)(\bU_r\bv_2)\ge C_2n^{-1/2}\right]&\lesssim C_1^{r}\exp(-n^{1/2})\\&\lesssim \exp(-n^{1/2}).
    \end{align*}
    Therefore
    \[\sup_{\bv_1,\bv_2\in\mathcal E_{1/3}}(\bW_r\bv_1)^\top\left(\frac{1}{|T|}\sum_{b\in T}\bx_b\bx_b^\top-\bSigma\right)(\bU_r\bv_2)\lesssim n^{-1/2}\] with probability at least $1-O(\exp(-n^{1/2}))$, implying that
    \[\left\|\bW_r^\top\left(\frac{1}{|T|}\sum_{b\in T}\bx_b\bx_b^\top-\bSigma\right)\bU_r\right\|\lesssim n^{-1/2}\] with probability at least $1-O(\exp(-n^{1/2}))$.
\end{proof}

\begin{lem}\label{lem:sum_of_random_vectors}
    Let $T\subset [n]$ be of size $O(n)$. Let $\bU_r$ be a $p\times r$ unitary matrix. Then
    \[\Pr\left[\left\|\frac{1}{n}\sum_{a\in[n]}\bx_a^\top\bU_r\right\|\ge q\sqrt{\frac{r}{n}}\right]\lesssim \exp(-nq^2).\]
\end{lem}
\begin{proof}
    Define $\mathbb{S}^{r-1}:=\{\bv\in\R^r\mid\|\bv\|=1\}$. Then,
    \begin{align*}
        \left\|\frac{1}{n}\sum_{a\in[n]}\bx_a^\top\bU_r\right\|&=\sup_{\bv\in\mathbb{S}^{r-1}}\frac{1}{n}\sum_{a\in[n]}\bx_a^\top\bU_r\bv.
    \end{align*}
    Let $\mathcal E_{1/2}$ be a $1/2$-net for $\mathbb S^{r-1}$. We can take this net such that $|\mathcal E_{1/2}|\le C^r$ for some constant $C$. Since $\mathbb S^{r-1}$ is compact, we can choose a vector $\bv^\star$ where the supremum
    $\sup_{\bv\in\mathbb{S}^{r-1}}\frac{1}{n}\sum_{a\in[n]}\bx_a\bU_r\bv$ is attained. By the definition of a $1/2$-net, there exists a point $\hat\bv\in\mathcal E_{1/2}$ such that $\|\hat\bv-\bv^\star\|\le 1/2$. Then,
    \begin{align*}
        \left\|\frac{1}{n}\sum_{a\in[n]}\bx_a^\top\bU_r\right\|&=\sup_{\bv\in\mathbb{S}^{r-1}}\frac{1}{n}\sum_{a\in[n]}\bx_a^\top\bU_r\bv\\
        &=\frac{1}{n}\sum_{a\in[n]}\bx_a^\top\bU_r\bv^\star\\
        &=\frac{1}{n}\sum_{a\in[n]}\bx_a^\top\bU_r\hat\bv+\frac{1}{n}\sum_{a\in[n]}\bx_a^\top\bU_r(\bv^\star-\hat\bv)\\
        &\le\frac{1}{n}\sum_{a\in[n]}\bx_a^\top\bU_r\hat\bv+\frac{1}{2} \left\|\frac{1}{n}\sum_{a\in[n]}\bx_a^\top\bU_r\right\|\\
        \frac{1}{2}\left\|\frac{1}{n}\sum_{a\in[n]}\bx_a^\top\bU_r\right\|&\le\frac{1}{n}\sum_{a\in[n]}\bx_a^\top\bU_r\hat\bv\\
        &=\max_{\bv\in\mathcal E_{1/2}}\frac{1}{n}\sum_{a\in[n]}\bx_a^\top\bU_r\bv\\
        \left\|\frac{1}{n}\sum_{a\in[n]}\bx_a^\top\bU_r\right\|&=2\max_{\bv\in\mathcal E_{1/2}}\frac{1}{n}\sum_{a\in[n]}\bx_a^\top\bU_r\bv.
    \end{align*}
    For a fixed $\bv$, $\bx_a^\top\bU_r\bv$ is a subgaussian random variable with $\|\bx_a^\top\bU_r\bv\|_{\psi_2}\lesssim\|\bU\bv\|=1$. Then we can apply \citet[Theorem 2.6.2]{vershynin2020high} to conclude
    \[\Pr\left[\left|\frac{1}{n}\sum_{a\in[n]}\bx_a^\top\bU_r\bv\right|\ge q\right]\lesssim \exp(-q^2n^2).\]
    We can apply the union bound over $\mathcal E_{1/2}$ to conclude
    \begin{align*}
        \Pr\left[\exists\bv\in\mathcal E_{1/2},\>\left|\frac{1}{n}\sum_{a\in[n]}\bx_a^\top\bU_r\bv\right|\ge q\right]&\lesssim 2C^r\exp(-q^2n^2)\\
        &\lesssim \exp(-rn^2q^2).
    \end{align*}
    Therefore
    \begin{align*}
        \Pr\left[\left\|\frac{1}{n}\sum_{a\in[n]}\bx_a^\top\bU_r\right\|\ge q\sqrt{\frac{r}{n}}\right]\lesssim\exp(-nq^2).
    \end{align*}
\end{proof}

\begin{lem}\label{lem:normal_third_power}
    Let $\bB\in\R^{p\times p}$ be matrix. Then, for a set $T=\{a_1,\ldots,a_{|T|}\}\subset[n]$ of size $\Omega(n)$,
    \[\Pr\left[\left\|\frac{1}{|T|}\sum_{a\in T}\bx_a\bx_a^\top\bB\bx_a-\EE\left[\bx_a\bx_a^\top\bB\bx_a\right]\right\|\ge Cn^{-1/4}\|\bB\|_\textnormal{F}\right]\lesssim \exp(-n^{1/4})\] for some constant $C$.
\end{lem}
\begin{proof}
    We will use the following definitions, following \citet{subweibullbound}:
    
    Define the norm
    \begin{equation*}
        \|X\|_{\phi}:=\inf\{\eta:\EE[\phi(|X|/\eta)]\le 1\}
    \end{equation*}
    Using the function $\psi_\alpha(t):=\exp(t^\alpha)-1$, we get the Orlicz norm $\|\cdot\|_{\psi_\alpha}$.
    
    If we define the function $\Psi_{\alpha,L}$, where $\alpha>0$ and $L\ge 0$, by defining its inverse function
    \[\Psi_{\alpha,L}^{-1}(x):=\sqrt{\log(1+x)}+L(\log(1+x))^{1/\alpha},\] we get the generalized Bernstein-Orlicz norm $\|\cdot\|_{\Psi_{\alpha,L}}$.
    
    Now we will apply \citet[Theorem 3.1]{subweibullbound} to the sum
    \[\sum_{a\in T}\frac{1}{|T|}\left(\left(\bx_a\bx_a^\top\bB\bx_a\right)_i-\EE\left[\left(\bx_a\bx_a^\top\bB\bx_a\right)_i\right]\right)\] for each coordinate $i\in[p]$, with $\alpha=2/3$ in the theorem.
    
    First, we have 
    \begin{align*}          
        \left\|\left(\bx_a\bx_a^\top\bB\bx_a\right)_i-\EE\left[\left(\bx_a\bx_a^\top\bB\bx_a\right)_i\right]\right\|_{\psi_{2/3}}&=\left\|\left(\bx_a\bx_a^\top\bB\bx_a\right)_i\right\|_{\psi_{2/3}}\\
        &=\left\|x_{a,i}\left(\bx_a^\top\bB\bx_a\right)\right\|_{\psi_{2/3}}\\
        &\le\|x_{a,i}\|_{\psi_2}\cdot\left\|\bx_a^\top\bB\bx_a\right\|_{\psi_1}\\
        &\le C\|\bB\|_\textnormal{F}
    \end{align*}
    for all $a\in T,i\in[r]$, where the penultimate line is from \citet[Proposition D.2]{subweibullbound} and the last line is from Assumption \ref{assumption:distribution}.
    
    Therefore, by \citet[Theorem 3.1]{subweibullbound}, we have
    \begin{align}\nonumber
        &\left\|\sum_{a\in T}\frac{1}{|T|}\left(\left(\bx_a\bx_a^\top\bB\bx_a\right)_i-\EE\left[\left(\bx_a\bx_a^\top\bB\bx_a\right)_i\right]\right)\right\|_{\Psi_{2/3,L_n(2/3)}}\\\nonumber
        &\le C\left\|\left(\frac{1}{|T|}\left\|\left(\bx_{a_1}\bx_{a_1}^\top\bB\bx_{a_1}\right)_i\right\|_{\psi_{2/3}},\ldots,\frac{1}{|T|}\left\|\left(\bx_{a_{|T|}}\bx_{a_{|T|}}^\top\bB\bx_{a_{|T|}}\right)_i\right\|_{\psi_{2/3}}\right)\right\|\\
        &\le\frac{C'\|\bV\bA^\top\|_\textnormal{F}}{\sqrt{|T|}}.\label{eqn:gbo_bound_23}
    \end{align}
    Now, we apply \citet[Proposition A.3]{subweibullbound} to the bound \ref{eqn:gbo_bound_23} to conclude that
    \[\Pr\left[\frac{1}{|T|}\left|\sum_{a\in T}\left(\bx_a\bx_a^\top\bB\bx_a\right)_i-\EE\left[\left(\bx_a\bx_a^\top\bB\bx_a\right)_i\right]\right|\ge C'\|\bB\|_\textnormal{F}\left(\frac{\sqrt{q}}{\sqrt{|T|}}+\frac{q^{3/2}}{|T|}\right)\right]\le2\exp(-q).\]
    Making the substitution $q=q'n^{1/2}$, we get 
    \[\Pr\left[\frac{1}{|T|}\left|\sum_{a\in T}\left(\bx_a\bx_a^\top\bB\bx_a\right)_i-\EE\left[\left(\bx_a\bx_a^\top\bB\bx_a\right)_i\right]\right|\ge C'\|\bB\|_\textnormal{F}\left(\frac{\sqrt{q'}}{|T|^{1/4}}+\frac{{q'}^{3/4}}{|T|^{1/4}}\right)\right]\le2\exp(-q'n^{1/2}).\]
    Applying the union bound across all indicies $i\in[p]$, we get
    \[\Pr\left[\left\|\frac{1}{|T|}\sum_{a\in T}\bx_a\bx_a^\top\bB\bx_a-\EE\left[\bx_a\bx_a^\top\bB\bx_a\right]\right\|_{\infty}\ge C'\|\bB\|_\textnormal{F}\left(\frac{\sqrt{q'}}{|T|^{1/4}}+\frac{{q'}^{3/4}}{|T|^{1/4}}\right)\right]\le 2p\exp(-q'|T|^{1/2}),\]
    which implies
    \[\Pr\left[\left\|\frac{1}{|T|}\sum_{a\in T}\bx_a\bx_a^\top\bB\bx_a-\EE\left[\bx_a\bx_a^\top\bB\bx_a\right]\right\|\ge C'\sqrt{p}\|\bB\|_\textnormal{F}\left(\frac{\sqrt{q'}}{|T|^{1/4}}+\frac{{q'}^{3/4}}{|T|^{1/4}}\right)\right]\le 2p\exp(-q'|T|^{1/2}).\]
    Finally, setting $q':=1$, we get
    \[\Pr\left[\left\|\frac{1}{|T|}\sum_{a=1}^n\bx_a\bx_a^\top\bB\bx_a-\EE\left[\bx_a\bx_a^\top\bB\bx_a\right]\right\|\ge C''|T|^{-1/4}\|\bB\|_\textnormal{F}\right]\lesssim \exp(-|T|^{1/4}).\]
    Recalling that $|T|=\Omega(n)$ completes the proof.
\end{proof}

\begin{lem}\label{lem:sub_weibull}
    Let $\bB\in\R^{p\times p}$ be matrix. Then
    \begin{align*}
        \Pr\left[\left|\frac{1}{n}\sum_{a=1}^n \left((\bx_a^\top\bB\bx_a)^2-\EE[(\bx_a^\top\bB\bx_a)^2]\right)\right|\ge Cn^{-3/8}\|\bB\|_\textnormal{F}^2\right]\lesssim \exp(-n^{1/4})
    \end{align*}
    for some constant $C$.
\end{lem}
\begin{proof}
    We use the definitions of the Orlicz norm and the generalized Bernstein-Orlicz norm written in Lemma \ref{lem:normal_third_power}.
    
    We will apply \citet[Theorem 3.1]{subweibullbound} to the sum
    \[\sum_{a=1}^n \frac{1}{n}\left((\bx_a^\top\bB\bx_a)^2-\EE[(\bx_a^\top\bB\bx_a)^2]\right)\] with $\alpha=1/2$.
    
    First, we have \[\left\|(\bx_a^\top\bB\bx_a)^2-\EE[(\bx_a^\top\bB\bx_a)^2]\right\|_{\psi_{1/2}}=\|(\bx_a^\top\bB\bx_a)^2\|_{\psi_{1/2}}\le\|\bx_a^\top\bB\bx_a\|_{\psi_1}^2\le C\|\bB\|_\textnormal{F}^2\] for some constant $C$.
    
    Therefore, \citet[Theorem 3.1]{subweibullbound} says that
    \begin{align}\nonumber
        &\left\|\sum_{a=1}^n \frac{1}{n}\left((\bx_a^\top\bB\bx_a)^2-\EE[(\bx_a^\top\bB\bx_a)^2]\right)\right\|_{\Psi_{1/2,L_n(1/2)}}\\\nonumber
        &\le C\left\|
        \left(\frac{1}{n}\|(\bx_1^\top\bB\bx_1)^2\|_{\psi_{1/2}},\ldots,\frac{1}{n}\|(\bx_n^\top\bB\bx_n)^2\|_{\psi_{1/2}}\right)\right\|\\
        &\le\frac{C'\|\bB\|_\textnormal{F}^2}{\sqrt{n}}.\label{eqn:gbo_bound}
    \end{align}
    
    Now, we apply \citet[Proposition A.3]{subweibullbound} to the bound (\ref{eqn:gbo_bound}) to conclude that
    \begin{align*}
        \Pr\left[\frac{1}{n}\left|\sum_{a=1}^n \left((\bx_a^\top\bB\bx_a)^2-\EE[(\bx_a^\top\bB\bx_a)^2]\right)\right|\ge C\|\bB\|_\textnormal{F}^2\left(\frac{\sqrt{q}}{\sqrt{n}}+\frac{q^2}{n}\right)\right]\le 2\exp(-q).
    \end{align*}
    Setting $q:=n^{1/4}$, this gives the bound
    \begin{align*}
        \Pr\left[\frac{1}{n}\left|\sum_{a=1}^n \left((\bx_a^\top\bB\bx_a)^2-\EE[(\bx_a^\top\bB\bx_a)^2]\right)\right|\ge Cn^{-3/8}\|\bB\|_\textnormal{F}^2\right]\le 2\exp(-
        n^{1/4}).
    \end{align*}
\end{proof}

\begin{lem}\label{lem:lambda_tilde_to_lambda_hat}
    Let $\widetilde{\lambda}_1,\ldots,\widetilde{\lambda}_r$ be the $r$ largest generalized eigenvalues of $(\frac{1}{{n'}^2}\bX^\top\boldsymbol H\bX,\frac{1}{n'}\bPsi)$ and let $\widehat{\lambda}_1,\ldots,\widehat{\lambda}_r$ be the $r$ largest generalized eigenvalues of $(\frac{1}{{n'}^2}\bX^\top\widehat{\boldsymbol H}\bX,\frac{1}{n'}\bPsi)$, where
    \[\boldsymbol H=\bX\bA_\star\bA_\star^\top\bX^\top,\hspace{10mm}\bPsi=\bX\bX^\top\] and
    \[\frac{1}{n'}\|\boldsymbol H-\widehat{\boldsymbol H}\|\le C_1\sqrt{\frac{1}{n'\log n'}}.\] Then, with probability at least $1-O(\exp(-n))$,
    \[|\widehat{ \lambda}_i-\widetilde{\lambda}_i|\le C_2\sqrt{\frac{1}{n'\log n'}}\] for all $i$.
\end{lem}
\begin{proof}

First, we make the bound
\[\left|\lambda_i(\frac{1}{{n'}^2}\bX^\top{\boldsymbol H}\bX,\frac{1}{n'}\bPsi)-\lambda_i(\frac{1}{{n'}^2}\bX^\top\widehat{\boldsymbol H}\bX,\frac{1}{n'}\bPsi)\right|\le\lambda_{\max}(\frac{1}{{n'}^2}\bX^\top({\boldsymbol H}-\widehat{\boldsymbol H})\bX,\frac{1}{n'}\bPsi).\]
We can bound the error term as follows
\begin{align*}
    \lambda_{\max}(\frac{1}{{n'}^2}\bX^\top({\boldsymbol H}-\widehat{\boldsymbol H})\bX,\frac{1}{n'}\bPsi)&=\max_{\bv:\frac{1}{n'}\bv^\top\bPsi\bv=1}\frac{1}{{n'}^2}\bv^\top\bX^\top(\boldsymbol H-\widehat{\boldsymbol H})\bX\bv\\
    &\le\max_{\bv:\frac{1}{n'}\bv^\top\bPsi\bv=1}\frac{1}{{n'}^2}\|\bX^\top(\boldsymbol H-\widehat{\boldsymbol H})\bX\|\cdot\frac{1}{n'}\|\bPsi\|\\
    &\le \frac{1}{{n'}^2}\|\bX^\top(\boldsymbol H-\widehat{\boldsymbol H})\bX\|\cdot\frac{1}{n'}\|\bPsi\|\\
    &\le \frac{1}{{n'}^2}\|\bX\|^2\|\bH-\widehat{\bH}\|\cdot\frac{1}{n'}\|\bX\|^2\\
    &\lesssim\|\bH-\widehat{\bH}\|\\
    &\lesssim\sqrt{\frac{1}{{n'}\log n'}},
\end{align*}
where the second-to-last line holds with probability $1-O(\exp(-n))$ because $\bX$ is a random $n\times p$ matrix with independent subgaussian rows, and $p=o(n)$, so $\|\bX\|\lesssim \sqrt{n'}$ with probability at least $1-O(\exp(-n))$ by \citet[Exercise 4.43]{vershynin2025highNEW}.

Taking the union bound over all $i\in[r]$, we have that each eigenvalue is within $C_2\sqrt{\frac{1}{{n'}\log n'}}$ with probability at least $1-O(\exp(-n))$. 

\end{proof}

\begin{lem}\label{lem:u_tilde_to_u_hat}
    Let $\widetilde{\bU}$ be the $r$ largest generalized eigenvectors of $(\frac{1}{{n'}^2}\bX^\top\boldsymbol H\bX,\frac{1}{n'}\bPsi)$ and let $\widehat{\bU}$ be the $r$ largest generalized eigenvalues of $(\frac{1}{{n'}^2}\bX^\top\widehat{\boldsymbol H}\bX,\frac{1}{n'}\bPsi)$, where
    \[\boldsymbol H=\bX\bA_\star\bA_\star^\top\bX^\top,\hspace{10mm}\bPsi=\bX\bX^\top\] and
    \[\frac{1}{n'}\|\boldsymbol H-\widehat{\boldsymbol H}\|\lesssim \sqrt{\frac{1}{n'\log n'}}.\]
    Then, 
    \begin{align*}
        \inf_{\boldsymbol{O}\in\mathcal O_{r\times r}}\|\widetilde{\bU}\boldsymbol{O}-\widehat{\bU}\|\lesssim\sqrt{\frac{1}{n'\log n'}}.
    \end{align*}
\end{lem}
\begin{proof}
Our goal will be to apply the Davis-Kahan Theorem for the generalized eigenvalue problem.

We know 
\begin{align}\nonumber
    \left|\frac{1}{{n'}^2}\bX^\top\widehat{\boldsymbol H}\bX-\frac{1}{{n'}^2}\bX^\top\boldsymbol H\bX\right|&=\left|\frac{1}{{n'}^2}\bX^\top(\boldsymbol H-\widehat{\boldsymbol H})\bX\right|\\\nonumber
    &\lesssim \sqrt{\frac{1}{{n'}\log n'}}=:\epsilon\label{eqn:xhx_to_xhhat_x}
\end{align}
with probability at least $1-O(\exp(-n))$. Note that $\epsilon=o(1)$. 

We also show that $\lambda_{\max}(\frac{1}{n'}\bPsi)/\lambda_{\min}(\frac{1}{n'}\bPsi)$ is bounded by a constant, with high probability. By \citet[Theorem 4.6.1]{vershynin2020high}, 
\begin{equation}\label{eqn:lambda_min_sigma}
    \Pr[\lambda_{\min}(\bPsi)\le\frac{1}{4}n']\le \Pr[\sigma_{\min}(\bX)\le\frac{1}{2}\sqrt{n'}]\lesssim\exp(-n)
\end{equation} and \[\Pr[\lambda_{\max}(\bPsi)\ge4n']=\Pr[\sigma_{\max}(\bX)\ge2\sqrt{n'}]\lesssim\exp(-n)\] (the theorem requires the rows of $\bX$ to be isotropic; we apply it to $\bSigma^{-1/2}\bX$ and use the fact that $\|\bSigma\|\lesssim 1$).

Therefore, with probability at least $1-O(\exp(-n))$, $\lambda_{\max}(\frac{1}{n'}\bPsi)/\lambda_{\min}(\frac{1}{n'}\bPsi)\le 16$. 

Now, we can apply the Davis-Kahan theorem for the generalized eigenvalue problem \citep{generalized_dk}. We have
\begin{align}\label{eqn:generalized_dk}
    \inf_{\boldsymbol{O}\in\mathcal O_{r\times r}}\|\widetilde{\bU}\boldsymbol{O}-\widehat{\bU}\|\le\frac{1}{\lambda_{\min}^{3/2}(\frac{1}{n'}\bPsi)}\frac{7\sqrt{2}\sqrt{r-s+1}}{\min\{\nu_s-\nu_{s-1}, \nu_r-\nu_{r+1}\}}\left\|\frac{1}{{n'}^2}\bX(\widehat{\bH}-\bH)\bX^\top\right\|
\end{align}
for $1\le s\le r$, where $\nu_i:=\lambda_i(\frac{1}{n'}\bH)$. Letting $s=1$, we have $\min\{\nu_s-\nu_{s-1}, \nu_r-\nu_{r+1}\}=\nu_r-\nu_{r+1}$. First, since $\bA_\star$ is rank $r$, we know $\nu_{r+1}=0$. 

Note $\lambda_r(\frac{1}{n'}\bX\bA_\star\bA_\star^\top\bX^\top)\ge\frac{1}{n'}\sigma_{\min}^2(\bX)\sigma_r^2(\bA_\star)$. By \citet[Theorem 4.6.1]{vershynin2020high}, $\sigma_{\min}(\bX)\le\sqrt{n'}/2$ with probability at least $1-O(\exp(-n))$. Therefore, $\nu_r\ge \sigma_{r}^2(\bA_\star)/2$ with the same probability.

Recalling that $\lambda_{\min}(\frac{1}{n'}\bPsi)\ge\frac{1}{2}$ by \eqref{eqn:lambda_min_sigma} and plugging in to the Davis-Kahan theorem stated in \eqref{eqn:generalized_dk}, we get
\begin{align*}
    \inf_{\boldsymbol{O}\in\mathcal O_{r\times r}}\|\widetilde{\bU}\boldsymbol{O}-\widehat{\bU}\|\le\frac{28\sqrt{2}\sqrt{r}}{\sigma_{r}^2(\bA_\star)}\epsilon\lesssim\sqrt{\frac{1}{n'\log n'}}
\end{align*}

\end{proof}

\begin{lem}\label{lem:utilde_lambdatilde_utilde}
Let $\widetilde\bU\in\R^{p\times r}$ be the $r$ largest generalized eigenvectors of the pair $(\frac{1}{{n'}^2}\bX^\top\boldsymbol H\bX,\frac{1}{{n'}}\bPsi)$, where $\bPsi=\bX\bX^\top$ and $\boldsymbol H=\bX\bA_\star\bA_\star^\top\bX^\top$, and let $\widetilde{\bLambda}$ be the corresponding eigenvalues. Then,
\[\bA_\star\bA^\top_\star=\widetilde{\bU}\widetilde{\bLambda}\widetilde{\bU}^\top.\]
\end{lem}
\begin{proof}
First, note that since $n>p$, $\bPsi$ is positive definite with probability 1, so it is invertible and has an invertible square root.

Let $\bV:=\frac{1}{\sqrt{n'}}\bPsi^{1/2}\widetilde\bU$. This implies $\bV^\top\bV=\bI_{r}.$ Now, by the generalized eigenvalue property of $\widetilde\bU$, we have
\begin{align*}
    \frac{1}{{n'}^2}\bX^\top\boldsymbol H\bX\widetilde\bU&=\frac{1}{n'}\bPsi\widetilde\bU\widetilde{\bLambda}\\
    \frac{1}{{n'}^{3/2}}\bX^\top\boldsymbol H\bX\bPsi^{-1/2}\bV&=\frac{1}{\sqrt{n'}}\bPsi^{1/2}\bV\widetilde{\bLambda}\\
    \frac{1}{{n'}}\bPsi^{-1/2}\bX^\top\boldsymbol H\bX\bPsi^{-1/2}\bV&=\bV\widetilde{\bLambda}\\
    \frac{1}{{n'}}\bPsi^{-1/2}\bX^\top\boldsymbol H\bX\bPsi^{-1/2}&=\bV\widetilde{\bLambda}\bV^\top\\
    &=\frac{1}{n'}\bPsi^{1/2}\widetilde\bU\widetilde{\bLambda}\widetilde\bU^\top\bPsi^{1/2}\\
    \bPsi^{-1}\bX^\top\boldsymbol H\bX\bPsi^{-1}&=\widetilde\bU\widetilde{\bLambda}\widetilde\bU^\top\\
    \bA_\star\bA_\star^\top&=\widetilde\bU\widetilde{\bLambda}\widetilde\bU^\top.
\end{align*}

\end{proof}

\begin{lem}\label{lem:approx_aat}
    Let $\widehat{\bU}$ and $\widehat{\bLambda}$ be the generalized eigenvectors and eigenvalues of $(\frac{1}{{n'}^2}\bX^\top\widehat{\bH}\bX,\frac{1}{n'}\bPsi)$, where $\frac{1}{n'}\|\widehat{\bH}-\bX\bA_\star\bA_\star^\top\bX^\top\|\lesssim\sqrt{\frac{1}{n'\log n'}}$. Then there exists an orthogonal matrix $\boldsymbol O$ such that
    \[\|\widehat{\bU}\widehat{\bLambda}^{1/2}-\bA_\star\boldsymbol O\|\lesssim\sqrt{\frac{1}{n'\log n'}}.\]
\end{lem}
\begin{proof}
By Lemma \ref{lem:utilde_lambdatilde_utilde}, we have $\bA_\star\bA_\star^\top=\widetilde{\bU}\widetilde{\bLambda}\widetilde{\bU}^\top$. Therefore, there exists an orthogonal matrix $\boldsymbol P$ such that $\bA_\star\boldsymbol P=\widetilde{\bU}\widetilde{\bLambda}^{1/2}$. 

Equivalently, letting $\widetilde{\bu}_1,\ldots,\widetilde{\bu}_r$  be the columns of $\widetilde{\bU}$, letting $\widetilde{\lambda}_1,\ldots,\widetilde{\lambda}_r$ be the entries of $\widetilde{\bLambda}$, and letting $\boldsymbol e_i\in\R^r$ be the $i$-th standard basis vector, we have
\[\bA_\star\boldsymbol P=\sum_{i=1}^r\sqrt{\widetilde{\lambda}_i}\widetilde{\bu}_i\boldsymbol{e}_i^\top.\]

Now we will reindex the eigenvalues to account for multiplicity. Let the unique eigenvalues sorted in nonincreasing order be $\widetilde{\mu}_1,\ldots,\widetilde{\mu}_s$, each with multiplicity $m_1,\ldots,m_s$. Let the eigenvectors corresponding to eigenvalue $\widetilde{\mu}_i$ be $\widetilde{\bv}_{i,1},\ldots,\widetilde{\bv}_{i,m_i}$. Use the same reindexing function that takes $\widetilde\bu_i\mapsto \widetilde\bv_{i,j}$ to take $\boldsymbol{e}_i\mapsto\boldsymbol{e}'_{i,j}$.

Then, we can write
\[\bA_\star\boldsymbol P=\sum_{i=1}^s \sqrt{\widetilde{\lambda}_i}\sum_{j=1}^{m_i}\widetilde{\bv}_{i,j}{\boldsymbol{e}'}_{i,j}^\top.\]

We will show that if we find the $\widehat{\boldsymbol U}$, the generalized eigenvectors of the pair $(\frac{1}{{n'}^2}\bX^\top\widehat{\boldsymbol H}\bX,\frac{1}{n'}\bPsi)$, we get approximately $\widetilde\bU$, up to rotation. Lemma \ref{lem:u_tilde_to_u_hat} says that \[\inf_{\boldsymbol O\in\mathcal O_{r\times r}}\|\widehat\bU\boldsymbol{O}-\widetilde\bU\|\lesssim\sqrt{\frac{1}{n'\log n'}}.\] Since $\mathcal O_{r\times r}$ is compact, we can define $\boldsymbol O$ as the matrix that achieves this bound. Note that this also means that $\|(\widehat{\bU}\boldsymbol{O})_{I}-\widetilde{\bU}_I\|\lesssim\sqrt{\frac{1}{n'\log n'}}$ where $I$ selects any subset of the columns.

Let $\{\widehat{\bv}_{i,j}\}_{i,j}$ be the reindexed versions of $\widehat{\bu}_1,\ldots,\widehat{\bu}_p$ according to the same reindexing function that takes $\widetilde\bu_i\mapsto \widetilde\bv_{i,j}$.

$\boldsymbol O$ is an orthogonal transformation of eigenvectors corresponding to the same eigenvalue. For an eigenvalue $\widetilde{\mu}_i$, let $I_i$ select the columns of $\widetilde{\bU}$ corresponding to $\widetilde{\mu}_i$. Then, for each unique eigenvalue $\widetilde{\mu}_i$, we have $(\widehat{\bU}\boldsymbol O)_{I_i}=\widehat{\bU}_{I_i}\boldsymbol O$. Then, we can write
\begin{align*}
    \left|\widehat{\bU}_{I_i}\boldsymbol O-\sum_{j=1}^{m_i}\widetilde{\bU}_{I_i}\right|&\lesssim\sqrt{\frac{1}{n'\log n'}}\\
    \left|\left(\sum_{j=1}^{m_i}\widehat{\bv}_{i,j}{\boldsymbol{e}'}_{i,j}^\top\right)\boldsymbol O-\sum_{j=1}^{m_i}\widetilde{\bv}_{i,j}{\boldsymbol{e}'}_{i,j}^\top\right|&\lesssim\sqrt{\frac{1}{n'\log n'}}
\end{align*}

Now, we turn to the eigenvalues. By Lemma \ref{lem:lambda_tilde_to_lambda_hat}, we have $|\widehat{\mu}_i-\widetilde{\mu}_i|\lesssim\sqrt{\frac{1}{n'\log n'}}$ for all $i$. Since the smallest eigenvalue is lower bounded by $\sigma_{\min}(\bA_\star)$, which we assume to be $\Omega(1)$, the square root function is Lipschitz continuous on the eigenvalues. Therefore $|\sqrt{\widehat{\mu}_i}-\sqrt{\widetilde{\mu}_i}|\lesssim\sqrt{\frac{1}{n'\log n'}}$

Therefore
\begin{align*}
    \left\|\sqrt{\widetilde{\mu}_i}\sum_{j=1}^{m_i}\widetilde{\bv}_{i,j}{\boldsymbol{e}'}_{i,j}^\top-\sqrt{\widehat{\mu}_i}\left(\sum_{j=1}^{m_i}\widehat{\bv}_{i,j}{\boldsymbol{e}'}_{i,j}^\top\right)\boldsymbol O\right\|
    &\le \left\|\sqrt{\widetilde{\mu}_i}\sum_{j=1}^{m_i}\widetilde{\bv}_{i,j}{\boldsymbol{e}'}_{i,j}^\top-\sqrt{\widetilde{\mu}_i}\left(\sum_{j=1}^{m_i}\widehat{\bv}_{i,j}{\boldsymbol{e}'}_{i,j}^\top\right)\boldsymbol O\right\|\\
    &\hspace{5mm}+\left\|\sqrt{\widetilde{\mu}_i}\left(\sum_{j=1}^{m_i}\widehat{\bv}_{i,j}{\boldsymbol{e}'}_{i,j}^\top\right)\boldsymbol O-\sqrt{\widehat{\mu}_i}\left(\sum_{j=1}^{m_i}\widehat{\bv}_{i,j}{\boldsymbol{e}'}_{i,j}^\top\right)\boldsymbol O\right\|\\
    &\lesssim |\sqrt{\widetilde{\mu}_i}|\sqrt{\frac{1}{n'\log n'}}+\sqrt{\frac{1}{n'\log n'}}\lesssim \|\bA_\star\|\sqrt{\frac{1}{n'\log n'}}
\end{align*}
Since this holds for each $i$, we have
\begin{align*}
    \|\bA_\star\boldsymbol P\boldsymbol O^\top-\widetilde{\bU}\widetilde{\bLambda}^{1/2}\|&=\left\|\left(\sum_{i=1}^s \widetilde{\mu}_i\sum_{j=1}^{m_i}\widetilde{\bv}_{i,j}{\boldsymbol{e}'}_{i,j}^\top\right)\boldsymbol O^\top-\sum_{i=1}^s \widehat{\mu}_i\sum_{j=1}^{m_i}\widehat{\bv}_{i,j}{\boldsymbol{e}'}_{i,j}^\top\right\|\\
    &=\left\|\sum_{i=1}^s \widetilde{\mu}_i\sum_{j=1}^{m_i}\widetilde{\bv}_{i,j}{\boldsymbol{e}'}_{i,j}^\top-\left(\sum_{i=1}^s \widehat{\mu}_i\sum_{j=1}^{m_i}\widehat{\bv}_{i,j}{\boldsymbol{e}'}_{i,j}^\top\right)\boldsymbol O\right\|\\
    &=\left\|\sum_{i=1}^s \widetilde{\mu}_i\sum_{j=1}^{m_i}\widetilde{\bv}_{i,j}{\boldsymbol{e}'}_{i,j}^\top-\sum_{i=1}^s \widehat{\mu}_i\left(\sum_{j=1}^{m_i}\widehat{\bv}_{i,j}{\boldsymbol{e}'}_{i,j}^\top\right)\boldsymbol O\right\|\\
    &\lesssim r\|\bA_\star\|\sqrt{\frac{1}{n'\log n'}}\lesssim \sqrt{\frac{1}{n'\log n'}}.
\end{align*}
$\boldsymbol P\boldsymbol O^\top$ is an orthogonal matrix, giving the claimed result.
\end{proof}

\begin{lem}\label{lem:gaussian_satisfies_assumption}
    If $\bSigma\in\R^{p\times p}$ satisfies $\|\bSigma\|\lesssim 1$, then $\mathcal N(0,\bSigma)$ satisfies Assumption \ref{assumption:distribution}.
\end{lem}
\begin{proof}
We will show each part of Assumption \ref{assumption:distribution}.

\underline{Condition 1:} $\EE[\bx_i\bx_i^\top]=\bSigma$, so $\|\EE[\bx_i\bx_i^\top]\|\lesssim 1$ by assumption. By \citet[Corollary 3]{OULDBABAcumulant}, 
\[\EE[\vc(\bx_i\bx_i^\top)\vc(\bx_i\bx_i^\top)^\top]=(\boldsymbol{I}_{p^2}+\bK_{p,p})\bSigma\otimes\bSigma+\vc(\bSigma)\vc(\bSigma)^\top,\]
where $\bK_{p,p}$ is the $p^2\times p^2$commutation matrix. We upper bound the norm of both parts. First, $\|(\boldsymbol{I}_{p^2}+\bK_{p,p})\bSigma\otimes\bSigma\|\le\|\bSigma\otimes\bSigma\|+\|\bK_{p,p}\bSigma\otimes\bSigma\|\le 2\|\bSigma\|^2\lesssim 1$. Second, $\|\vc(\bSigma)\vc(\bSigma)^\top\|\le\|\bSigma\|^2\lesssim 1$. This concludes the proof of condition 1.

\underline{Condition 2:} We will show that all marginals $\langle \bx_i,\bv\rangle$ have constant subgaussian norm. $\langle \bx_i,\bv\rangle\sim \mathcal N(0,\|\bSigma^{1/2}\bv\|_2^2)$. Since $\|\bSigma^{1/2}\bv\|_2^2\lesssim 1$, this has bounded subgaussian norm.

\underline{Condition 3:} We will show the concentration of the norm using a proof technique similar to the one used to prove \citet[Theorem 3.1.1]{vershynin2025highNEW}. Let $\bU\bLambda\bU^\top$ be the SVD of $\bSigma$. We know $\bx_i$ is equal in distribution to $\bSigma^{1/2}\bv$ where $\bv\sim\mathcal N(0,\bI_p)$. Therefore $\|\bx_i\|$ is equal in distribution to $\|\bSigma^{1/2}\bv\|=\|\bU\bLambda^{1/2}\bv\|=\|\bLambda^{1/2}\bv\|$. By the equivalence of bounds on the subgaussian norm and concentration bounds (\citep[Proposition 2.6.6]{vershynin2025highNEW}), it suffices to show $\Pr[|\|\bLambda^{1/2}\bv\|-\EE[\|\bLambda^{1/2}\bv\|]|\ge q]\le 2\exp(-Cq^2)$. Consider $\|\bLambda^{1/2}\bv\|^2-\|\bSigma\|_\textnormal{F}^2=\sum_{\ell=1}^p(\lambda_\ell(v_\ell^2-1)$. Each term has bounded subexponential norm. Now we apply the Bernstein inequality \citep[Corollary 2.9.2]{vershynin2025highNEW} to conclude 
\[\Pr[|\|\bLambda^{1/2}\bv\|^2-\|\bLambda\|_\textnormal{F}^2|\ge q]\le 2\exp\left(-C\min\left\{\frac{q^2}{\|\bLambda\|_\textnormal{F}^2},\frac{q}{\|\bLambda\|_\textnormal{F}}\right\}\right)\le 2\exp(-C'\min\{q^2,q\}).\]
Observing that $|z-1|\ge\delta$ implies $|z^2-1|\ge\max\{\delta,\delta^2\}$ for $z,\delta\ge 0$, we conclude $\Pr[|\|\bLambda^{1/2}\bv\|-\EE[\|\bLambda^{1/2}\bv\|]|\ge \delta]\le 2\exp(C'\delta^2)$ as desired.

\underline{Condition 5:} Let the SVD of $\bLambda^{1/2}\bU^\top\bB_2\bU\bLambda^{1/2}$ be $\bV\bGamma\boldsymbol{W}^\top$, where $\bGamma=\text{diag}(\gamma_1,\ldots,\gamma_r,0,\ldots,0)$.

We can write
\begin{align*}
    \bx_i^\top\boldsymbol{B}_2\bx_j&=\bv_i^\top\bV\bGamma\boldsymbol{W}^\top\bv_j,
\end{align*}
where $\bv_a,\bv_b\sim\mathcal N(0,\boldsymbol I_p)$ (not necessarily independent). $\bv_i^\top\bV\bGamma\boldsymbol{W}^\top\bv_j$ is equal in distribution to $\bv_i^\top\bGamma\bv_j$. Then
\begin{align*}
    \bv_i^\top\bGamma\bv_j&=\sum_{\ell=1}^{p} \gamma_\ell v_{1,\ell} v_{2,\ell}.
\end{align*}
We know that $\lambda_\ell v_{1,\ell}v_{2,\ell}$ has mean 0 and is subexponential with subexponential norm $\|\lambda_\ell x_{i,\ell}x_{j,\ell}\|_{\psi_1}\lesssim |\lambda_\ell|$. Therefore $\|\bv_i^\top\bGamma\bv_j\|_{\psi_1}\lesssim \sum_{\ell=1}^r|\gamma_\ell|\le\sqrt r\|\bLambda^{1/2}\bU^\top\bB_2\bU\bLambda^{1/2}\|_\textnormal{F}\lesssim\|\bB_2\|_\textnormal{F}$.

\underline{Condition 5:} As in the proof of the last condition, $\bx_a$ is equal in distribution to $\bU\bLambda^{1/2}\bv_a$ where $\bv_a\sim\mathcal N(0,\mathcal I_p)$. Let the SVD of $\bLambda^{1/2}\bU^\top\bB_1\bU\bLambda^{1/2}$ be $\bV\bGamma\bV^\top$, where $\bGamma=\text{diag}(\gamma_1,\ldots,\gamma_r,0,\ldots,0)$. Then, define $\bz_a:=\bV^\top\bv_a$ for $a\in\{i,j,k\}$. Note $\bx_a^\top\bB_1\bx_b=\bz_a^\top\bGamma\bz_b$. Then, we write

\begin{align}\nonumber
    &\hspace{-10mm}|(\bx_i-\bx_j)^\top\bB_1(\bx_i-\bx_j)-(\bx_i-\bx_k)^\top\bB_1(\bx_i-\bx_k)|\\\nonumber
    &=|(2\bx_i^\top\bB_1\bx_j+\bx_j^\top\bB_1\bx_j)-(2\bx_i^\top\bB_1\bx_k+\bx_k^\top\bB_1\bx_k)|\\\nonumber
    &=|(2\bz_i^\top\bGamma\bz_j+\bz_j^\top\bGamma\bz_j)-(2\bz_i^\top\bGamma\bz_k+\bz_k^\top\bGamma\bz_k)|\\\nonumber
    &=\left|\sum_{\ell=1}^r\left(2\gamma_\ell z_{i,\ell}z_{j,\ell}+\gamma_\ell z_{j,\ell}^2\right)-\sum_{\ell=1}^r\left(2\gamma_\ell z_{i,\ell}z_{k,\ell}+\gamma_\ell z_{k,\ell}^2\right)\right|\\\nonumber
    &=\left|\sum_{\ell=1}^r\left(2\gamma_\ell z_{i,\ell}(z_{j,\ell}-z_{k,\ell})+\gamma_\ell(z_{j,\ell}^2-z_{k,\ell}^2)\right)\right|\\
    &=\left|\sum_{\ell=1}^r\left(2\gamma_\ell z_{i,\ell}(z_{j,\ell}-z_{k,\ell})+\gamma_\ell(z_{j,\ell}+z_{k,\ell})(z_{j,\ell}-z_{k,\ell})\right)\right|.\label{eqn:abs_val_trace_as_sum}
\end{align}

Defining $w_{\ell,1}:=z_{j,\ell}+z_{k,\ell}$ and $w_{\ell,2}:=z_{j,\ell}-z_{k,\ell}$, we note that $w_{\ell,1}$ and $w_{\ell,2}$ are independent conditional on $z_{j,\ell}$ and both are distributed according to $\mathcal N(z_{j,\ell},1)$. 
Then, we can rewrite (\ref{eqn:abs_val_trace_as_sum}) as
\begin{align*}
    \left|\sum_{\ell=1}^r\left(\gamma_\ell w_{\ell,1}-2\gamma_\ell z_{i,\ell}\right)w_{\ell,2}\right|.
\end{align*}
We will finish the lower bound by conditioning on $\{z_{j,\ell}\}_{\ell\in[r]}$, as follows.
\begin{align}\nonumber
    \EE\left[\left|\sum_{\ell=1}^r\left(\gamma_\ell w_{\ell,1}-2\gamma_\ell z_{i,\ell}\right)w_{\ell,2}\right|\right]
    &=\EE\left[\EE\left[\left|\sum_{\ell=1}^r\left(\gamma_\ell w_{\ell,1}-2\gamma_\ell z_{i,\ell}\right)w_{\ell,2}\right|\mid\{z_{j,\ell}\}_{\ell\in[r]}\right]\right]\\\nonumber
    &\ge\frac{1}{2}\EE\left[\left|\sum_{\ell=1}^r\left(\gamma_\ell w_{\ell,1}-2\gamma_\ell z_{i,\ell}\right)\right|\right]
\end{align}
Where we used the fact that if $a\sim\mathcal N(0,\sigma^2)$, then $\EE[|a|]\ge\sigma/2$. Now, $w_{\ell,1}-z_{i,\ell}=z_{j,\ell}+z_{k,\ell}-2z_{i,\ell}$, which is distributed according to $\mathcal N(0,6)$. So,

\begin{align}
    \frac{1}{2}\EE\left[\left|\sum_{\ell=1}^r\left(\gamma_\ell w_{\ell,1}-2\gamma_\ell z_{i,\ell}\right)\right|\right]&\ge\left|\sum_{\ell=1}^r\gamma_\ell\right|\ge\|\bLambda^{1/2}\bU^\top\bB_1\bU\bLambda^{1/2}\|_\textnormal{F}\gtrsim\|\bB_1\|_\textnormal{F}.
\end{align}

\end{proof}

\section{More Experiments}\label{sec:more_experiments}

\subsection{Synthetic Data}\label{sec:more_synth_data}
We conduct three synthetic data experiments, each with $n=120$ data points, each with $p=20$ features. In the first set of experiments, we draw the feature vectors $\bx_i$ from $\mathcal N(0, \Sigma_{\chi})$, where $\Sigma_{\chi}$ is a diagonal matrix where a random subset of half of the entries are $\chi$ and the other entries are $1/\chi$. We use $\chi=1$, $\chi=5$, and $\chi=20$ to model various levels of non-isotropy.

In the second set of experiments, we generate the feature vectors $\bx_i$ by drawing a random subset of half of the coordinates from $\text{Bern}(1-\chi)$ and the other half of the coordinates from $\text{Bern}(1-1/\chi)$. We use $\chi=2$, $\chi=5$, and $\chi=20$.

In the third set of experiments, we generate feature vectors with correlated coordinates, with correlation decaying according to an autoregressive (AR) process. Specifically, for a specified $\rho$, we generate the matrix $\Sigma_\rho=(\rho^{|i-j|})_{i,j}$, and draw $\bx_i\sim\mathcal N(0,\Sigma_\rho)$. We use $\rho=0.5$, $\rho=0.8$, and $\rho=0.95$.

For synthetic data settings, we generate labels synthetically as described in the body of the paper in Subsection \ref{sec:data_sources}.

\subsection{Real Data}\label{sec:more_real_data}

We describe the real datasets here.

\paragraph*{ACS Employment}
We also use the ACS Employment dataset from Folktables. This dataset is also based on ACS PUMS. The target is whether the person is employed. The features are age, educational attainment, sex, disability status, employment status of parents, mobility status, citizenship status, military service, ancestry, nativity, relationship to reference person, hearing difficulty, vision difficulty, cognitive difficulty, race, and grandparents living with grandchildren. The full documentation of this dataset appears in \citet{folktables}. We convert categorical features (marital status, employment status of parents, mobility status, citizenship, military service, ancestry, race, and relationship to reference person) to one-hot vectors. Then, we take the top $p=20$ principle components of the data, normalized to have mean 0 and standard deviation 1, as our features. 

\paragraph*{ACS Mobility}
We use the ACS Mobility dataset through the Folktables \citep{folktables} Python library. This dataset is based on the American Community Survey (ACS) Public Use Microdata Sample (PUMS) provided by the US Census Bureau. The target label is whether the person moved in the last year. The features are age, educational attainment, marital status, sex, disability status, employment status of parents, citizenship status, military service, ancestry, nativity, relationship to reference person, hearing difficulty, vision difficulty, cognitive difficulty, race, grandparents living with grandchildren, class of worker, employment status, hours worked per week, travel time to work, and income. The full documentation of this dataset appears in \citet{folktables}. We convert categorical features (marital status, employment status of parents, citizenship, military service, ancestry, race, class of worker, employment status, and relationship to reference person) to one-hot vectors. Then, we take the top $p=20$ principle components of the data, normalized to have mean 0 and standard deviation 1, as our features.

\paragraph*{Credit Card Default}
We use the Default of Credit Card Clients dataset from the UCI Machine Learning Repository \citep{default_of_credit_card_clients_350}. This dataset is composed of customer data from Taiwan. The target label is whether a person defaulted on their credit card. The features are amount of credit, gender, education, marital status, age, history of past repayments (6 features), amount of past bill statements (6 features), and amount of previous payments (6 features), for a total of 23 features. We convert marital status from a category to a one-hot vector. Then, we take the top $p=20$ principle components of the data, normalized to have mean 0 and standard deviation 1, as our features. 

\paragraph*{Community Crime}
We use the Communities and Crime dataset from the UCI Machine Learning Repository \citep{communities_and_crime_183}. This dataset ``combines socio-economic data from the 1990 US Census, law enforcement data from the 1990 US LEMAS survey, and crime data from the 1995 FBI UCR.'' The target label is number of violent crimes per capita. We remove the identifying features: state, county, community, community name, and fold. The remaining features describe aspects of the community and can be found in full at \citet{communities_and_crime_183}. We take the top $p=20$ principle components of the data, normalized to have mean 0 and standard deviation 1, as our features. 

\paragraph*{CDC Diabetes}
We use the CDC Diabetes Health Indicators Dataset from the UCI Machine Learning Repository \citep{cdc_diabetes}. This dataset is based on data released by the Center for Disease Control (CDC). The target label is whether a person has diabetes. The features are presence of high blood pressure, presence of high cholesterol, presence of a cholesterol check, body mass index, smoking status, history of stroke, presence of heart disease or heart attack, presence of physical activity, fruit consumption, vegetable consumption, heavy alcohol consumption, presence of healthcare, presence of financial barriers to healthcare, self-reported general health, self-reported mental health, self-reported physical health, difficulty walking, sex, age, education, and income. We take the top $p=20$ principle components of the data, normalized to have mean 0 and standard deviation 1, as our features. 
\subsection{Fairness Measurements}\label{sec:fairness_measurements}
Definition \ref{def:indiv_fairness} defines individual fairness with a parameter $l$. To measure the empirical fairness parameter $l$ of a classifier $\mathcal A$ with respect to fairness metric $d_{\bK}$, we form the ratios \[\xi_{i,j}:=\frac{D(\mathcal A(\bx_i),\mathcal A(\bx_j))}{d_{\bK}(\bx_i,\bx_j)}\] for each pair $i,j$. Then, $L=\max_{i,j}\xi_{i,j}$. In Table \ref{tab:fairness_experiment} we report $\max_{i,j}\widehat{\xi}_{i,j}$, the value of $L$ when $\bK=\widehat{\bK}$, and $\max_{i,j}\xi_{i,j}^\star$, the value of $L$ when $\bK=\bK_\star$.

In practice, there are issues with these measurements. The empirical $L$ can vary quite widely according to random variations in $\mathcal A$ or $d_{\bK}$ since it is a maximum of random variables (and because $d_{\bK}$ appears in the denominator so small additive variation in $d_{\bK}$ can lead to large variation in $l$. This has two effects: first, even when measuring $L$ with respect to the estimated fairness metric used to train the classifier, we sometimes observe large $l$ values. Second, when comparing the $l$ value computed with respect to the true fairness metric to the $l$ value computed with respect to the estimate fairness metric, we sometimes see large relative differences. 

Because of these practical issues, we also report a fairness measurement proposed by \citet{maity2021statistical}. For an algorithm $\mathcal A:\bX\to\boldsymbol Y$ and a loss function $\ell:\bX\times\boldsymbol Y\to\mathbb R_+$, \citet{maity2021statistical} gives an algorithm $\Phi_d:\bX\times\boldsymbol{Y}\to\bX$ that takes a pair $(\bx,y)$ to a point $\Phi_d(\bx,y)$ that is close to $\bx$ with respect to a fairness metric $d$, but which fares worse under $\mathcal A$ relative to $\ell$. \citet{maity2021statistical} proposes measuring the statistic
\begin{equation}\label{eqn:loss_ratio}\zeta:=\EE\left[\frac{\ell(\mathcal A(\Phi_d(\bx,y)), y)}{\ell(\mathcal A(\bx), y)}\right].\end{equation} This statistic is motivated by a modified definition of individual fairness called distributionally robust fairness, proposed by \citet{Yurochkin2019TrainingIF}. \citet{maity2021statistical} suggests that \eqref{eqn:loss_ratio} is an interpretable summary of the fairness of a model, considering a loss function that measures the impact of the model. In our experiments, we define the loss function $\ell$ as the $\ell_1$ loss. In Table \ref{tab:fairness_experiment} we report $\widehat{\zeta}$, the value of \eqref{eqn:loss_ratio} when $d=d_{\widehat{\bK}}$, and $\zeta^\star$, the value of \eqref{eqn:loss_ratio} when $d=d_{\bK_\star}$. When using this fairness measurement, we observe both that the model appears to be fair with respect to the estimated fairness metric, and that the fairness with respect to the true fairness metric is very similar.

\subsection{Results for the Additional Data Sources}\label{sec:more_results}

Here, we present results for the additional data sources here. We follow the experimental procedure described in Subsection \ref{sec:procedure}.

Gradient descent convergence is shown in Figure \ref{fig:gd}. The table of fairness results appears as Table \ref{tab:fairness_experiment}.

\begin{figure}
    \centering
    \includegraphics[width=0.5\linewidth]{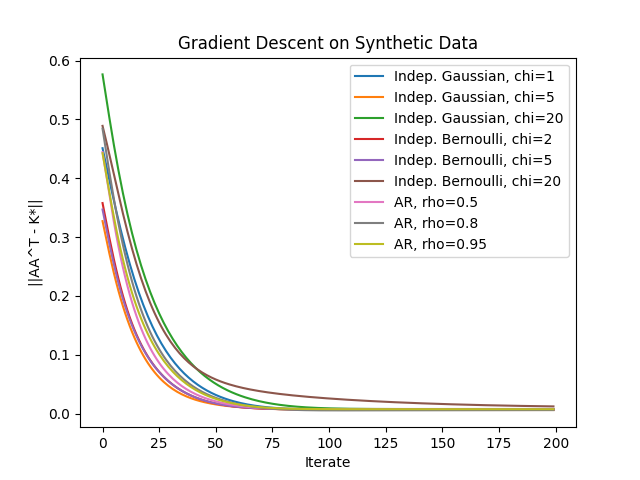}\includegraphics[width=0.5\linewidth]{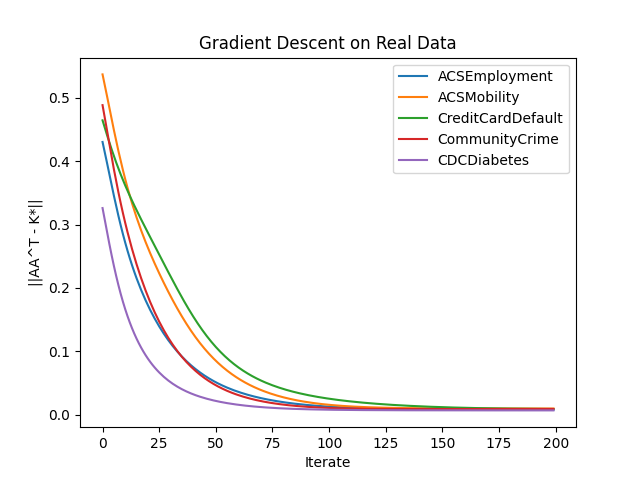}
    \caption{Gradient descent performance on synthetic and real data.}
    \label{fig:gd}
\end{figure}

\begin{table}
    \centering
\begin{tabular}{lcccccc}\toprule
Dataset & $\max_{i,j}\widehat{\xi}_{i,j}$ & $\max_{i,j}\xi^\star_{i,j}$ & Difference & $\widehat{\zeta}$ & $\zeta^\star$ & Difference \\\midrule
Gaussian, $\chi=1$ & 14.9823 & 12.874 & 16.38\% & 1.0028 & 0.9998 & 0.30\% \\
Gaussian, $\chi=5$ & 19.9851 & 14.2388 & 40.36\% & 1.0066 & 1.0032 & 0.34\% \\
Gaussian, $\chi=20$ & 25.0564 & 26.6872 & 6.11\% & 1.0019 & 1.0000 & 0.19\% \\
Binomial, $\chi=1$ & 13.7352 & 19.393 & 29.17\% & 1.0033 & 0.9983 & 0.50\% \\
Binomial, $\chi=5$ & 2.3325 & 2.461 & 5.22\% & 0.9992 & 1.0043 & 0.51\% \\
Binomial, $\chi=20$ & 2.238 & 2.2103 & 1.25\% & 0.9972 & 0.9998 & 0.26\% \\
AR, $\rho=0.5$ & 13.6159 & 26.1403 & 47.91\% & 0.9986 & 1.0017 & 0.31\% \\
AR, $\rho=0.8$ & 24.5866 & 21.7825 & 12.87\% & 1.0022 & 0.9962 & 0.60\% \\
AR, $\rho=0.95$ & 17.4981 & 14.6528 & 19.42\% & 1.0131 & 1.0109 & 0.22\% \\
ACSEmployment & 7.3757 & 6.4485 & 14.38\% & 0.9999 & 0.998 & 0.19\% \\
ACSMobility & 1.3863 & 1.5378 & 9.85\% & 0.9998 & 0.9997 & 0.01\% \\
CreditCardDefault & 2.4675 & 0.9473 & 160.48\% & 0.9982 & 1.0028 & 0.46\% \\
CDCDiabetes & 2.4154 & 2.5213 & 4.20\% & 1.0009 & 0.9995 & 0.14\% \\
CommunityCrime & 6.1965 & 5.8934 & 5.14\% & 1.1096 & 1.1829 & 6.20\% \\\bottomrule\\
\end{tabular}
    \caption{Comparing fairness violations with respect to the estimated and true metrics.}
    \label{tab:fairness_experiment}
\end{table}

\end{document}